\documentclass{article}

\usepackage{microtype}
\usepackage{graphicx}
\usepackage{subcaption}
\usepackage{booktabs}
\usepackage{hyperref}
\usepackage{subfiles}

\usepackage{amsmath}
\usepackage{amssymb}
\usepackage{mathtools}
\usepackage{amsthm}
\usepackage{multirow}
\usepackage{booktabs}
\usepackage{thm-restate}

\usepackage{amsmath,amsfonts,bm, xcolor}
\definecolor{skyblue}{RGB}{135,206,235}

\def\norm#1{\left\lVert#1\right\rVert}
\def\inner#1{\left\langle#1\right\rangle}
\def\abs#1{\left|#1\right|}

\def\floor#1{\left\lfloor #1 \right\rfloor}
\def\bignorm#1{\left\lVert#1\right\rVert}
\def\bigabs#1{\left|#1\right|}
\def\bigopen#1{\left(#1\right)}
\def\bigset#1{\left\{#1\right\}}
\def\bigclosed#1{\left[#1\right]}
\def\tightparagraph#1{\noindent\textbf{#1}~~}

\newcommand{\igd}{{$\mathrm{IGD}$}}
\newcommand{\herding}{{$\mathrm{Herding}$}}
\newcommand{\sings}{{$\mathrm{SS}$}}
\newcommand{\randr}{{$\mathrm{RR}$}}
\newcommand{\grab}{{$\mathrm{GraB}$}}

\newcommand{\tightmargin}{\vspace{-5pt}}
\newcommand{\smallertightmargin}{\vspace{-2pt}}

\def\eqref#1{equation~\ref{#1}}

\def\floor#1{\lfloor #1 \rfloor}
\def\1{\bm{1}}

\def\vx{{\bm{x}}}
\def\vy{{\bm{y}}}
\def\vz{{\bm{z}}}

\def\mI{{\bm{I}}}

\DeclareMathAlphabet{\mathsfit}{\encodingdefault}{\sfdefault}{m}{sl}
\SetMathAlphabet{\mathsfit}{bold}{\encodingdefault}{\sfdefault}{bx}{n}

\def\gO{{\mathcal{O}}}

\def\sR{{\mathbb{R}}}

\newcommand{\E}{\mathbb{E}}

\newcommand{\R}{\mathbb{R}}

\usepackage[noabbrev]{cleveref}
\usepackage{tabularx}
\usepackage{threeparttable}
\usepackage{xcolor}

\usepackage{enumitem}

\usepackage{anyfontsize}

\Crefname{assumption}{Assumption}{Assumptions}

\usepackage[accepted]{icml2025}

\theoremstyle{plain}
\newtheorem{theorem}{Theorem}[section]

\newtheorem{lemma}[theorem]{Lemma}

\theoremstyle{definition}
\newtheorem{definition}[theorem]{Definition}
\newtheorem{assumption}[theorem]{Assumption}
\theoremstyle{remark}

\usepackage[textsize=tiny]{todonotes}

\icmltitlerunning{Incremental Gradient Descent with Small Epoch Counts is Surprisingly Slow on Ill-Conditioned Problems}

\begin{document}

\twocolumn[
\icmltitle{Incremental Gradient Descent with Small Epoch Counts\\is Surprisingly Slow on Ill-Conditioned Problems}

\icmlsetsymbol{equal}{*}

\begin{icmlauthorlist}
\icmlauthor{Yujun Kim}{equal,kaist}
\icmlauthor{Jaeyoung Cha}{equal,kaist}
\icmlauthor{Chulhee Yun}{kaist}
\end{icmlauthorlist}

\icmlaffiliation{kaist}{Kim Jaechul Graduate School of AI, Korea Advanced Institute of Science and Technology}

\icmlcorrespondingauthor{Chulhee Yun}{chulhee.yun@kaist.ac.kr}

\icmlkeywords{Machine Learning, ICML}

\vskip 0.3in
]

\printAffiliationsAndNotice{\icmlEqualContribution}

\begin{abstract}
Recent theoretical results demonstrate that the convergence rates of permutation-based SGD (e.g., random reshuffling SGD) are faster than uniform-sampling SGD; however, these studies focus mainly on the \textit{large epoch regime}, where the number of epochs $K$ exceeds the condition number $\kappa$. 
In contrast, little is known when $K$ is smaller than $\kappa$, and it is still a challenging open question whether permutation-based SGD can converge faster in this \textit{small epoch regime}~\citep{safran2021random}.
As a step toward understanding this gap, we study the naive deterministic variant, Incremental Gradient Descent (IGD), on smooth and strongly convex functions.
Our lower bounds reveal that for the small epoch regime, IGD can exhibit surprisingly slow convergence even when all component functions are strongly convex.
Furthermore, when some component functions are allowed to be nonconvex, we prove that the optimality gap of IGD can be significantly worse throughout the small epoch regime.
Our analyses reveal that the convergence properties of permutation-based SGD in the small epoch regime may vary drastically depending on the assumptions on component functions.
Lastly, we supplement the paper with tight upper and lower bounds for IGD in the large epoch regime.
\end{abstract}

\section{Introduction}
\label{sec:introduction}
Many machine learning and deep learning tasks can be formulated as finite-sum minimization problems:
\begin{align*}
    \min_{\vx \in \R^d} F(\vx) := \frac{1}{n}\sum_{i=1}^n f_i(\vx),
\end{align*}
where the objective $F(\vx)$ is the average of a finite number of component functions $f_i(\vx)$. 
In modern deep learning applications, the number of components $n$ is often extremely large, making full gradient optimization methods computationally expensive.
To address this, stochastic gradient descent (SGD) and its variants have gained attention for their computational efficiency and scalability \citep{lan2020first}.

SGD methods can be categorized based on the strategy used to select the component index $i(t)$ at iteration $t$: (1)~\textit{with-replacement} SGD, and (2)~\textit{permutation-based} SGD.
In with-replacement SGD, also known as \textit{uniform-sampling} SGD, each index is drawn independently from a uniform distribution over $\{1, 2, \dots, n\}$.
This approach has been the primary focus of theoretical studies, as it guarantees the stochastic gradient at each step to be an unbiased estimator of the gradient of the overall objective $F$ \citep{bubeck2015convex}. 

In contrast, permutation-based SGD---where indices are sampled in a shuffled order, also referred to as \textit{without-replacement} SGD or \textit{shuffling gradient methods}---is more commonly used in practice.
Its popularity arises from strong empirical performance and simplicity of implementation, making it the standard choice for real-world machine learning applications.
However, despite its widespread use, the theoretical understanding of permutation-based SGD had remained underdeveloped until recently, due to challenges arising from the lack of independence between iterates.

Nevertheless, recent advances have successfully addressed the theoretical challenges of permutation-based SGD \citep{haochen2019random,nagaraj2019sgd}.
For example, Random Reshuffling (\randr), one of the most common permutation-based methods, randomly shuffles the indices at the start of each epoch.
It has been theoretically shown that \randr{} achieves a convergence rate of $\gO(1/nK^2)$ for smooth and strongly convex objectives, which is faster than the rate $\gO(1/nK)$ of with-replacement SGD, where $K$ is the number of epochs \citep{ahn2020sgd,mishchenko2020random}.

While these results suggest that \randr{} is theoretically superior to with-replacement methods, the story is far from complete.
Existing analyses of permutation-based SGD are mostly restricted to the large epoch regime, where $K$ is sufficiently large relative to the problem's condition number $\kappa$ (defined in \Cref{subsec:ftnclass}).
However, this regime is often unrealistic in practical machine learning scenarios, especially when training large language models.
Neural network training typically involves highly ill-conditioned optimization landscapes \citep{li2018visualizing,ghorbani2019investigation}, where $\kappa$ is large, and $K$ is comparatively small due to computational constraints. 
In such cases, the small epoch regime, where $K$ is smaller than $\kappa$, becomes significantly more relevant, yet its convergence behavior remains poorly understood.

In fact, \citet{safran2021random} establish a lower bound in strongly convex objectives for \randr{}, revealing that \randr{} cannot outperform with-replacement SGD in the small epoch regime.
This highlights the need to further investigate permutation-based methods under small epoch constraints and explore whether they can outperform with-replacement SGD in such settings.
However, analyzing permutation-based SGD in the small epoch regime poses significant theoretical challenges (as explained in \Cref{subsec:ubchallenge}).
Even for Incremental Gradient Descent (\igd) \citep{bertsekas2011incremental,gurbuzbalaban2019convergence}, the simplest permutation-based SGD method where components are processed sequentially and deterministically from indices $1$ to $n$ in each epoch, its convergence behavior in this regime is not well-understood.

In this study, as an initial step toward understanding the convergence behavior of permutation-based SGD in the small epoch regime, we focus on the convergence analysis of \igd{}.
Our study presents convergence rates for both the small epoch regime and the large epoch regime, offering a result that highlights the distinct behavior of permutation-based SGD in the small epoch regime.

\subsection{Summary of Our Contributions}
\begin{table*}[ht]
\tightmargin
\vspace{-2pt}
\centering
    \begin{threeparttable}[t]
    \centering
    \caption{Summary of our results. 
    All upper bounds, except for \Cref{thm:herding-at-opt}, apply to arbitrary permutation-based SGD. 
    \Cref{thm:herding-at-opt} specifically applies to a permutation-based SGD method proposed in its theorem.
    All lower bound results apply to \igd.}
    \label{tab:theorem-comparison}
    \begin{tabular}{|>{\centering\arraybackslash}p{1cm}|>{\centering\arraybackslash}p{3.4cm}|>{\centering\arraybackslash}p{5.4cm}|>{\centering\arraybackslash}p{5.6cm}|}
    \hline
    Epoch & Component Assumption & Convergence Rate & Gradient Assumption\\
    \hline
    \hline
    \multirow{9.25}{*}{\shortstack{Small \\ \\ $K \lesssim \kappa$}} 
    & \multirow{2.75}{*}{\shortstack{Strongly Convex\\ Identical Hessian}} & $\mathcal{O}\left(\frac{G_*^2}{\mu K}\right)$, \Cref{thm:small-ub-idhess}& \footnotesize $\norm{\nabla f_i(\vx^*)} \leq G_*$ \\
    \cline{3-4}
    &  &$\Omega\left(\frac{G^2}{\mu K}\right)$, \Cref{thm:small-lb-idhess}& \footnotesize $\norm{\nabla f_i(\vx) - \nabla F(\vx)} \leq G, \forall \vx$ \\
    \cline{2-4}
    & \multirow{2.5}{*}{Strongly Convex} & $\gO\left(\frac{L^2 G_*^2}{\mu^3 K^2}\right)$, \citet{mishchenko2020random} & \footnotesize $\norm{\nabla f_i(\vx^*)} \leq G_*$\\ 
    \cline{3-4}
    & & $\Omega\left(\frac{L G^2}{\mu^2} \min\{1, \frac{L^2}{\mu^2 K^4}\}\right)$, \Cref{thm:small-lb-sc} & \fontsize{8.3}{9} $\norm{\nabla f_i(\vx) - \nabla F(\vx)} \leq G + \norm{\nabla F(\vx)}, \forall \vx$\\ 
    \cline{2-4}
    & Potentially Nonconvex & $\textstyle \Omega\left(\frac{G^2}{L} \left(1 + \frac{L}{2 \mu n K}\right)^{n}\right)$, \Cref{thm:small-lb-concave} & \fontsize{8.3}{9} $\norm{\nabla f_i(\vx) - \nabla F(\vx)} \leq G + 3 \norm{\nabla F(\vx)}, \forall \vx$\\
    \cline{2-4}
    & Strongly Convex & $\mathcal{O}\left(\frac{H^2 L^2 G_*^2}{\mu^3 n^2 K^2}\right)$, \Cref{thm:herding-at-opt}\tnote{1} & \footnotesize $\norm{\nabla f_i(\vx^*)} \leq G_*$ \\ 
    \hline
    \multirow{5.5}{*}{\shortstack{Large\\ \\ $K \gtrsim \kappa$}} 
    & \multirow{2.5}{*}{Convex} &$\mathcal{O}\left(\frac{LG_*^2}{\mu^2 K^2}\right)$, \citet{liu2024on} & \footnotesize $\frac{1}{n}\sum_{i=1}^n \norm{\nabla f_i(\vx^*)} \leq G_*$ \\
    \cline{3-4}
    & & $\Omega\left(\frac{L G^2}{\mu^2 K^2}\right)$, \Cref{thm:large-lb-idhess} & \footnotesize $\norm{\nabla f_i(\vx) - \nabla F(\vx)} \leq G, \forall \vx$ \\ 
    \cline{2-4}
    & \multirow{2.5}{*}{Potentially Nonconvex} &$\mathcal{O}\left(\frac{L^2G^2}{\mu^3 K^2}\right)$, \Cref{thm:large-ub-generalizedgrad}\tnote{2} & \fontsize{8.3}{9} $\norm{\nabla f_i(\vx) - \nabla F(\vx)} \leq G + P \norm{\nabla F(\vx)}, \forall \vx$ \\
    \cline{3-4}
    & & $\Omega\left(\frac{L^2 G^2}{\mu^3 K^2}\right)$, \Cref{thm:large-lb-concave}\tnote{3} & \fontsize{8.3}{9} $\norm{\nabla f_i(\vx) - \nabla F(\vx)} \leq G + \kappa \norm{\nabla F(\vx)}, \forall \vx$ \\
    \hline
    \end{tabular}
    \begin{tablenotes}
    \scriptsize
        \item [1] Only shows the existence of a permutation that guarantees this convergence, and $H = \tilde{\gO}(\sqrt{d})$.
        \item [2] Requires $K \gtrsim (1 + P)\kappa$.
        \hfill \hspace{25pt} \begin{minipage}{\textwidth} \raggedright
            \item [3] Requires $K \geq \max \{\kappa^3/n^2, \kappa^{3/2}\}$ and $\kappa \geq n$.
        \end{minipage}
    \tightmargin
    \tightmargin
    \tightmargin  
    \tightmargin  
    \end{tablenotes} 
    \end{threeparttable}
\end{table*}

Our analysis focuses on the setting where the objective $F$ is smooth and strongly convex, and the step size is kept constant throughout the optimization process.
We summarize our contributions as follows, where the convergence rates reflect the function optimality gap at the final iterate.
For a clear overview, we refer readers to \Cref{tab:theorem-comparison} and \Cref{fig:theory_plot}.
\begin{itemize}[leftmargin=0.5em, topsep=0pt, itemsep=0pt]
    \item In \Cref{sec:small-epoch}, we provide convergence analyses of \igd{} in the small epoch regime.
    We establish lower bound convergence rates under three scenarios (\Cref{thm:small-lb-idhess,thm:small-lb-sc,thm:small-lb-concave}): (i)~strongly convex components sharing the same Hessian, (ii)~strongly convex components, and (iii)~allowing nonconvex components.
    Additionally, we provide the upper bound convergence rates for the first two cases (\Cref{thm:small-ub-idhess}, \Cref{thm:small-ub-scvx}).
    Our results indicate that even with stronger assumptions, \igd{} remains slower than the known upper bound of with-replacement SGD.
    Furthermore, \igd{} exhibits surprisingly slow convergence even when all components are strongly convex, and the inclusion of nonconvex components further amplifies this slowdown.

    \item 
    In \Cref{subsec:small-optimal}, we study whether a suitable permutation strategy can accelerate permutation-based SGD in the small epoch regime.
    We prove that there exists a permutation such that repeatedly using it in permutation-based SGD can outperform with-replacement SGD (\Cref{thm:herding-at-opt}).
    To our knowledge, this is the first result showing the existence of a permutation-based SGD method that converges faster than with-replacement SGD in this regime. 
    
    \item In \Cref{sec:large-epoch}, we establish tight convergence rates for \igd{} in the large epoch regime.
    We derive matching lower and upper bound rates, up to polylogarithmic factors, for scenarios where all components are convex or some are nonconvex (\Cref{thm:large-lb-idhess,thm:large-lb-concave,thm:large-ub-generalizedgrad}).
    Unlike in the small epoch regime where nonconvex components significantly slow convergence, the rate gap between these two scenarios is only a factor of $\kappa$, revealing the clear distinction in the behavior of \igd{} in the small and large epoch regimes.
\end{itemize}

\tightmargin
\section{Preliminaries}
\label{sec:preliminaries}
\smallertightmargin
We start by introducing the basic notation used throughout this paper.
We use $n$ to denote the number of component functions and $K$ to denote the total number of epochs.
The Euclidean norm is denoted by $\|\cdot\|$.
For a positive integer $N \in \mathbb{N}$, we use $[N]$ to represent the set $\{1, 2, \dots, N\}$.
The symbol $q = \text{poly}(p_1, \dots, p_s)$ means that $q$ can be expressed as a finite sum of monomials of the form $p_1^{c_1} p_2^{c_2} \cdots p_s^{c_s}$, where each $c_i$ is a bounded real number (which may be negative or non-integer).
Similarly, $q = \text{polylog}(p_1, \dots, p_s)$ denotes a function expressible as $q = \sum_{(c, c_1, \dots, c_s)} \log^c \prod_{i\in[s]} p_i^{c_i}$ for bounded real $c$ and $c_i$.

Importantly, while existing works use $\gO$ and $\Omega$ (or $\tilde{\gO}$ and $\tilde{\Omega}$ to hide polylogarithmic factors) to express the growth rates of the convergence rates, we adopt the symbols $\lesssim$ and $\gtrsim$ in this paper to describe our results in better detail.
Formally, $x \lesssim y$ means that there exists a universal constant $c > 0$ such that $x \leq c \cdot y \cdot \text{polylog}(n,K,\mu,L,\dots)$ holds for the specified $n$, $K$, and other parameters; vice-versa, $x \gtrsim y$ means $x \geq c \cdot y \cdot \text{polylog}(n,K,\mu,L,\dots)$.
Unlike $\gO$ and $\Omega$, which are often used to express the asymptotic behavior of the rate as $K \rightarrow \infty$, $\lesssim$ and $\gtrsim$ here apply to all valid values of $K$.
The reason for using these symbols is that many of the upper and lower bounds in this paper are established in the small epoch regime, where the total number of epochs $K$ is explicitly bounded above by the condition number $\kappa$.

\subsection{Definitions and Assumptions}
\label{subsec:ftnclass}

We list definitions and assumptions that will be used to describe the function class.

\begin{definition}[Smoothness]
\label{def:smooth}
A differentiable function $F:\R^d \rightarrow \R$ is \textit{$L$-smooth}, for some $L > 0$, if
\begin{align*}
    \|\nabla F(\vx) - \nabla F(\vy)\| \le L\|\vx - \vy\|, \,\, \forall \vx, \vy \in \R^d.
\end{align*}
\end{definition}

\begin{definition}[Strong Convexity]
\label{def:sc}
A differentiable function $F:\R^d \rightarrow \R$ is \textit{$\mu$-strongly convex}, for some $\mu > 0$, if
\begin{align*}
    F(\vy) \geq F(\vx) + \inner{\nabla F(\vx), \vy-\vx} + \frac{\mu}{2}\|\vy-\vx\|^2
\end{align*}
for all $\vx, \vy \in \sR^d$.
If this inequality holds with $\mu = 0$, we say that $F$ is \textit{convex}.
\end{definition}

Now, we define a common assumption on the objective function used in our analyses.
\begin{assumption}[Common Assumption]
\label{ass:common}
The overall function $F: \R^d \rightarrow \R$ is $\mu$-strongly convex and each component function $f_i$ is $L$-smooth.
\end{assumption}

Additionally, we define the \textit{condition number} of $F$ as $\kappa := \frac{L}{\mu}$, which is closely related to the problem geometry.
We note that component smoothness is commonly utilized in the literature studying permutation-based SGD \citep{ahn2020sgd,mishchenko2020random,lu2022grab,liu2024on}.  

Lastly, we introduce assumptions on the gradients.
\begin{assumption}[Bounded Gradient Errors]
\label{ass:grad-generalized}
There exists constants $G \ge 0$ and $P \ge 0$ such that, for all $\vx \in \R^d$ and $i \in [n]$,
\begin{align*}
    \|\nabla f_i(\vx) - \nabla F(\vx)\| \le G + P \|\nabla F(\vx)\|.
\end{align*}
\end{assumption}
\begin{assumption}[Bounded Gradients at the Optimum]
\label{ass:grad-optimum}
There exists a constant $G_* \ge 0$ such that, for all $i \in [n]$,
the gradient norm of each component function satisfies
\begin{align*}
    \|\nabla f_i(\vx^*)\| \le G_*.
\end{align*}
\end{assumption}
Our results require either \Cref{ass:grad-generalized} or \Cref{ass:grad-optimum}.
Notably, whenever \Cref{ass:grad-generalized} holds, \Cref{ass:grad-optimum} also holds with $G_* = G$ because $\nabla F(\vx^*) = 0$.

\subsection{Algorithms}
\label{subsec:algorithms}

\begin{algorithm}[ht]
   \caption{Permutation-Based SGD}
   \label{alg:permutation}
\begin{algorithmic}
   \STATE {\bfseries Input:} Initial point $\vx_0$, Step size $\eta$, Number of epochs $K$
   \STATE Initialize $\vx_0^1 = \vx_0$
   \FOR{$k=1$ to $K$}
   \STATE Generate a permutation $\sigma_k:[n] \rightarrow [n]$
   \FOR{$i=1$ {\bfseries to} $n$}
   \STATE $\vx^k_i = \vx^k_{i-1} - \eta \nabla f_{\sigma_k(i)}(\vx^k_{i-1})$
   \ENDFOR
   \STATE $\vx^{k+1}_0 = \vx^k_n$
   \ENDFOR
   \STATE {\bfseries Output:} $\vx_n^K$
\end{algorithmic}
\end{algorithm}

We present the basic pseudocode for permutation-based SGD methods in \Cref{alg:permutation}.
At the start of $k$-th epoch, a permutation $\sigma_k: [n] \rightarrow [n]$ is generated.
The algorithm then updates the iterate according to the component functions in the order $f_{\sigma_k(1)}, f_{\sigma_k(2)}, \dots, f_{\sigma_k(n)}$.
The method by which the permutation $\sigma_k$ is generated determines the specific variant of permutation-based SGD.
Here, we describe some popular methods studied in the literature:
\begin{itemize}[leftmargin=1em, topsep=0pt, itemsep=0pt]
    \item \textbf{Incremental Gradient Descent} (\igd, \Cref{alg:igd}): Each $\sigma_k$ is the identity permutation.
    \item \textbf{Single Shuffling} (\sings): The first permutation $\sigma_1$ is drawn uniformly at random and reused for all epochs.
    \item \textbf{Random Reshuffling} (\randr): Each $\sigma_k$ is independently drawn uniformly at random in every epoch.
    \item \textbf{Gradient Balancing} (\grab{} \citep{lu2022grab}): Each $\sigma_k$ is selected based on observations at the previous epoch.
\end{itemize}

It has been widely studied that the performance guarantees vary drastically with the choice of permutation strategy \citep{mohtashami2022characterizing,lu2022general}.
In general, one might expect \igd{} to converge slowly, as the identity mapping could represent the \textit{worst-case} scenario for convergence.
In contrast, \grab{} can be faster than \igd{}, \sings{}, or \randr{} as it adaptively selects effective permutations over time.

In this paper, we derive upper bound results for \textit{arbitrary} permutation-based SGD, and the lower bound results for \igd{}.
These results allow us to characterize how much the convergence of permutation-based SGD deteriorates when permutations are chosen in the worst-case manner.
For further clarification of the relationship between the upper and the lower bound, we point readers to \Cref{subsec:appendA_minimaxexplanation}.

\subsection{What is known so far?}
\label{subsec:relatedworks}

For simplicity, in this section, we use the conventional symbols $\gO(\cdot)$ and $\Omega(\cdot)$ (even for the small epoch regime) to denote upper and lower bounds, respectively.
The symbol $\tilde{\gO}(\cdot)$ hides the dependency on polylogarithmic factors.
Convergence rates are expressed in terms of $n$, $K$, $\mu$, and $L$ to represent the optimality gap of the function.

\tightparagraph{Permutation-Based SGD.} 
Numerous studies have explored the convergence of permutation-based SGD \citep{bertsekas2011incremental,recht2012toward,haochen2019random,nagaraj2019sgd,gurbuzbalaban2019convergence,safran2020good,safran2021random,ahn2020sgd,mishchenko2020random,rajput2020closing,rajput2022permutationbased,nguyen2021unified,lu2022grab,cha2023tighter,liu2024on,cai2024last}.
Here, we summarize recent advances in the convergence analysis of the last iterate for strongly convex objectives, and we refer readers to these works for a more comprehensive understanding.

For both \randr{} and \sings{}, under the assumption of component convexity, \citet{mishchenko2020random} derive a convergence rate of $\tilde{\gO} (\frac{L^2}{\mu^3 n K^2})$.
Later, \citet{liu2024on} improve this result by a factor of $\kappa$, showing a rate of $\tilde{\gO} (\frac{L}{\mu^2 n K^2})$.
The corresponding lower bounds, $\Omega (\frac{L}{\mu^2 nK^2})$, are established by \citet{cha2023tighter} for \randr{} and \citet{safran2021random} for \sings{}, thereby fully closing the gap between the upper and lower bounds only up to polylogarithmic factors.

There are also several works that derive upper bounds applicable to arbitrary permutation-based SGD, which naturally encompass the convergence of \igd{}.
Under the assumption of component convexity, \citet{liu2024on} establish a rate of $\tilde{\gO}(\frac{L}{\mu^2 K^2})$, which is slower than the rate for \randr{} by a factor of $n$.
For the matching lower bound, \citet{safran2020good} derive a rate of $\Omega (\frac{1}{\mu K^2})$ for \igd{}, revealing a gap of $\kappa$ between the upper and lower bounds.

Recent research \citep{rajput2022permutationbased,lu2022general,mohtashami2022characterizing} has shifted toward exploring permutation-based SGD methods that go beyond \randr{}, focusing on manually selecting permutations that induce faster convergence rather than relying on random permutations.
A notable work by \citet{lu2022grab} proposes a practical permutation-based SGD algorithm called \grab{} and provides a theoretical guarantee of convergence at the rate of $\tilde{\gO} (\frac{L^2}{\mu^3 n^2 K^2})$---a strictly faster rate than \randr{}.
Later, \citet{cha2023tighter} establishes a matching lower bound, confirming that \grab{} is optimal (for low-dimensional functions).

We note that most of these works require a condition on $K$ of the form $K \ge \kappa^\alpha \log (nK)$ with $\alpha \ge 1$.

\textbf{Small Epoch Analysis.}
The convergence behavior of permutation-based in the small epoch regime was first explicitly investigated by \citet{safran2021random}.
They provide both upper and lower bounds for \randr{} and \sings{} in both the small and large epoch regimes, with the rates matching exactly up to polylogarithmic factors for quadratic objectives with commuting component Hessians.
Interestingly, in the small epoch regime, both \randr{} and \sings{} achieve a convergence rate of $\Theta (\frac{1}{\mu n K})$, equivalent to the known rate of $\tilde{\gO} (\frac{1}{\mu T})$ for with-replacement SGD, where the total number of iterations $T$ can be expressed as $nK$ in the without-replacement setting~\citep{shamir2013stochastic,liu2024revisiting}.

To the best of our knowledge, no meaningful upper bound result with a rate of $\tilde{\gO} (\frac{1}{\mu n K})$ has been established for permutation-based SGD in the small epoch regime.
This rate is of significant importance, as it corresponds to the rate for with-replacement SGD and matches the best-known lower bound for \randr{} in this regime \citep{safran2021random,cha2023tighter}.
The upper bounds provided by \citet{safran2021random} for \randr{} and \sings{} are restricted to quadratic objectives with additional assumptions, and therefore, do not differ significantly from the scenario of a $1$-dimensional quadratic objective.

Some knowledgeable readers may point to the results of \citet{mishchenko2020random}, which present the convergence rates for \randr{} without imposing any constraint on $K$.
Specifically, under component convexity, Theorem~2 of \citet{mishchenko2020random} states
\begin{align*}
    \E \bigclosed{\norm{\vx_n^K - \vx^*}^2} = \tilde{\gO} \bigopen{\exp \bigopen{-\frac{\mu K}{\sqrt{2}L}} D^2 + \frac{L}{\mu^3 n K^2}},
\end{align*}
where $D := \norm{\vx_0 - \vx^*}$.
However, we believe that Theorem~2 does not provide a tight bound in the small epoch regime for two reasons.
First, the polynomial term induces the function optimality gap of $\tilde{\gO}(\frac{L^2}{\mu^3 n K^2})$, which is slower than the lower bound rate for \randr{} by a factor of $\frac{\kappa^2}{K}$.
Second, as $K$ decreases below $\kappa$, the exponential term grows rapidly and dominates, deviating substantially from the rate of $\tilde{\gO}(\frac{1}{\mu n K})$.
While their Theorem~1 improves the exponential term by assuming a strong convexity of components, it leaves the polynomial term unchanged.
Also, a more recent result by \citet{liu2024on} (Theorem~4.6) refines the polynomial term and also improves the exponential term to $\exp{(-K/\kappa)}\frac{LD^2}{K}$.
However, since the term inside the exponential remains unchanged, this still fails to reveal a tight bound when $K$ is small.

While \citet{koloskova2024on} derive an upper bound convergence rate for permutation-based SGD that does not rely on large $K$ for nonconvex objectives, we were unable to extend their proof techniques to the strongly convex setting to yield a rate of $\tilde{\gO}(\frac{1}{\mu n K})$.

\subsection{Why Do Existing Bounds Require Large Epochs?}
\label{subsec:ubchallenge}

To understand the challenges in establishing upper bounds for permutation-based SGD, it is important to observe that, unlike with-replacement SGD, permutation-based SGD uses each component function exactly once per epoch.
Therefore, $n$ steps of update in permutation-based SGD can be expressed as an approximation of gradient descent on the overall objective, combined with a cumulative error term.
Much of the prior literature on establishing the upper bounds for permutation-based SGD focuses on capturing the ``cumulative error'' effect within a single epoch.

Technically, to show that the cumulative error within an epoch is small, the step size must be sufficiently small to ensure that the iterate does not move too far during a single epoch.
Specifically, the step size must be less than $\gO(1/nL)$, where $L$ represents the smoothness parameter~\citep{mishchenko2020random,liu2024on}.
However, when the number of epochs $K$ is small, the step size should be larger in order to bring the iterate close to the optimal point.
In fact, the step size should be at least as large as a value proportional to $1/K$.

These two requirements---the need for a small step size to control error within a single epoch and the need for a larger step size to achieve fast convergence when $K$ is small---lead to a conflict. 
Consequently, existing analyses generally hold only when $K$ is sufficiently large.
While some analyses are valid even when $K$ is small, their bounds are not tight as discussed in the previous subsection.

\section{IGD in Small Epoch Regime}
\label{sec:small-epoch}
We have highlighted that studying permutation-based SGD in the small epoch regime, where the total number of epochs $K$ satisfies $K \lesssim \kappa$, is both underexplored and highly challenging, despite its practical relevance.
As an initial step toward understanding its convergence behavior in this regime, we investigate \igd{}, the simplest and deterministic permutation-based SGD method.
We explore this regime under three distinct scenarios: (i)~each component is strongly convex with a common Hessian, (ii)~each component is strongly convex, and (iii)~some components may be nonconvex.
For each scenario, we establish a convergence lower bound and demonstrate degradation in convergence.

\subsection{Convergence Analysis of IGD}
\label{subsec:small-igd}

We introduce our first lower bound result of \igd{} in the small epoch regime.

\begin{restatable}{theorem}{thmsmalllbidhess}
\label{thm:small-lb-idhess}
For any $n \ge 2$, $\kappa \ge 2$, and $K \le \frac{1}{2}\kappa$, there exists a $3$-dimensional function $F$ satisfying \Cref{ass:common,ass:grad-generalized} with $P = 0$, where each component function shares the same Hessian, i.e., $\nabla^2 f_i (\vx) = \nabla^2 F(\vx)$ for all $i \in [n]$ and $\vx \in \R^3$, along with an initialization point $\vx_0$, such that for any constant step size $\eta$, the final iterate $\vx_n^K$ obtained by \Cref{alg:igd} satisfies 
\tightmargin
\begin{align*}
    F(\vx_n^K) - F(\vx^*) \gtrsim \frac{G^2}{\mu K}.
\end{align*}
\end{restatable}
The proof of \Cref{thm:small-lb-idhess} is presented in \Cref{subsec:small-lb-idhess}. 
Note that if all component functions share the same Hessian, they are also $\mu$-strongly convex.
To the best of our knowledge, the previous best lower bound rate for \igd{} in this setting was $\frac{G^2}{\mu K^2}$ (\citet{safran2020good}), and \Cref{thm:small-lb-idhess} improves it by a factor of $K$.
Additionally, \randr{} has a lower bound of $\frac{G^2}{\mu n K}$ (\citet{safran2021random}), and the optimal permutation-based SGD method has a lower bound of $\frac{LG^2}{\mu^2n^2K^2}$ (\citet{cha2023tighter}) in the same setting.

For with-replacement SGD, the known upper bound on the function optimality gap is $\frac{G^2}{\mu n K}$ \citep{liu2024revisiting}, which is faster than the rate $\frac{G^2}{\mu K}$ in \Cref{thm:small-lb-idhess} by a factor of $n$.
We emphasize that this comparison is made under conditions advantageous to \igd{}, as the lower bound from \Cref{thm:small-lb-idhess} assumes all component functions share the same Hessian, while the upper bound for with-replacement SGD does not require such a condition.
However, this comparison has some subtleties: the upper bound rate is derived under a varying step size scheme, leaving open the possibility that IGD can converge faster under such a scheme.
For a more complete comparison, it would be important to extend \Cref{thm:small-lb-idhess} to the varying step size setting, which we leave for future work.

Next, we present the upper bound for arbitrary permutation-based SGD methods when each component is $1$-dimensional and shares the same Hessian.

\begin{restatable}{theorem}{thmsmallubidhess}
\label{thm:small-ub-idhess}
Let $n \geq 1$, $\frac{\kappa}{n} \lesssim K \leq \kappa$, and an initialization point $x_0$. Suppose $F$ is a $1$-dimensional function satisfying \Cref{ass:common,ass:grad-optimum}.
Assume that each component function $f_i$ shares the same Hessian for all $i \in [n]$ and $x \in \R$.
Then, for any choice of permutation $\sigma_k$ in each epoch, the final iterate $x_n^K$ obtained by \Cref{alg:permutation} with the step size $\eta = \frac{1}{\mu n K} \max \bigset{\log \bigopen{\frac{L \bigabs{x_0 - x^*}}{G_*}}, 1}$ satisfies
\tightmargin
\begin{align*}
    F(x_n^K) - F(x^*) \lesssim \frac{G_*^2}{\mu K}.
\end{align*}
\end{restatable}
The proof of \Cref{thm:small-ub-idhess} is in \Cref{subsec:small-ub-idhess}.
We note that the minimum epoch requirement $K \gtrsim \frac{\kappa}{n}$ is necessary for valid analysis, as mentioned in \citet{safran2021random} (Remark~2).
While \Cref{thm:small-ub-idhess} additionally requires the objective to be $1$-dimensional due to technical challenges, the bound can be directly applied to objectives with diagonal Hessians, which aligns with the construction in the proof of \Cref{thm:small-lb-idhess}.

Indeed, the function class \Cref{thm:small-lb-idhess,thm:small-ub-idhess} apply to is restrictive.
However, our next theorem indicates that the convergence of \igd{} deteriorates immediately when the identical Hessian assumption is removed, even when each component function remains strongly convex.

\begin{restatable}{theorem}{thmsmalllbsc}
\label{thm:small-lb-sc}
For any $n \ge 3$, $\kappa \ge 2$, and $K \le \frac{1}{16\pi}\kappa$, there exists a $4$-dimensional function $F$ satisfying \Cref{ass:common,ass:grad-generalized} with $P = 1$, where each component function is $\mu$-strongly convex, along with an initialization point $\vx_0$, such that for any constant step size $\eta$, the final iterate $\vx_n^K$ obtained by running \Cref{alg:igd} satisfies
\tightmargin
\begin{align*}
    F(\vx_n^K) - F(\vx^*) \gtrsim \frac{LG^2}{\mu^2}\min\left\{1, \, \frac{\kappa^2}{K^4}\right\}.
\end{align*}
\end{restatable}
\Cref{thm:small-lb-sc} is a technically complex result, and we briefly outline the key strategy here.
We construct each component function by applying a rotation, positioning each minimizer to form a regular $n$-polygon.
The key idea is that, with a carefully chosen initialization, the iterates preserve rotational symmetry and also form a regular $n$-polygon, maintaining a constant distance from the global minimizer throughout the optimization process.
The proof of \Cref{thm:small-lb-sc} is presented in \Cref{subsec:small-lb-sc}.

Compared to \Cref{thm:small-lb-idhess}, \Cref{thm:small-lb-sc} provides a consistently larger lower bound.
Specifically, depending on the relationship between $K$ and $\sqrt{\kappa}$, the bound in \Cref{thm:small-lb-sc} is larger by a factor of either $\kappa K$ or $\kappa^3/K^3$, both exceeding $1$ in the small epoch regime.
When $K = \Theta(\kappa)$, both bounds in \Cref{thm:small-lb-idhess,thm:small-lb-sc} become $\frac{G^2}{L}$.

Theorem~5 of \citet{mishchenko2020random} provides an upper bound for \igd{} when all component functions are strongly convex.
We restate this result in \Cref{thm:small-ub-scvx}, with a slight modification to extend its applicability to arbitrary permutation-based SGD methods.

\begin{restatable}[\citet{mishchenko2020random}, Theorem~5]{proposition}{thmsmallubscvx}
\label{thm:small-ub-scvx}
Let $n \geq 1$, $K \gtrsim \frac{\kappa}{n}$, and $\vx_0$ be the initialization point.
Suppose $F$ is a function satisfying \Cref{ass:common,ass:grad-optimum} where each component function is $\mu$-strongly convex.
Then, for any choice of permutation $\sigma_k$ in each epoch, the final iterate $\vx_n^K$ obtained by running \Cref{alg:permutation} with a step size $\eta = \frac{2}{\mu n K} \max \bigset{\log \bigopen{\frac{\norm{\vx_0 - \vx^*} \mu K}{\sqrt{\kappa} G_*}}, 1}$, satisfies
\tightmargin
\begin{align*}
    \norm{\vx_n^K - \vx^*}^2 \lesssim \frac{LG_*^2}{\mu^3K^2}.
\end{align*}
\end{restatable}
The proof of \Cref{thm:small-ub-scvx} is presented in \Cref{subsec:small-ub-scvx}.
The squared distance bound in \Cref{thm:small-ub-scvx} naturally translates to a function optimality gap of $\frac{L^2G_*^2}{\mu^3 K^2}$.
Although \Cref{thm:small-lb-sc} and \Cref{thm:small-ub-scvx} do not match in general, they do align when $K=\Theta(\sqrt{\kappa})$: in this case, both bounds become the rates $\frac{LG^2}{\mu^2}$ and $\frac{LG_*^2}{\mu^2}$, achieving a tight match up to polylogarithmic factors.

Now, we present the result for the case where nonconvex components exist.
While some slowdown in convergence is expected, \Cref{thm:small-lb-concave} reveals that it is far more drastic.

\begin{restatable}{theorem}{thmsmalllbconcave}
\label{thm:small-lb-concave}
For any $n \geq 4$, $\kappa \geq 4$, and $K \leq \frac{\kappa}{4}$, there exists a $2$-dimensional function $F$ satisfying \Cref{ass:common,ass:grad-generalized} with $P=3$ such that for any constant step size $\eta$, the final iterate $\vx_n^K$ obtained by running \Cref{alg:igd} starting from the initialization point $\vx_0 = (D, 0)$ satisfies
\tightmargin
\begin{align*}
    F(\vx_n^K) - F(\vx^*) \gtrsim \min \bigset{\mu D^2, \frac{G^2}{L}\bigopen{1 + \frac{L}{2 \mu n K}}^{n}}.
\end{align*}
\end{restatable}
Our construction involves component functions that are concave in particular directions.
The proof of \Cref{thm:small-lb-concave} is presented in \Cref{subsec:small-lb-concave}. 
One distinction of this statement is the explicit inclusion of the initial distance $D$.
This dependence cannot be removed unless the initial point is placed exponentially far from the global minimum, which would lead to an unfair comparison with upper bound theorems, as they typically include a $\log D$ term in their bounds.

Roughly, an expression of the form $(1 + a)^b$ can be approximated by $\exp (ab)$.
Applying this to $(1 + \frac{\kappa}{2 nK})^{n}$, we obtain the approximation $\exp (\frac{\kappa}{2K})$.
Thus, when $K = \Theta(\kappa)$, the second term scales as $\frac{G^2}{L}$, and as $K$ decreases, it grows at a rate exponential in $\frac{\kappa}{K}$.
This contrasts with other bounds, which typically exhibit polynomial dependence on $\frac{\kappa}{K}$.

To validate our findings, we conduct experiments on our lower bound constructions in \Cref{appendix:experiments}.
For \Cref{thm:small-lb-sc}, we confirm that the iterates follow a circular trajectory, as intended by the original design.
For \Cref{thm:small-lb-concave}, we observe that the function optimality gap for \igd{} skyrockets whereas other permutation-based SGD methods remain robust in the small epoch regime.
To our knowledge, no upper bound exists for \randr{} in this setting with nonconvex components.
Based on experimental results for \Cref{thm:small-lb-concave}, we conjecture that \randr{} will theoretically exhibit robust convergence in this nonconvex component setting, unlike \igd{}.

While our lower bound results are stated in terms of the function optimality gap to align with the form of upper bounds, our proof can be directly extended to derive lower bounds in terms of the distance to the optimal solution.
Specifically, the lower bounds on the distance metric $\norm{\vx_n^K - \vx^*}$ are: 
$\frac{G}{\mu K}$, $\frac{G}{\mu}\cdot \min \bigset{1, \frac{\kappa}{K^2}}$, and $\min \bigset{D, \, \frac{G}{L} \bigopen{1+ \frac{L}{2 \mu n K}}^{n/2}}$
for \Cref{thm:small-lb-idhess,thm:small-lb-sc,thm:small-lb-concave}, respectively.

\subsection{Breaking the Barrier of With-Replacement SGD}
\label{subsec:small-optimal}

Up to this point, we have analyzed the worst-case convergence behavior of permutation-based SGD with respect to the permutation choice in the small epoch regime.
A natural question that follows is: what happens in the best case?

Unlike the large epoch regime where the question has been sufficiently explored \citep{lu2022grab,cha2023tighter}, the small epoch regime remains less understood.
In this section, we slightly deviate from the main topic and demonstrate that a well-designed permutation can enable permutation-based SGD to achieve a faster convergence rate than with-replacement SGD in the small epoch regime---which \randr{} has been proven not to do so \citep{safran2021random}.

Before presenting our finding, we introduce an additional lemma that is used in deriving our result.
\begin{lemma}[Herding Algorithm \citep{bansal2017algorithmic}]
\label{lem:herd}
Let $\vz_1, \cdots, \vz_n \in \sR^d$ satisfy $\|\vz_i\| \leq 1$ for all $i \in [n]$ and $\sum_{i=1}^n \vz_i = 0$. 
Then, there exists an algorithm, \herding{}, that outputs a permutation $\sigma:[n] \rightarrow [n]$ such that
\begin{align*}
    \textstyle \max_{i \in [n]} \norm{\, \sum_{j=1}^i \vz_{\sigma(j)} \,} \leq H, \text{ where} \, H = \tilde{\gO}(\sqrt{d}).
\end{align*}
\end{lemma}
The \herding{} algorithm was used in \citet{lu2022grab} for designing \grab{}.
Our next theorem leverages \herding{} in a different way to show the existence of a permutation-based SGD method (but impractical) that achieves acceleration even in the small epoch regime.

\begin{restatable}[Herding at Optimum]{theorem}{thmherdingatopt}
\label{thm:herding-at-opt}
Let $n \geq 1$, $K \gtrsim \frac{\kappa}{n}$, and $\vx_0$ be the initialization point.
Suppose $F$ is a function satisfying \Cref{ass:common,ass:grad-optimum} where each component function is $\mu$-strongly convex.
Then, there exists a permutation $\sigma$ such that the final iterate $\vx_n^K$ obtained by running \Cref{alg:permutation} with $K$ epochs of $\sigma$ and a step size $\eta = \frac{2}{\mu n K} \max \bigset{\log \bigopen{\frac{\norm{\vx_0 - \vx^*} \mu n K}{\sqrt{\kappa} H G_*}}, 1}$, satisfies
\tightmargin
\begin{align*}
    \norm{\vx_n^K - \vx^*}^2 \lesssim \frac{H^2LG_*^2}{\mu^3n^2K^2}.
\end{align*}
\end{restatable}
Unlike \grab{} which dynamically adapts the permutation at each epoch based on the gradient observations, \Cref{thm:herding-at-opt} applies a fixed $\sigma$ consistently throughout entire epochs.
The permutation $\sigma$ is obtained by running \herding{} for the scaled component gradients at the global optimum $\vx^*$, ensuring $\max_{i \in [n]} \| \, \sum_{j=1}^i \nabla f_{\sigma(j)} (\vx^*) \, \| \leq HG_*$.
The proof of \Cref{thm:herding-at-opt} is presented in \Cref{subsec:herding-at-opt}.

By $L$-smoothness, it immediately follows that the function optimality gap is bounded as $F(\vx^K_n) - F(\vx^*) \lesssim \frac{H^2 L^2 G_*^2}{\mu^3 n^2 K^2}$.
We make two key observations regarding this result.
First, \citet{cha2023tighter} prove the lower bound rate of $\frac{LG^2}{\mu^2 n^2 K^2}$ applicable to arbitrary permutation-based SGD without any constraint on $K$.
This confirms that \Cref{thm:herding-at-opt} achieves optimal performance in terms of $n$ and $K$ among permutation-based SGD methods.
Second, this rate outperforms the rate of $\frac{G^2}{\mu n K}$ for with-replacement SGD \citep{liu2024revisiting} whenever $n \geq H^2 \kappa^2/K$.
In particular, even when $K \lesssim \kappa$, problems involving a large number of component functions with small input dimensions can still satisfy this condition.

To our knowledge, this is the first result showing that a permutation-based SGD method may outperform with-replacement SGD in the small epoch regime.
However, we identify two key limitations.
First, \Cref{thm:herding-at-opt} is not an implementable algorithm, as it requires prior knowledge of component gradients at $\vx^*$.
Second, the upper bound in \Cref{thm:herding-at-opt} and the lower bound established by \citet{cha2023tighter} still differ by a factor of $H^2 \kappa$.
An interesting future direction would be to design a practical permutation-based SGD method that tightly matches this lower bound.

We conclude by suggesting a setting where we can efficiently obtain this permutation. 
Suppose all component functions have the same Hessian so that $\nabla^2 f_i - \nabla^2 F \equiv 0$. 
Then, the gradient difference $\nabla f_i - \nabla F$ remains constant across the domain, leading to the following equation:
\begin{align*}
    \nabla f_i(\vx^*) = \nabla f_i(\vx^*) - \nabla F(\vx^*) = \nabla f_i(\vx_0) - \nabla F(\vx_0).
\end{align*}
In this scenario, we can use scaled gradient errors at the initialization $(\nabla f_i(\vx_0) - \nabla F(\vx_0))/G_*$, which can be efficiently obtained, as inputs to \herding{} to attain the desired permutation $\sigma$. 
Furthermore, the lower bound construction of \citet{cha2023tighter}, which achieves a rate of $\frac{LG^2}{\mu^2 n^2 K^2}$, also satisfies the identical Hessian assumption.
This confirms the algorithmically optimal convergence for this specific function class, up to a factor of $H^2 \kappa$ gap. 

\section{IGD in Large Epoch Regime}
\label{sec:large-epoch}
\smallertightmargin
We now shift focus to the \textit{large epoch regime}, where $K \gtrsim \kappa$.
We examine convergence under two distinct scenarios: (i)~each component is convex, and (ii)~some components may be nonconvex.
While the presence of nonconvex components significantly deteriorates convergence in the small epoch regime, we observe that this effect diminishes in the large epoch regime.

\subsection{Convergence with Component Convexity}
\label{subsec:large-convex}

We first focus on the case where all components are convex.

\begin{restatable}{theorem}{thmlargelbidhess}
\label{thm:large-lb-idhess}
For any $n \geq 2$, $\kappa \geq 2$, and $K \geq \kappa$, there exists a $3$-dimensional function $F$ satisfying \Cref{ass:common,ass:grad-generalized} with $P = 0$, 
where each component function shares the same Hessian,
along with an initialization point $\vx_0$, such that for any constant step size $\eta$, the final iterate obtained by running \Cref{alg:igd} satisfies
\tightmargin
\begin{align*}
    F(\vx_n^K) - F(\vx^*) \gtrsim \frac{LG^2}{\mu^2K^2}.
\end{align*}
\end{restatable}
The detailed proof of the theorem is provided in \Cref{subsec:large-lb-idhess}.
As previously discussed in \Cref{thm:small-lb-idhess}, since the overall function is strongly convex and each component function shares the same Hessian, it follows that each component function is also $\mu$-strongly convex.
The previous best lower bound for \igd{} in this setting was $\frac{G^2}{\mu K^2}$ \citep{safran2020good}, and our result improves upon this by a factor of $\kappa$.
Also, when $K = \Theta(\kappa)$, this bound simplifies to $\frac{G^2}{L}$, thereby continuously interpolating the lower bound results in the small epoch regime (\Cref{thm:small-lb-idhess,thm:small-lb-sc,thm:small-lb-concave}).

Next, we present a complementary upper bound result, originally established in Theorem~4.6 of \citet{liu2024on}.
For consistency with the assumptions used throughout this paper, we restate it under slightly stronger assumptions.

\begin{restatable}[\citet{liu2024on}, Theorem 4.6]{proposition}{largeubavg}
\label{thm:large-ub-avg}
Let $n \geq 1$, $K\gtrsim \kappa$, and $\vx_0$ be the initialization point. Suppose $F$ is a function satisfying \Cref{ass:common,ass:grad-optimum} where each component function is convex.
Then, for any choice of permutation $\sigma_k$ in each epoch, the final iterate obtained by \Cref{alg:permutation} with the step size $\eta = \frac{1}{\mu n K} \max \bigset{\log \bigopen{\frac{\|\vx_0 - \vx^*\|^2 \mu^3 K^2}{LG^2 (1 + \log K)}}, 1}$ satisfies
\tightmargin
\begin{align*}
    F(\vx_n^K) - F(\vx^*) \lesssim \frac{LG_*^2}{\mu^2K^2}.
\end{align*}
\end{restatable}
We observe that the lower bound in \Cref{thm:large-lb-idhess} and the upper bound in \Cref{thm:large-ub-avg} match exactly, up to polylogarithmic factors.
The component functions for the lower bound satisfy strictly stronger assumptions than those required for the upper bound.
Unlike upper bounds where stronger assumptions may improve the convergence rate, fulfilling stronger assumptions in lower bound analyses rather strengthens the result of the bound.
Thus, \Cref{thm:large-lb-idhess} remains a valid lower bound matching \Cref{thm:large-ub-avg}.

\subsection{Convergence without Component Convexity}
\label{subsec:large-withoutconvex}

In this section, we investigate the case where the assumption of component convexity is removed.
Our next theorem, \Cref{thm:large-lb-concave}, establishes a lower bound for \igd{}, quantifying the degradation in convergence rate when nonconvex components are included in the large epoch setting.

\begin{restatable}{theorem}{thmlargelbconcave}
\label{thm:large-lb-concave}
For any $n \geq 4$, $\kappa \geq n$, and $K \geq \max \bigset{\kappa^3/n^2, \, \kappa^{3/2}}$, there exists a $4$-dimensional function $F$ satisfying \Cref{ass:common,ass:grad-generalized} with $P = \kappa$, along with an initialization point $\vx_0$, such that for any constant step size $\eta$, the final iterate obtained by running \Cref{alg:igd} satisfies
\tightmargin
\begin{align*}
    F(\vx_n^K) - F(\vx^*) \gtrsim \frac{L^2G^2}{\mu^3K^2}.
\end{align*}
\end{restatable}
\tightmargin
The proof of \Cref{thm:large-lb-concave} is presented in \Cref{subsec:large-lb-concave}. 
Additional assumptions on $n$ and $K$ are introduced for technical reasons.
Since the construction in \Cref{thm:large-lb-concave} involves nonconvex components, \Cref{thm:large-ub-avg} is no longer applicable for direct comparison.
\Cref{thm:large-ub-generalizedgrad} addresses this by providing an upper bound allowing nonconvex component functions for arbitrary permutation-based SGD.

\begin{restatable}{theorem}{thmlargeubgeneralizedgrad}
\label{thm:large-ub-generalizedgrad}
Let $n \geq 1$, $K \gtrsim (1+P)\kappa$, and $\vx_0$ be the  initialization point.
Suppose $F$ is a function satisfying \Cref{ass:common,ass:grad-generalized}. 
Then, for any choice of permutation $\sigma_k$ in each epoch, the final iterate $\vx_n^K$ obtained by \Cref{alg:permutation} with a step size $\eta = \frac{2}{\mu n K} \max \bigset{\log\bigopen{\frac{(F(\vx_0) - F(\vx^*))\mu^3K^2}{L^2G^2}}, 1}$ satisfies
\tightmargin
\begin{align*}
    F(\vx_n^K) - F(\vx^*) \lesssim \frac{L^2G^2}{\mu^3K^2}.
\end{align*}
\end{restatable}
The proof of \Cref{thm:large-ub-generalizedgrad} is in \Cref{subsec:large-ub-generalizedgrad}.
This upper bound aligns with the lower bound in \Cref{thm:large-lb-concave}, differing only by polylogarithmic factors, when the objective is sufficiently ill-conditioned and the number of epochs $K$ is sufficiently large, specifically $K \gtrsim \max \bigset{\kappa^3/n^2, \kappa^2}$.

Importantly,  the convergence rate in this setting degrades by only a factor of $\kappa$ compared to the convex components case.
These results highlight an \textit{intriguing} behavior of \igd{}: allowing nonconvex components significantly degrades convergence in the small epoch regime; however, this slowdown is much less severe in the large epoch regime.

Similar to the small epoch case, the lower bounds in the large epoch regime can also be expressed in terms of the distance to the optimum.
Specifically, the lower bounds on $\norm{\vx_n^K - \vx^*}$ are $\frac{G}{\mu K}$ and $\frac{LG}{\mu^2 K}$ (for $K \ge \max \{\kappa^3/n^2, \kappa^2\}$) for \Cref{thm:large-lb-idhess,thm:large-lb-concave}, respectively.

\subsection{Comparison with Other Methods}
\label{subsec:large-comparison}

Here, we provide a detailed comparison of the convergence rates across different permutation-based SGD methods.

\tightparagraph{Random Reshuffling.}
In the large epoch regime, \citet{liu2024on} show that \randr{} achieves an upper bound of $\frac{L G_*^2}{\mu^2 n K^2}$, while \citet{cha2023tighter} establish a tight matching lower bound under the same setting.
Both results assume that the component functions are convex.
This implies that in settings where all component functions are convex, \randr{} outperforms \igd{} by a factor of $n$ in terms of convergence rate in the large epoch regime.

\tightparagraph{Optimal Permutation-based SGD.}
\citet{lu2022grab} demonstrate that \grab{} achieves an upper bound of $\frac{H^2 L^2 G^2}{\mu^3 n^2 K^2}$, where $H$ is a constant that scales as $\sqrt{d}$ (\Cref{lem:herd}).
Similarly, \citet{cha2023tighter} establish a lower bound of $\frac{L^2 G^2}{\mu^3 n^2 K^2}$, for any permutation strategy over $K$ epochs, assuming sufficiently ill-conditioned problems and a large number of epochs.
Both results are derived without assuming component convexity.
Together, these results indicate that, when nonconvex components exist and $d$ is fixed, the optimal convergence rate for permutation-based SGD in the large epoch regime is $\frac{L^2 G^2}{\mu^3 n^2 K^2}$.
This implies that in settings where some components are nonconvex, \igd{} converges at a rate slower than optimal permutation-based SGD by a factor of $n^2$.

\section{Conclusion}
We provide a detailed analysis of \igd{} across both small and large epoch regimes, considering various assumptions on the component functions.
Our results show that, unlike in the large epoch regime, even when the component functions are strongly convex, the convergence can be significantly slow. 
Furthermore, the presence of nonconvex components exacerbates this slowdown exponentially.
We also demonstrate the existence of a permutation-based SGD method that allows faster convergence in the small epoch regime.

Finally, we highlight two promising directions for future work.
The first is to establish a tight convergence bound for \randr{} in the small epoch regime, similar to our analysis for \igd{} in \Cref{sec:small-epoch}.
We discuss the current state of research and the key challenges in this direction in \Cref{subsec:appendA_rrstatus}.
The second is to develop an efficient and practical permutation-based SGD method that enjoys provable fast convergence in this regime.

\section*{Acknowledgements}
This work was partly supported by a National Research Foundation of Korea (NRF) grant funded by the Korean government (MSIT) (No.\ RS-2023-00211352) and an Institute for Information \& communications Technology Planning \& Evaluation (IITP) grant funded by the Korean government (MSIT) (No.\ RS-2019-II190075, Artificial Intelligence Graduate School Program (KAIST)). CY acknowledges support from a grant funded by Samsung Electronics Co., Ltd.

\section*{Impact Statement}

This paper aims to advance the theoretical understanding of convex optimization in machine learning. 
While optimization methods have broad applications, we do not foresee any specific ethical concerns or societal implications arising directly from this work.

\bibliography{bibliography}
\bibliographystyle{icml2025}

\newpage
\appendix
\onecolumn
\allowdisplaybreaks
\section{Supplementary Details}
\label{appendix:supplementary}

In this section, we provide additional details omitted from the main text.

\subsection{Visualization of Upper and Lower Bounds}
\label{subsec:appendA_plot}
We begin by presenting a plot that summarizes our theoretical findings.
Theorems~\ref{thm:small-lb-idhess} to \ref{thm:small-lb-concave} (and the result from \citet{mishchenko2020random}) apply to the small epoch regime ($K \lesssim \kappa$), and Theorems~\ref{thm:large-lb-idhess} to \ref{thm:large-ub-generalizedgrad} (and the result from \citet{liu2024on}) apply to the large epoch regime ($K \gtrsim \kappa$).
In the figure, solid lines indicate upper bounds and dash-dot lines represent lower bounds.
Each color represents a pair of upper and lower bounds derived under similar assumptions---what we refer to as \emph{matching bounds}.
The vertical line at $K = \kappa$ marks the transition between the small and large epoch regimes.
Both axes are log-scaled for better visualization of rate differences.

\tightmargin
\tightmargin
\begin{figure}[ht]
    \centering
    \includegraphics[width=\linewidth]{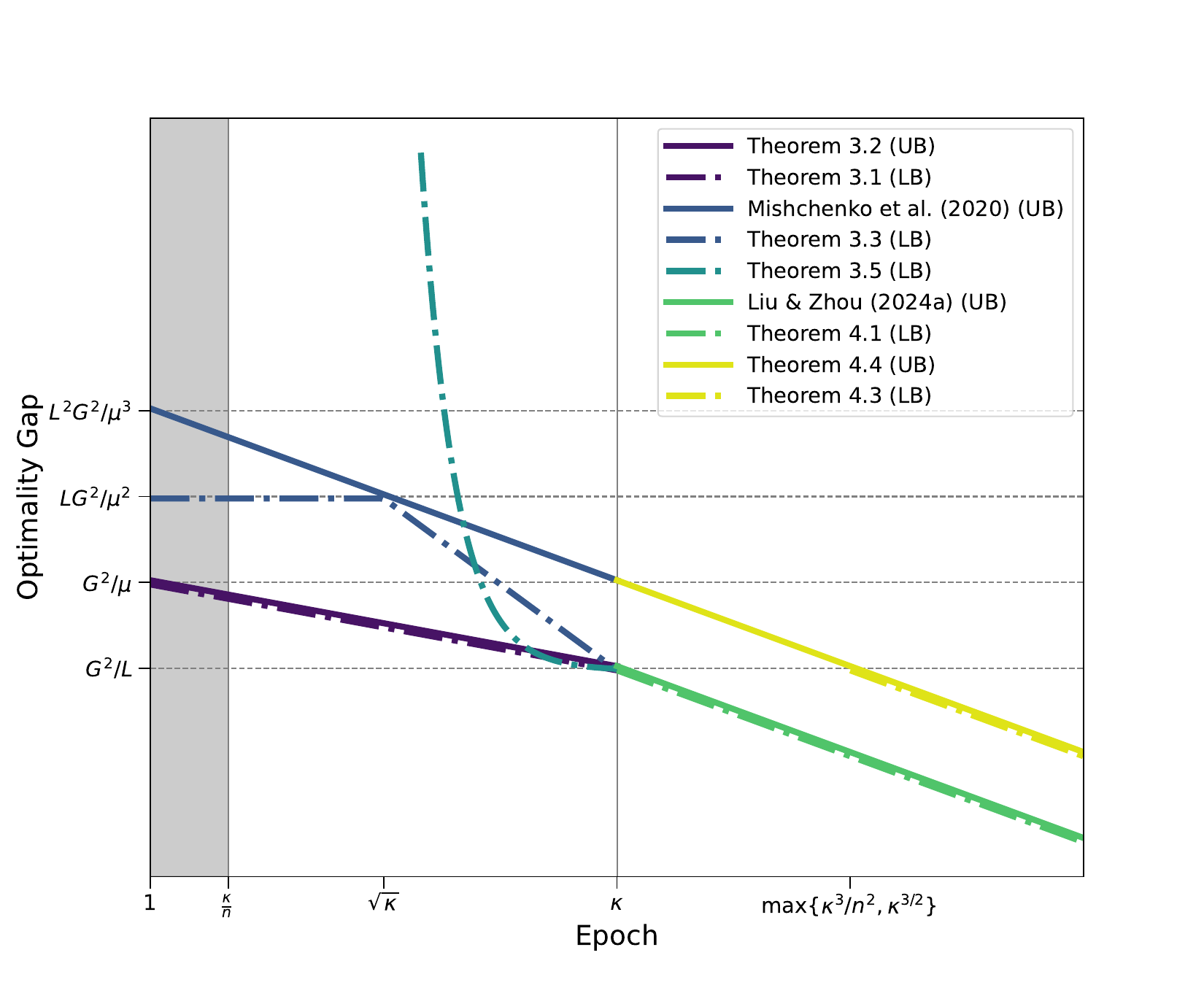}
    \tightmargin
    \tightmargin
    \caption{Visualization of the bounds in \Cref{tab:theorem-comparison}.
    Both axes are log-scaled.
    Upper bounds (UB) are represented using a solid line, and lower bounds (LB) are depicted with a dash-dot line. 
    The small and large epoch results are combined into a single figure with a separation by the vertical line $K = \kappa$. 
    Upper bound results for the small epoch regime only hold under $K \gtrsim \kappa/n$, while lower bound results hold for $K$ greater than some constant.}
    \label{fig:theory_plot}
\end{figure}
\tightmargin
\tightmargin

\subsection{Pseudeocode of \igd{}}
\label{subsec:appendA_pseudeocode}

Next, we provide the pseudocode for \igd{} as \Cref{alg:igd}.
\begin{algorithm}[H]
   \caption{Incremental Gradient Descent}
   \label{alg:igd}
\begin{algorithmic}
   \STATE {\bfseries Input:} Initial point $\vx_0$, Step size $\eta$, Number of epochs $K$
   \STATE Initialize $\vx_0^1 = \vx_0$
   \FOR{$k=1$ to $K$}
   \FOR{$i=1$ {\bfseries to} $n$}
   \STATE $\vx^k_i = \vx^k_{i-1} - \eta \nabla f_i(\vx^k_{i-1})$
   \ENDFOR
   \STATE $\vx^{k+1}_0 = \vx^k_n$
   \ENDFOR
   \STATE {\bfseries Output:} $\vx_n^K$
\end{algorithmic}
\end{algorithm}

\subsection{Connecting \igd{} Lower Bounds with Upper Bounds of General Permutation-Based SGD}
\label{subsec:appendA_minimaxexplanation}

In this paper, we derive upper bound results for \textit{arbitrary} permutation-based SGD, and the lower bound results for \igd{}.
To clarify the connection between these results, we explain why the lower bounds derived for \igd{} are relevant to the upper bounds for arbitrary permutation-based SGD.
Intuitively, deriving the upper bound for arbitrary permutation-based SGD can be viewed as bounding the following inf-sup problem from above:
\begin{align}
    \inf_{\text{step size } \eta} \sup_{\substack{\text{function } F(\vx) \\ \text{permutation } \{\sigma_k\}_{k=1}^K}} F(\vx_n^K) - F(\vx^*). \label{eq:persgdub}
\end{align}
On the other hand, the corresponding lower bound is one that bounds the following sup-inf problem from below:
\begin{align}
    \sup_{\substack{\text{function } F(\vx) \\ \text{permutation } \{\sigma_k\}_{k=1}^K}} \inf_{\text{step size } \eta} F(\vx_n^K) - F(\vx^*). \label{eq:persgdlb}
\end{align}
Notably, in the lower bound formulation, the permutations $\{\sigma_k\}_{k=1}^K$ appear in the supremum term.
This implies that the lower bound for \igd{}, which can be formulated as:
\begin{align}
    \sup_{\text{function } F(\vx)} \inf_{\text{step size } \eta} F(\vx_n^K) - F(\vx^*), \label{eq:lgdlb}
\end{align}
where every $\sigma_k$ is an identity mapping, is at most \cref{eq:persgdlb}.
Therefore, our lower bound results, derived specifically for \igd{}, also provide valid lower bounds for the upper bound results established for any permutation-based SGD.

To further clarify, we compare it with the work of \citet{lu2022grab}.
In \citet{lu2022grab}, the authors introduce a permutation-based SGD algorithm called \grab{} that provably converges faster by carefully selecting permutations at each epoch.
This problem can be formulated as bounding the following inf-sup problem:
\begin{align}
    \inf_{\substack{\text{step size } \eta\\ \text{permutation } \{\sigma_k\}_{k=1}^K}} \sup_{\text{function } F(\vx)} F(\vx_n^K) - F(\vx^*). \label{eq:grabub}
\end{align}

In addition, \citet{cha2023tighter} proves that \grab{} is an optimal permutation-based SGD by providing a lower bound that holds for every possible combination of permutations over $K$ epochs:
\begin{align}
    \sup_{\text{function } F(\vx)} \inf_{\substack{\text{step size } \eta\\ \text{permutation } \{\sigma_k\}_{k=1}^K}}  F(\vx_n^K) - F(\vx^*). \label{eq:grablb}
\end{align}

Clearly, \cref{eq:grabub} and \cref{eq:grablb} are smaller than \cref{eq:persgdub} and \cref{eq:persgdlb}, respectively.

\subsection{Status and Open Challenges in Establishing Tight Bounds for \randr{} in the Small Epoch Regime}
\label{subsec:appendA_rrstatus}

We begin by summarizing the current state of research on \randr{} in the small epoch regime. To the best of our knowledge, there are two noteworthy results (under the assumption that the overall function is strongly convex and each component is smooth):

\begin{enumerate}
    \item \citep{mishchenko2020random}: When all component functions are also strongly convex, an upper bound of $\tilde{\gO}(\frac{L^2}{\mu^3 n K^2})$ is provided.
    \item \citep{safran2021random}: When all component functions are quadratic and their Hessians commute, a tight convergence rate of $\Theta(\frac{1}{\mu n K})$ is established.
\end{enumerate}

Unlike scenario (2) where the authors provide matching UB and LB (up to polylogarithmic factor), the lower bound in scenario (1) is unknown, and it remains open whether the rate $\tilde{\gO}(\frac{L^2}{\mu^3 n K^2})$ can be improved or not.

Given this context, there are two clear directions for future exploration in small epoch \randr{} literature:

\begin{itemize}
    \item \textbf{Upper Bound Direction.} Improve the existing bound of $\tilde{\gO}(\frac{L^2}{\mu^3 n K^2})$ under the strongly convex component assumption, or derive new bounds under weaker assumptions (e.g., convexity, or even without convexity).
    \item \textbf{Lower Bound Direction.} Develop a matching lower bound (under the strongly convex component case) to close the gap with the existing upper bound $\tilde{\gO}(\frac{L^2}{\mu^3 n K^2})$.
\end{itemize}

The primary challenge on the upper bound side is that deriving new upper bounds in the small epoch regime appears to require sophisticated analytical techniques (due to challenges discussed in \Cref{subsec:ubchallenge}). 
As can be found in \citet{safran2021random}, even the proof for 1D quadratic is highly technical. 
One promising technique we explored is from \citet{koloskova2024on}. In contrast to traditional analyses that group updates within a single epoch (i.e., chunks of size $n$), this method groups updates into chunks of size $\tau:=1/\eta L$. 
While this chunk-based approach can be successfully applied to derive upper bounds for \igd{}, it becomes problematic for \randr{}.
Specifically, when the chunk size $\tau$ does not align neatly within epochs, handling the dependencies between iterates becomes extremely difficult.

Regarding the lower bound direction, we believe any progress beyond current results will likely require more complicated constructions that go beyond simple quadratic functions. 
This is because for simple quadratic functions where the Hessians commute with each other (e.g., $f_i(x_1, x_2) = \frac{L}{2}x_1^2 + a_i x_1 + \frac{\mu}{2}x_2^2 + b_i x_2$), the tight rate of $\Theta(\frac{1}{\mu n K})$ is already established by \citet{safran2021random}. Therefore, to surpass the existing LB barrier $\Omega(\frac{1}{\mu n K})$, future constructions must involve quadratic functions with non-commuting Hessians or even non-quadratic functions, necessitating more advanced analytical techniques. 
While our own lower bound construction in \Cref{thm:small-lb-sc} is based on quadratic functions with non-commuting Hessians, it is tailored to \igd{}, and we do not see a clear way to extend this idea to \randr{}.
\newpage
\section{Proofs for Small Epoch Lower Bounds}
\label{appendix:small-epoch-lb}

In this section, we present the detailed proofs for \Cref{thm:small-lb-idhess,thm:small-lb-sc,thm:small-lb-concave} which are the lower bound results in the small epoch regime.
To establish these results, we construct a specific function $F$ that achieves the stated lower bound for each theorem.
We note that constructing a lower-bound function for SGD presents a significant challenge, as it must exhibit poor convergence for any choice of step size $\eta$.
The difficulty lies in the fact that the convergence behavior of SGD is highly sensitive to $\eta$: a small step size leads to slow updates, whereas a large step size can cause divergence. 

To overcome this challenge, we partition the positive real line of possible step sizes into three regimes: small, moderate, and large. 
For each regime, we design a distinct lower-bound function tailored to follow the stated convergence behavior within that range. 
Finally, we combine these functions across dimensions, ensuring that the resulting function satisfies the stated lower bound for any choice of $\eta$.
This ``dimension-aggregating'' technique has been developed in the recent literature (e.g., \citet{safran2021random,yun2022minibatch,cha2023tighter}).

\subsection{Proof of \Cref{thm:small-lb-idhess}}
\label{subsec:small-lb-idhess}
\thmsmalllbidhess*

\begin{proof}
We divide the range of step sizes $\eta > 0$ into three regimes that will be specified subsequently. For each regime, we construct the overall functions $F_1$, $F_2$, and $F_3$ respectively, along with their respective component functions and an initial point.
Each function is $1$-dimensional and satisfies \Cref{ass:common}.
$F_1$ and $F_3$ satisfy
\Cref{ass:grad-generalized} with $G=P=0$, and $F_2$ satisfies with $P=0$.
Also, the component functions within each overall function share the same Hessian. 
Importantly, each function is designed to satisfy the following properties:
\begin{itemize}
    \item (Small step size regime) There exists an initialization point $x_0 = \text{poly}(\mu, L, n, K, G)$ such that for any choice of $\eta \in \bigopen{0, \frac{1}{\mu n K}}$, the final iterate $x_n^K$ obtained by running \Cref{alg:igd} satisfies $F_1(x_n^K) - F_1(x^*) \gtrsim \frac{G^2}{\mu K}$.
    \item (Moderate step size regime) There exists an initialization point $y_0 = \text{poly}(\mu, L, n, K, G)$ such that for any choice of $\eta \in \left[\frac{1}{\mu n K}, \frac{2}{L}\right)$, the final iterate $y_n^K$ obtained by running \Cref{alg:igd} satisfies $F_2(y_n^K) - F_2(y^*) \gtrsim \frac{G^2}{\mu K}$.
    \item (Large step size regime) There exists an initialization point $z_0 = \text{poly}(\mu, L, n, K, G)$ such that for any choice of $\eta \in \left[\frac{2}{L}, \infty\right)$, the final iterate $z_n^K$ obtained by running \Cref{alg:igd} satisfies $F_3(z_n^K) - F_3(z^*) \gtrsim \frac{G^2}{\mu K}$.
\end{itemize}
Here, $x^*$, $y^*$, $z^*$ denote the minimizers of $F_1$, $F_2$, and $F_3$, respectively.
Detailed constructions of $F_1$, $F_2$, and $F_3$, as well as the verification of the assumptions and the stated properties are presented in \Cref{subsubsec:small-lb-idhess-f1,subsubsec:small-lb-idhess-f2,subsubsec:small-lb-idhess-f3}.

We now aggregate these functions across dimensions: $F(\vx) := F(x, y,z) = F_1(x) + F_2(y) + F_3(z)$ and $f_i(\vx) = f_{1i}(x) + f_{2i}(y) + f_{3i}(z)$ for all $i \in [n]$.
Here, $f_i$, $f_{1i}$, $f_{1i}$, $f_{3i}$ denote the $i$-th component function of $F$, $F_1$, $F_2$, and $F_3$, respectively.
Since each dimension is independent, it is obvious that $\vx^* = (x^*, y^*, z^*)$ minimizes $F$.

Finally, by choosing the initialization point as $\vx_0 = (x_0, y_0, z_0)$, the final iterate $\vx_n^K = (x_n^K, y_n^K, z_n^K)$ obtained by running \Cref{alg:igd} on $F$ satisfies
\begin{align*}
    F(\vx_n^K) - F(\vx^*) \gtrsim \frac{G^2}{\mu K},
\end{align*}
regardless of the choice of $\eta > 0$.

Note that $F$ satisfies the stated assumptions as
\begin{align*}
    \mu \mI \preceq \min \{ \nabla^2 F_1(x), \nabla^2 F_2(y), \nabla^2 F_3(z) \} \preceq \nabla^2 F(\vx) \preceq \max \{ \nabla^2 F_1(x), \nabla^2 F_2(y), \nabla^2 F_3(z) \} \preceq L \mI,
\end{align*}
and
\begin{align*}
    \bignorm{\nabla f_i(\vx) - \nabla F(\vx)} \le \bignorm{\nabla f_{1i}(x) - \nabla F_1(x)} + \bignorm{\nabla f_{2i}(y) - \nabla F_2(y)} + \bignorm{\nabla f_{3i}(z) - \nabla F_3(z)} \le 0 + G + 0 = G.
\end{align*}
Moreover, $\nabla^2f_i(\vx) = \mathrm{diag}(\nabla^2 f_{1i}(x), \nabla^2 f_{2i}(y), \nabla^2 f_{3i}(z)) = \mathrm{diag}(\nabla^2 f_i(x), \nabla^2 f_i(y), \nabla^2 f_i(z)) = \nabla^2 F(\vx)$ holds, since the component functions within each overall function share the same diagonal Hessian.
This concludes the proof of \Cref{thm:small-lb-idhess}.
\end{proof}

In the following subsections, we present the specific construction of $F_1$, $F_2$, and $F_3$, and demonstrate that each satisfies the stated lower bound within its corresponding step size regime.
For simplicity of notation, we omit the index of the overall function when referring to its component functions, e.g., we write $f_i(x)$ instead of $f_{1i}(x)$. 
Moreover, we use the common variable notation $x$ while constructing functions for each dimension, though we use different variables in the ``dimension-aggregation'' step.

\subsubsection{Construction of $F_1$}
\label{subsubsec:small-lb-idhess-f1}

Let $F_1(x) = \frac{\mu}{2}x^2$ with component functions $f_i(x) = F_1(x)$ for all $i \in [n]$.
It is clear that $F_1$ satisfies \Cref{ass:common}, \Cref{ass:grad-generalized} with $P=0$, and its component functions share an identical Hessian.
Also, we note that $x^* = 0$ and $F_1(x^*) = 0$.

Let the initialization be $x_0 = \frac{G}{\mu \sqrt{K}}$.
For all $\eta \in \bigopen{0, \frac{1}{\mu n K}}$, the final iterate is given by
\begin{align*}
    x_n^K = (1 - \eta \mu)^{nK} x_0 \geq \bigopen{1 - \frac{1}{nK}}^{nK} x_0 \geq \frac{G}{4 \mu \sqrt{K}},
\end{align*}
where the last inequality uses the fact that $(1 - \frac{1}{m})^m \geq \frac{1}{4}$ for all $m \geq 2$.

Thus, we have
\begin{align*}
    F_1(x_n^K) - F_1(x^*) = \frac{\mu}{2} (x_n^K)^2 \gtrsim \frac{G^2}{\mu K}.
\end{align*}

\subsubsection{Construction of $F_2$}
\label{subsubsec:small-lb-idhess-f2}

We construct the function by dividing the cases by the parity of $n$. We first consider the case where $n$ is even, and address the case where $n$ is odd later in this subsection.
Let $F_2(x) = \frac{\mu K}{2}x^2$ with component functions
\begin{align*}
    f_i(x) = \begin{cases}
        \frac{\mu K}{2}x^2 + Gx & \textrm{ if } \, i \leq n/2,\\
        \frac{\mu K}{2}x^2 - Gx & \textrm{ otherwise.}
    \end{cases}
\end{align*}
It is clear that $f_i$ satisfies \Cref{ass:grad-generalized} with $P=0$ and shares the same Hessian.
From the assumption $K \leq \frac{1}{2}\kappa$, we have $\mu \leq \mu K \leq \frac{L}{2}$.
Hence, each $f_i$ is $L$-smooth and $\mu$-strongly convex.
Also, we note that $x^* = 0$ and $F_2(x^*) = 0$.

By \Cref{lem:quadratic-twotype-IGD-closed}, the final iterate obtained by running \Cref{alg:igd} is given by
\begin{align}
    x_n^K = \bigopen{1 - \eta \mu K}^{nK}x_0 + \frac{G}{\mu K} \cdot \frac{1 - (1-\eta \mu K)^{\frac{n}{2}}}{1 + (1-\eta \mu K)^{\frac{n}{2}}} \bigopen{1 - \bigopen{1 - \eta \mu K}^{nK}}.\label{eq:small-idhess-lb}
\end{align}

For any $\eta \in \left[\frac{1}{\mu n K}, \frac{2}{L}\right)$, it follows that $\frac{1}{n} \le \eta \mu K < \frac{2 \mu K}{L} = \frac{2K}{\kappa} \le 1$.
Then, we have $(1 - \eta \mu K)^{nK} \leq \bigopen{1 - \frac{1}{n}}^{nK} \leq e^{-K} \leq e^{-1}$ which implies $1 - (1 - \eta \mu K)^{nK} \geq 1 - e^{-1}$.
Moreover, we have $\bigopen{1 - \eta \mu K}^{\frac{n}{2}} \leq (1 - \frac{1}{n})^{\frac{n}{2}} \leq e^{-\frac{1}{2}}$ and thus,
\begin{align*}
    \frac{1 - (1-\eta \mu K)^{\frac{n}{2}}}{1 + (1-\eta \mu K)^{\frac{n}{2}}} \geq \frac{1 - e^{-\frac{1}{2}}}{2}.
\end{align*}

Substituting these inequalities into \cref{eq:small-idhess-lb} and setting $x_0 = 0$, we obtain
\begin{align*}
    x_n^K \geq \frac{\bigopen{1 - e^{-1}}\bigopen{1 - e^{-\frac{1}{2}}}G}{2 \mu K},
\end{align*}
and
\begin{align*}
    F_2(x_n^K) - F_2(x^*) = \frac{\mu K}{2} (x_n^K)^2 \gtrsim \frac{G^2}{\mu K}.
\end{align*}

We now consider the case where $n$ is odd.
Let $F_2(x) = \frac{\mu K}{2}x^2$ with component functions
\begin{align*}
    f_i(x) = \begin{cases}
        \frac{\mu K}{2}x^2 & \textrm{ if } \, i=1,\\
        \frac{\mu K}{2}x^2 + Gx & \textrm{ if } \, 2 \leq i \leq (n+1)/2,\\
        \frac{\mu K}{2}x^2 - Gx & \textrm{ if } \, (n+3)/2 \leq i \leq n.
    \end{cases}
\end{align*}

Compared to the case of even $n$, $f_1(x) = \frac{\mu K}{2} x^2$ is introduced newly.
It is clear that $f_i$ satisfies \Cref{ass:grad-generalized} with $P=0$ and shares the same Hessian.
From the assumption $K \leq \frac{1}{2} \kappa$, we have $\mu \leq \mu K \leq \frac{L}{2}$.
Hence, each $f_i$ is $L$-smooth and $\mu$-strongly convex.
Also, we note that $x^* = 0$ and $F_2(x^*) = 0$.

By \Cref{lem:quadratic-threetype-IGD-closed}, the final iterate obtained by running \Cref{alg:igd} is given by
\begin{align}
    x_n^K & = (1 - \eta \mu K)^{nK} x_0 + \frac{G}{\mu K} \cdot \frac{1 - (1-\eta \mu K)^{nK}}{1 - (1-\eta \mu K)^n}\bigopen{1 - \bigopen{1 - \eta \mu K}^{\frac{n-1}{2}}}^2. \label{eq:small-idhess-lb-three}
\end{align}

For any $\eta \in \left[\frac{1}{\mu n K}, \frac{2}{L}\right)$, it follows that $\frac{1}{n} \le \eta \mu K < \frac{2 \mu K}{L} = \frac{2K}{\kappa} \le 1$.
Then, we have $(1 - \eta \mu K)^{nK} \leq \bigopen{1 - \frac{1}{n}}^{nK} \leq e^{-K} \leq e^{-1}$. 
Moreover, $\bigopen{1 - \eta \mu K}^{\frac{n-1}{2}} \leq (1 - \frac{1}{n})^{\frac{n-1}{2}} \leq e^{-\frac{n-1}{2 n}} \leq e^{-\frac{1}{4}}$ holds for $n \geq 2$. 
Substituting these inequalities into \cref{eq:small-idhess-lb-three} and setting $x_0 = 0$, we have
\begin{align*}
    x_n^K \geq \frac{G}{\mu K} \cdot \frac{1 - e^{-1}}{1} (1 - e^{-\frac{1}{4}})^2.
\end{align*}

Thus, we obtain the following optimality gap:
\begin{align*}
    F_2(x_n^K) - F_2(x^*) = \frac{\mu K}{2} (x_n^K)^2 \gtrsim \frac{G^2}{\mu K}.
\end{align*}

\subsubsection{Construction of $F_3$}
\label{subsubsec:small-lb-idhess-f3}

Let $F_3(x) = \frac{L}{2} x^2$ with component functions $f_i(x) = F_3(x)$ for all $i \in [n]$.
It is clear that $F_1$ satisfies \Cref{ass:common}, \Cref{ass:grad-generalized} with $P=0$, and its component functions share an identical Hessian.
Also, we note that $x^* = 0$ and $F_3(x^*) = 0$.

For all $\eta \in \left[\frac{2}{L}, \infty\right)$, the final iterate is given by
\begin{align*}
    x_n^K = \bigopen{1 - \eta L}^{nK} x_0.
\end{align*}

In this regime, the step size is excessively large, resulting in
\begin{align*}
    1 - \eta L \le 1 - \frac{2}{L} \cdot L \le -1,
\end{align*}
which implies $\bigabs{(1 - \eta L)^{nK}} \ge 1$.
Thus, the iterate does not converge and satisfies $\abs{x^K_n} \geq \abs{x_0}$. 

By setting the initialization $x_0 = \frac{G}{\sqrt{\mu L K}}$, we have
\begin{align*}
    F_3(x_n^K) - F_3(x^*) = \frac{L}{2} (x_n^K)^2 \geq \frac{L}{2}(x_0)^2 \gtrsim \frac{G^2}{\mu K}.
\end{align*}

\subsection{Proof of \Cref{thm:small-lb-sc}}
\label{subsec:small-lb-sc}
\thmsmalllbsc*

\begin{proof}
Similar to the approach in \Cref{thm:small-lb-idhess}, we divide the range of step sizes into three regimes.
For each regime, we construct the overall functions $F_1$, $F_2$, and $F_3$ respectively, along with their respective component functions and an initial point.
Finally, we aggregate these functions across different dimensions to derive the stated lower bound.

The functions $F_1$ and $F_3$ are $1$-dimensional, and $F_2$ is a $2$-dimensional function.
Each function is carefully designed to satisfy the following properties:
\begin{itemize}
    \item (Small step size regime) There exists an initialization point $x_0 = \text{poly}(\mu, L, n, K, G)$ such that for any choice of $\eta \in \bigopen{0, \frac{1}{\mu n K}}$, the final iterate $x_n^K$ obtained by running \Cref{alg:igd} satisfies $F_1(x_n^K) - F_1(x^*) \gtrsim \frac{LG^2}{\mu^2}\min\bigset{1, \frac{\kappa^2}{K^4}}$.
    \item (Moderate step size regime) There exists an initialization point $(y_0, z_0) = \text{poly}(\mu, L, n, K, G)$ such that for any choice of $\eta \in \left[\frac{1}{\mu n K}, \frac{2}{L}\right)$, the final iterate $(y_n^K, z_n^K)$ obtained by running \Cref{alg:igd} satisfies $F_2(y_n^K, z_n^K) - F_2(y^*, z^*) \gtrsim \frac{LG^2}{\mu^2}\min\bigset{1, \frac{\kappa^2}{K^4}}$.
    \item (Large step size regime) There exists an initialization point $w_0 = \text{poly}(\mu, L, n, K, G)$ such that for any choice of $\eta \in \left[\frac{2}{L}, \infty\right)$, the final iterate $w_n^K$ obtained by running \Cref{alg:igd} satisfies $F_3(w_n^K) - F_3(w^*) \gtrsim \frac{LG^2}{\mu^2}\min\bigset{1, \frac{\kappa^2}{K^4}}$.
\end{itemize}
Here, $x^*$, $(y^*, z^*)$, and $w^*$ denote the minimizers of $F_1$, $F_2$, and $F_3$, respectively.
All these functions are designed to satisfy \Cref{ass:common}.
$F_1$ and $F_3$ satisfy \Cref{ass:grad-generalized} with $G = P=0$, and $F_2$ satisfies with $P=1$.
Moreover, each component function within each overall function is $\mu$-strongly convex.
Detailed constructions of $F_1$, $F_2$, and $F_3$, as well as the verification of the assumptions and the stated properties are presented in \Cref{subsubsec:small-lb-sc-f1,subsubsec:small-lb-sc-f2,subsubsec:small-lb-sc-f3}.

By following a similar approach to the proof of \Cref{thm:small-lb-idhess}, we can conclude that the aggregated $4$-dimensional function $F(\vx) := F(x, y, z, w) = F_1(x) + F_2(y, z) + F_3(w)$ and its component functions satisfy \Cref{ass:common}.
Additionally,
\begin{align*}
    \bignorm{\nabla f_i(\vx) - \nabla F(\vx)}
    & \le \bignorm{\nabla f_{1i}(x) - \nabla F_1(x)} + \bignorm{\nabla f_{2i}(y) - \nabla F_2(y)} + \bignorm{\nabla f_{3i}(z) - \nabla F_3(z)} \\
    & \le 0 + \bigopen{G + \bignorm{\nabla F_2(y)}} + 0 \le G + \bignorm{\nabla F(\vx)},
\end{align*}
thus satisfying \Cref{ass:grad-generalized} with $P=1$.
Also, since each dimension is independent, it is obvious that $\vx^* = (x^*, y^*, z^*, w^*)$ minimizes $F$.
Moreover, by choosing the initialization point as $\vx_0 = (x_0, y_0, z_0, w_0)$, the final iterate $\vx_n^K = (x_n^K, y_n^K, z_n^K, w_n^K)$ obtained by running \Cref{alg:igd} on $F$ satisfies
\begin{align*}
    F(\vx_n^K) - F(\vx^*) \gtrsim \frac{LG^2}{\mu^2}\min\bigset{1, \, \frac{\kappa^2}{K^4}},
\end{align*}
regardless of the choice of $\eta > 0$.

This concludes the proof of \Cref{thm:small-lb-sc}.
\end{proof}

In the following subsections, we present the specific construction of $F_1$, $F_2$, and $F_3$, and demonstrate that each satisfies the stated lower bound within its corresponding step size regime.
For simplicity of notation, we omit the index of the overall function when referring to its component functions, e.g., we write $f_i(x)$ instead of $f_{1i}(x)$. 
Moreover, we use the common variable notation $x$ (and $y$) while constructing functions for each dimension, though we use different variables in the ``dimension-aggregation'' step.

\subsubsection{Construction of $F_1$}
\label{subsubsec:small-lb-sc-f1}

Let $F_1(x) = \frac{\mu}{2}x^2$ with component functions $f_i(x) = F_1(x)$ for all $i \in [n]$.
It is clear that $F_1$ satisfies \Cref{ass:common}, \Cref{ass:grad-generalized} with $G = P = 0$, and has $\mu$-strongly convex component functions.
Also, we note that $x^* = 0$ and $F_1(x^*) = 0$.

Let the initialization be $x_0 = \sqrt{\kappa} \min \bigset{1, \frac{\kappa}{K^2}} \frac{G}{\mu}$.
For all $\eta \in \bigopen{0, \frac{1}{\mu n K}}$, the final iterate is given by
\begin{align*}
    x_n^K = (1 - \eta \mu)^{nK} x_0 \geq \bigopen{1 - \frac{1}{nK}}^{nK} x_0 \geq \sqrt{\kappa} \min \bigset{1, \, \frac{\kappa}{K^2}} \frac{G}{4\mu},
\end{align*}
where the last inequality uses the fact that $(1 - \frac{1}{m})^m \geq \frac{1}{4}$ for all $m \geq 2$.

Thus, we have
\begin{align*}
    F_1(x_n^K) - F_1(x^*) = \frac{\mu}{2} (x_n^K)^2 \gtrsim \frac{LG^2}{\mu^2}\min\bigset{1, \, \frac{\kappa^2}{K^4}}.
\end{align*}

\subsubsection{Construction of $F_2$}
\label{subsubsec:small-lb-sc-f2}

In this subsection, we let $L'$ denote $L/2$.
We introduce the design of each component function as follows:
\begin{align*}
    f_i(x, y) = \begin{cases}
        \frac{\mu}{2}x^2 + \frac{L'}{2}y^2 - Gx & \text{if } \, i = 1,\\
        \bigopen{f_1 \circ (R_{i-1})^{-1}}(x, y) &\text{if} \, 2 \le i \le n,
    \end{cases}
\end{align*}
where
\begin{align*}
    R_i := \begin{bmatrix}
    \cos \theta_i & -\sin \theta_i\\
    \sin \theta_i & \cos \theta_i
\end{bmatrix}
\end{align*}
is the matrix for the counter-clock wise rotation in $\sR^2$ by an angle $\theta_i = i \delta$ with $\delta := \frac{2\pi}{n}$.

Using these component functions, the overall function $F_2 := \frac{1}{n} \sum_{i=1}^n f_i$ is given by $\frac{\mu + L'}{4}(x^2 + y^2)$.
This result can be verified by expanding the closed form of $f_i$:
\begin{align*}
    f_i(x, y) & = \frac{\mu}{2}(x \cos \theta_{i-1} + y \sin \theta_{i-1})^2 + \frac{L'}{2}(-x \sin \theta_{i-1} + y \cos \theta_{i-1})^2 - G(x \cos \theta_{i-1} + y \sin \theta_{i-1})\\
    & = \frac{1}{2}\bigopen{\mu \cos^2 \theta_{i-1} + L' \sin^2 \theta_{i-1}}x^2 + \frac{1}{2}\bigopen{\mu \sin^2 \theta_{i-1} + L' \cos^2 \theta_{i-1}}y^2\\
    & \quad\quad + (\mu - L')\sin \theta_{i-1} \cos \theta_{i-1}  xy - G(x \cos \theta_{i-1} + y \sin \theta_{i-1}).
\end{align*}
Since $n\ge 3$, we can utilize \Cref{lem:euler,lem:sumsquares}, and obtain
\begin{align*}
    &\frac{1}{n}\sum_{i=1}^n \sin \theta_{i-1} = \frac{1}{n}\sum_{i=1}^n \cos \theta_{i-1} = 0,\\
    &\frac{1}{n}\sum_{i=1}^n \sin^2 \theta_{i-1} = \frac{1}{n}\sum_{i=1}^n \cos^2 \theta_{i-1} = \frac{1}{2},\\
    &\frac{1}{n}\sum_{i=1}^n \sin \theta_{i-1} \cos \theta_{i-1} = \frac{1}{2n} \sum_{i=1}^n \sin \theta_{2(i-1)} = 0.
\end{align*}
Using these results, the overall function is simplified to
\begin{align*}
    F_2(x, y) = \frac{\mu + L'}{4}(x^2 + y^2),
\end{align*}
which has a minimizer $(x^*, y^*) = (0, 0)$.

Note that each component function $f_i$ is obtained by rotating $f_1$, and hence $f_i$ inherits the properties of $f_1$.
We can easily check that $f_1$ is both $\mu$-strongly convex and $L$-smooth.
Also, the gradient difference between the component function $f_1$ and the overall function $F_2$ can be expressed as
\begin{align*}
    \bignorm{\nabla f_1 (x, y) - \nabla F_2(x, y) } &= \norm{\bigopen{\bigopen{\mu x - G} - \frac{\mu + L'}{2}x, Ly - \frac{\mu + L'}{2}y}}\\
    & = \bignorm{\bigopen{\frac{\mu - L'}{2} x - G, \frac{L' - \mu}{2} y}} \\
    & \le G + \bignorm{\bigopen{\frac{\mu - L'}{2} x, \frac{L' - \mu}{2} y}} \\
    & = G + \bignorm{\bigopen{\frac{L' - \mu}{2} x, \frac{L' - \mu}{2} y}} \\
    & \le G + \bignorm{\bigopen{\frac{L' + \mu}{2} x, \frac{L' + \mu}{2} y}} = G + \bignorm{\nabla F_2(x,y)},
\end{align*}
proving that the construction satisfies \Cref{ass:grad-generalized} with $P = 1$. 

Before delving into the detailed proof, we outline the intuition for the construction.
We start by designing a step-size-dependent initialization point $(u_0(\eta), v_0(\eta))$, where $\eta \in \left[\frac{1}{\mu n K}, \frac{2}{L} \right)$.
For $i \in [n]$, we define $(u_i(\eta), v_i(\eta))$ as the result of running a single step of gradient descent on $f_i$ with a step size $\eta$, starting from $(u_{i-1}(\eta), v_{i-1}(\eta))$.

The key idea is to carefully design $(u_0(\eta), v_0(\eta))$ so that each subsequent iterate $(u_i(\eta), v_i(\eta))$ is obtained by rotating $(u_{i-1}(\eta), v_{i-1}(\eta))$ by an angle $\delta = \frac{2\pi}{n}$.
This aligns with our construction of the component functions $f_i$, which are also generated by continually rotating $f_1$ by the same angle $\delta$.
As a result, the relative position between each iterate and the component function used to compute the next iterate is preserved throughout the entire update process.
This symmetry ensures that the trajectory of the iterates $(u_i(\eta), v_i(\eta))$ forms a regular $n$-sided polygon.
Consequently, after running \Cref{alg:igd}, the final iterate and the initialization point $(u_0(\eta), v_0(\eta))$ are identical.

At the last step of the proof, we will show that the choice of the initialization point $(u_0(\frac{1}{\mu n K}), v_0(\frac{1}{\mu n K}))$ can be made in a step-size-independent manner without significantly affecting the final optimality gap, even when the step size $\eta$ is chosen from $\left(\frac{1}{\mu n K}, \frac{2}{L} \right)$ rather than being fixed at $\frac{1}{\mu n K}$.

We now proceed to describe the exact construction of $(u_0(\eta), v_0(\eta))$.
Consider the gradient of the component function $f_1(x, y)$, which is given by:
\begin{align*}
    \nabla_x f_1(x, y) = \mu x - G \, \text{ and} \,\, \nabla_y f_1(x, y) = L' y.
\end{align*}

A single iteration of gradient descent on $f_1$ using the step size $\eta$ yields:
\begin{align}
\label{eq:uvinitial}
\begin{split}
    &u_1(\eta) = u_0(\eta) - \eta (\mu u_0(\eta) - G),\\
    &v_1(\eta) = v_0(\eta) - \eta L' v_0(\eta).
\end{split}
\end{align} 

To maintain the rotational relationship between successive iterates, we require that the updated iterate $(u_1(\eta), v_1(\eta))$ satisfies the following relationship:
\begin{align}
    \begin{bmatrix}
        u_1(\eta) \\ v_1(\eta)
    \end{bmatrix}
    =
    \begin{bmatrix}
        \cos \delta & - \sin \delta \\
        \sin \delta & \cos \delta
    \end{bmatrix}
    \begin{bmatrix}
        u_0(\eta) \\ v_0(\eta)
    \end{bmatrix}. \label{eq:rotationrelationship}
\end{align}
As mentioned earlier, this rotational relationship ensures that the trajectory of the iterates forms a regular $n$-sided polygon.
Recall that the component function $f_i$ is defined as:
\begin{align*}
    f_i(x, y) = \bigopen{f_1 \circ (R_{i-1})^{-1}}(x, y).
\end{align*}
Additionally, let $\Lambda = \begin{bmatrix}
\mu & 0 \\ 0 & L \end{bmatrix}$. Since 
\begin{align*}
    \nabla f_1(x, y) = \Lambda \begin{bmatrix}
        x\\ y
    \end{bmatrix} - G\begin{bmatrix}
        1\\ 0
    \end{bmatrix},
\end{align*}
the gradient of $f_i$ can then be expressed as
\begin{align}
    \nabla f_i(x, y) = R_{i-1} \bigopen{
        \Lambda (R_{i-1})^{-1} \begin{bmatrix}
            x\\ y
        \end{bmatrix} - G\begin{bmatrix}
            1\\ 0
        \end{bmatrix}
    }. \label{eq:fgradient}
\end{align}

If $(u_1(\eta), v_1(\eta))$ satisfies both \cref{eq:uvinitial} and \cref{eq:rotationrelationship}, then the subsequent iterate $(u_2(\eta), v_2(\eta))$ satisfies:
\begin{align*}
    \begin{bmatrix}
        u_2(\eta) \\ v_2(\eta)
    \end{bmatrix}
    &= \bigopen{\mI - \eta R_{1} \Lambda (R_{1})^{-1}} \begin{bmatrix}
        u_1(\eta) \\ v_1(\eta)
    \end{bmatrix} + \eta G R_{1}\begin{bmatrix}
        1 \\ 0
    \end{bmatrix}\\
    &= \bigopen{\mI - \eta R_{1} \Lambda (R_{1})^{-1}} R_1 \begin{bmatrix}
        u_0(\eta) \\ v_0(\eta)
    \end{bmatrix} + \eta G R_{1}\begin{bmatrix}
        1 \\ 0
    \end{bmatrix}\\
    &= R_1 \bigopen{\bigopen{\mI - \eta \Lambda} \begin{bmatrix}
        u_0(\eta) \\ v_0(\eta)
    \end{bmatrix} + \eta G \begin{bmatrix}
        1 \\ 0
    \end{bmatrix}} 
    = R_1 \begin{bmatrix}
        u_1(\eta) \\ v_1(\eta)
    \end{bmatrix} 
    = R_2 \begin{bmatrix}
        u_0(\eta) \\ v_0(\eta)
    \end{bmatrix}.
\end{align*}

Thus, if the initial point $(u_0(\eta), v_0(\eta))$ and its successive iterate $(u_1(\eta), v_1(\eta))$ satisfies the rotational relationship, this relation persists throughout the entire update process.
Consequently, the trajectory of the iterates forms a regular $n$-sided polygon.

To enforce this rotational relationship, we solve the following system of equations:
\begin{align*}
    & u_0(\eta) \cos \delta - v_0(\eta) \sin \delta = (1 - \eta \mu) u_0(\eta) + \eta G,\\
    & u_0(\eta) \sin \delta + v_0(\eta) \cos \delta = (1 - \eta L') v_0(\eta).
\end{align*}

From these, we derive:
\begin{align}
    &u_0(\eta) = \frac{\eta L' - (1 - \cos \delta)}{(1 - \cos \delta)(2 - (\mu + L')\eta) + \eta^2 \mu L'} \cdot \eta G, \label{eq:u0}\\
    &v_0(\eta) = -\frac{\sin \delta}{(1 - \cos \delta)(2 - (\mu + L')\eta) + \eta^2 \mu L'} \cdot \eta G. \label{eq:v0}
\end{align}

Note that the numerator of $u_0(\eta)$ is positive as shown below:
\begin{align}
    \eta L' - (1- \cos \delta)
    &\overset{(a)}{\ge} \frac{\eta L'}{2} + \frac{L'}{2 \mu n K} - \frac{\delta^2}{2} \nonumber \\
    &\overset{(b)}{=} \frac{\eta L'}{2} + \frac{L}{4 \mu n K} - \frac{2\pi^2}{n^2} \nonumber \\
    &= \frac{\eta L'}{2} + \frac{nL - 8 \pi^2 \mu K}{4\mu n^2 K} \nonumber \\
    &\overset{(c)}{>} \frac{\eta L'}{2}, \label{eq:u0numerator}
\end{align}
where we apply $\eta \ge \frac{1}{\mu n K}$ and the inequality $1 - \cos \theta \le \frac{\theta^2}{2}$ at $(a)$, substitute $L'=\frac{L}{2}$ and $\delta = \frac{2\pi}{n}$ at $(b)$, and apply $n \ge 3 > \frac{\pi}{2}$ and $\kappa \ge 16 \pi K$ at $(c)$.
Also, we have $2 - (\mu + L')\eta \ge 0$ for $\eta < \frac{2}{L}$.
Thus, $u_0(\eta)$ is always positive and $v_0(\eta)$ is always negative for $\eta \in \left[\frac{1}{\mu n K}, \frac{2}{L} \right)$.

Let $(u_i^k(\eta), v_i^k(\eta))$ denote the $i$-th iterate at the $k$-th epoch of \Cref{alg:igd} using the step size $\eta$ and the initialization point $(u_0(\eta), v_0(\eta))$.
By definition, $(u_n^K(\eta), v_n^K(\eta))$ represents the final iterate after $K$ epochs.
Due to rotational symmetry, $(u_n^K(\eta), v_n^K(\eta))$ is identical to $(u_0(\eta), v_0(\eta))$.
Thus, the distance between the final iterate and the minimizer of $F_2$ can be lower bounded by the $x$-coordinate of the initialization point:
\begin{align*}
    \bignorm{(u_n^K(\eta), v_n^K(\eta))} = \bignorm{(u_0(\eta), v_0(\eta))} \ge u_0(\eta).
\end{align*}

We now derive a lower bound for $u_0(\eta)$.
We begin by upper bounding its denominator:
\begin{align*}
    (1 - \cos \delta)(2 - (\mu + L')\eta) + \eta^2 \mu L' 
    < \frac{\delta^2}{2} \cdot 2 + \eta^2 \mu L'
    = \frac{4\pi^2}{n^2} + \eta^2 \mu L'.
\end{align*}

Substituting this result into \cref{eq:u0} gives
\begin{align*}
    u_0(\eta) \ge \frac{(\eta L' - (1 - \cos \delta))\eta G}{\frac{4\pi^2}{n^2} + \eta^2 \mu L'} > \frac{ \eta^2 L' G }{\frac{8\pi^2}{n^2} + 2 \eta^2 \mu L'} := \varphi(\eta),
\end{align*}
where we apply \cref{eq:u0numerator} at the second inequality.
We can easily check that $\varphi(\eta)$ is an increasing function of $\eta$ and thus the minimum value is attained at $\eta = \frac{1}{\mu n K}$.
Substituting $\eta = \frac{1}{\mu n K}$, we have
\begin{align*}
    \varphi(\frac{1}{\mu n K})
    &= \frac{L' G }{\frac{8\pi^2}{n^2} \cdot \mu^2 n^2 K^2 + 2 \mu L'}\\
    &= \frac{L G}{16 \pi^2 \mu^2 K^2 + 2 \mu L} \quad (\because L' = \frac{L}{2})\\
    &= \frac{G}{\mu} \cdot \frac{\frac{L}{\mu K^2}}{16\pi^2 + \frac{2L}{\mu K^2}}.
\end{align*}

In summary, the distance between the final iterate and the minimizer of $F_2$ is bounded as:
\begin{align}
    \bignorm{(u_n^K(\eta), v_n^K(\eta))} = \bignorm{(u_0(\eta), v_0(\eta))}
    \ge u_0(\eta) \ge \varphi(\eta) \ge \varphi(\frac{1}{\mu n K}) \overset{(a)}{\geq} \frac{G}{32\pi^2\mu} \min \bigset{1, \, \frac{\kappa}{K^2}}, \label{eq:uvminimum}
\end{align}
where $(a)$ is derived through the following process:
\begin{align*}
    \varphi(\frac{1}{\mu n K})
    = \frac{G}{\mu} \frac{\frac{L}{\mu K^2}}{16\pi^2 + \frac{2L}{\mu K^2}} = \frac{G}{2\mu} \frac{\frac{\kappa}{8\pi^2 K^2}}{1 + \frac{\kappa}{8\pi^2 K^2}} \overset{(b)}{\geq} \frac{G}{4\mu} \min \bigset{1, \frac{\kappa}{8\pi^2 K^2}} \geq \frac{G}{32\pi^2\mu} \min \bigset{1, \, \frac{\kappa}{K^2}}.
\end{align*}
Here, we use the inequality $\frac{u}{1 + u} \geq \frac{1}{2}\min \bigset{1, u}$ for all $u \geq 0$ at $(b)$. 
The function optimality gap can then be bounded as:
\begin{align*}
    F_2(u_n^K(\eta), v_n^K(\eta)) - F_2(x^*, y^*) = \frac{\mu + L'}{4} \bignorm{(u_n^K(\eta), v_n^K(\eta))}^2
    \gtrsim \frac{L G^2}{\mu^2} \min \bigset{1, \, \frac{\kappa^2}{K^4}},
\end{align*}
for $\eta \in \left[\frac{1}{\mu n K}, \frac{2}{L} \right)$.
However, one caveat is that the initialization point $(u_0(\eta), v_0(\eta))$ depends on the choice of $\eta$.
Our goal is to identify a unified, step-size-independent initialization point $(x_0, y_0)$ that achieves the same lower bound (up to a scaling factor).
Specifically, we aim to ensure:
\begin{align*}
    F_2(x_n^K, y_n^K) \gtrsim \frac{L G}{\mu^2} \min \bigset{1, \, \frac{\kappa^2}{K^4}}
\end{align*}
for any choice of $\eta \in \left[\frac{1}{\mu n K}, \frac{2}{L} \right)$.

We claim that this goal can be achieved by setting the initialization point as $(x_0, y_0) = (u_0(\frac{1}{\mu n K}), v_0(\frac{1}{\mu n K}))$.
To prove this claim, consider two sequences of iterates: $\{(x_i^k, y_i^k)\}_{i \in [n], k \in [K]}$ and $\{(u_i^k(\eta), v_i^k(\eta))\}_{i \in [n], k \in [K]}$.
Both sequences are generated using the same permutation $\mathrm{id}_n$ and the same step size $\eta$, but they differ in their initialization points.
Specifically, $(x_i^k, y_i^k)$ starts from the initial point $(u_0(\frac{1}{\mu n K}), v_0(\frac{1}{\mu n K}))$, while $(u_i^k(\eta), v_i^k(\eta))$ starts from the initial point $(u_0(\eta), v_0(\eta))$.

Recall the gradient of $f_i$ from \cref{eq:fgradient}:
\begin{align*}
    \nabla f_i(x, y) = R_{i-1} \bigopen{
        \Lambda (R_{i-1})^{-1} \begin{bmatrix}
            x\\ y
        \end{bmatrix} - G\begin{bmatrix}
            1\\ 0
        \end{bmatrix}
    }.
\end{align*}

The update rule for the iterates generated by \igd{} is:
\begin{align*}
    \begin{bmatrix}
        x_i^k \\ y_i^k
    \end{bmatrix} & = \bigopen{\mI - \eta R_{i-1} \Lambda (R_{i-1})^{-1}} \begin{bmatrix}
        x_{i-1}^k \\ y_{i-1}^k
    \end{bmatrix} + \eta G R_{i-1}\begin{bmatrix}
        1 \\ 0
    \end{bmatrix} \text{ and } \\ 
    \begin{bmatrix}
        u_i^k(\eta) \\ v_i^k(\eta)
    \end{bmatrix} & = \bigopen{\mI - \eta R_{i-1} \Lambda (R_{i-1})^{-1}} \begin{bmatrix}
        u_{i-1}^k(\eta) \\ v_{i-1}^k(\eta)
    \end{bmatrix} + \eta G R_{i-1}\begin{bmatrix}
        1 \\ 0
    \end{bmatrix}.
\end{align*}

Taking the difference between the two sequences of iterates, we have:
\begin{align*}
    \begin{bmatrix}
        x_i^k - u_i^k(\eta) \\ y_i^k - v_i^k(\eta)
    \end{bmatrix}
    = \bigopen{\mI - \eta R_{i-1} \Lambda (R_{i-1})^{-1}} \begin{bmatrix}
        x_{i-1}^k - u_{i-1}^k(\eta) \\ y_{i-1}^k - v_{i-1}^k(\eta)
    \end{bmatrix}.
\end{align*}
Since $R_{i-1}$ is an unitary matrix, it follows that  $R_{i-1}\Lambda (R_{i-1})^{-1} \succeq \mu I$. 
Thus, we obtain the inequality $\mI - \eta R_{i-1}\Lambda R_{i-1}^{-1} \preceq (1 - \eta \mu)I$, which leads to the following bound:
\begin{align*}
    \norm{(x^k_i - u^k_i(\eta), y^k_i - v^k_i(\eta))} \leq (1 - \eta \mu) \norm{(x^k_{i-1} - u^k_{i-1}(\eta), y^k_{i-1} - v^k_{i-1}(\eta))}.
\end{align*}

Based on this inequality, we will demonstrate that $\norm{(x^K_n - u^K_n(\eta), y^K_n - v^K_n(\eta))}$ is not significant.
This can be interpreted as the gap between the initialization points shrinking progressively throughout the optimization process.
Specifically, we have
\begin{align*}
    \norm{(x^K_n - u^K_n(\eta), y^K_n - v^K_n(\eta))} & \leq (1 - \eta \mu)^{nK} \norm{(x_0 - u_0(\eta), y_0 - v_0(\eta))}\\
    & \leq e^{-\eta \mu n K} \norm{(x_0 - u_0(\eta), y_0 - v_0(\eta))}\\
    & \overset{(a)}{\leq} e^{-1} \norm{(u_0(\frac{1}{\mu n K}) - u_0(\eta), v_0(\frac{1}{\mu n K}) - v_0(\eta))}\\
    & \leq e^{-1} \bigopen{\abs{u_0(\frac{1}{\mu n K}) - u_0(\eta)} + \abs{v_0(\frac{1}{\mu n K}) - v_0(\eta)}}\\
    & \overset{(b)}{\leq} e^{-1} \bigopen{u_0(\eta) + \abs{v_0 (\frac{1}{\mu n K})} + \abs{v_0(\eta)}}\\
    & \overset{(c)}{\leq} e^{-1}\bigopen{u_0(\eta) + \frac{8 \pi K}{\kappa}u_0(\frac{1}{\mu n K}) + \frac{8 \pi K}{\kappa}u_0(\eta)}\\
    & \overset{(d)}{\leq} \frac{1 + \frac{16 \pi K}{\kappa}}{e}u_0(\eta)\\
    & \overset{(e)}{\leq} \frac{2}{e}u_0(\eta)\\
    & \overset{(f)}{\leq} \frac{2}{e}\norm{(u^K_n(\eta), v^K_n(\eta))}.
\end{align*}
Here, we apply:
\begin{itemize}
    \item $(a)$: $\eta \ge \frac{1}{\mu n K}$ and $(x_0, y_0) = \bigopen{u_0\bigopen{\frac{1}{\mu n K}}, v_0\bigopen{\frac{1}{\mu n K}}}$.
    \item $(b)$, $(d)$: $u_0(\eta)$ is positive and increasing (shown in \Cref{lem:uincreasing}).
    \item $(c)$: Follows from \Cref{lem:vsmallthanu}.
    \item $(e)$: $K \le \frac{\kappa}{16\pi}$.
    \item $(f)$: $u_0(\eta) = u^K_n(\eta)$.
\end{itemize}

In summary, the distance between $(x_n^K, y_n^K)$ and the minimizer of $F_2$ can be bounded as:
\begin{align*}
    \norm{(x^K_n, y^K_n)} \geq \norm{(u^K_n(\eta), v^K_n(\eta))} - \norm{(x^K_n - u^K_n(\eta), y^K_n - v^K_n(\eta))} \geq \bigopen{1 - \frac{2}{e}}\norm{(u^K_n(\eta), v^K_n(\eta))},
\end{align*}
where we have already derived that $\norm{(u^K_n(\eta), v^K_n(\eta))} \geq \frac{G}{32\pi^2\mu} \min \bigset{1, \, \frac{\kappa}{K^2}}$ in \cref{eq:uvminimum}.
Thus, we conclude
\begin{align*}
    \norm{(x^K_n, y^K_n)} \geq \frac{\bigopen{1 - 2e^{-1}}G}{32\pi^2\mu} \min \bigset{1, \, \frac{\kappa}{K^2}}
\end{align*}
and consequently,
\begin{align*}
    F_2(x^K_n, y^K_n)  - F_2(x^*, y^*) = \frac{L' + \mu}{4}\norm{(x^K_n, y^K_n)}^2 \geq \frac{L}{8}\norm{(x^K_n, y^K_n)}^2 \gtrsim \frac{LG^2}{\mu^2} \min \bigset{1, \, \frac{\kappa^2}{K^4}}.
\end{align*}

\subsubsection{Construction of $F_3$}
\label{subsubsec:small-lb-sc-f3}

Let $F_3(x) = \frac{L}{2} x^2$ with component functions $f_i(x) = F_3(x)$ for all $i \in [n]$.
It is clear that $F_1$ satisfies \Cref{ass:common}, \Cref{ass:grad-generalized} with $G=P=0$ and has $\mu$-strongly convex component functions.
Also, we note that $x^* = 0$ and $F_3(x^*) = 0$.

For all $\eta \in \left[\frac{2}{L}, \infty\right)$, the final iterate is given by
\begin{align*}
    x_n^K = \bigopen{1 - \eta L}^{nK} x_0.
\end{align*}

In this regime, the step size is excessively large, resulting in
\begin{align*}
    1 - \eta L \le 1 - \frac{2}{L} \cdot L \le -1,
\end{align*}
which implies $\bigabs{(1 - \eta L)^{nK}} \ge 1$.
Thus, the iterate does not converge and satisfies $\abs{x^K_n} \geq \abs{x_0}$. 

By setting the initialization $x_0 = \frac{G}{\mu} \min \bigset{1, \, \frac{\kappa}{K^2}}$, we have
\begin{align*}
    F_3(x_n^K) - F_3(x^*) = \frac{L}{2} (x_n^K)^2 \geq \frac{L}{2}(x_0)^2 \gtrsim \frac{LG^2}{\mu^2} \min \bigset{1, \, \frac{\kappa^2}{K^4}}.
\end{align*}

\subsection{Proof of \Cref{thm:small-lb-concave}}
\label{subsec:small-lb-concave}
\thmsmalllbconcave*
\begin{proof}
Similar to the approach in \Cref{thm:small-lb-idhess}, we divide the range of step sizes into two regimes.
For each regime, we construct the overall functions $F_1$ and $F_2$, respectively, along with their respective component functions and an initial point.
Finally, we aggregate these functions across different dimensions to derive the stated lower bound.

Each function is $1$-dimensional and carefully designed to satisfy the following properties:
\begin{itemize}
    \item (Small step size regime) For any choice of the initialization point $x_0 = D$ and step size $\eta \in \bigopen{0, \frac{1}{\mu n K}}$, the final iterate $x_n^K$ obtained by running \Cref{alg:igd} satisfies $F_1(x_n^K) - F_1(x^*) \gtrsim \mu D^2$.
    \item (Moderate \& Large step size regime) There exists an initialization point $y_0 = \text{poly}(\mu, L, n, K, G)$ such that for any choice of $\eta \in \left[\frac{1}{\mu n K}, \infty \right)$, the final iterate $y_n^K$ obtained by running \Cref{alg:igd} satisfies $F_2(y_n^K) - F_2(y^*) \gtrsim \frac{G^2}{L}\bigopen{1 + \frac{L}{2\mu n K}}^{n}$.
\end{itemize}
Here, $x^*$ and $y^*$ denote the minimizer of $F_1$ and $F_2$, respectively.
Both functions are designed to satisfy \Cref{ass:common}. 
$F_1$ satisfies \Cref{ass:grad-generalized} with $G=P=0$, and $F_2$ satisfies with $P=3$.
Detailed constructions of $F_1$ and $F_2$, as well as the verification of the assumptions and the stated properties are presented in \Cref{subsubsec:small-lb-concave-f1,subsubsec:small-lb-concave-f2}.

By following a similar approach to the proof of \Cref{thm:small-lb-idhess,thm:small-lb-sc}, we can conclude that the aggregated $2$-dimensional function $F(\vx) := F(x, y) = F_1(x) + F_2(y)$ and its component functions satisfy the stated assumptions.
Also, since each dimension is independent, it is obvious that $\vx^* = (x^*, y^*)$ minimizes $F$.
Finally, by starting from the initialization point as $\vx_0 = (D, 0)$, the final iterate $\vx_n^K = (x_n^K, y_n^K)$ obtained by running \Cref{alg:igd} on $F$ satisfies
\begin{align*}
    F(\vx_n^K) - F(\vx^*) \gtrsim \min \bigset{ \mu D^2, \, \frac{G^2}{L}\bigopen{1 + \frac{L}{2\mu n K}}^{n} },
\end{align*}
for any choice of $D \in \R$ and $\eta > 0$.

This concludes the proof of \Cref{thm:small-lb-concave}.
\end{proof}

One key distinction of \Cref{thm:small-lb-concave} compared to other lower bound theorems is the explicit inclusion of $D$ in the statement.
While it is possible to express the lower bounds in other theorems with a dependency on $D$, we chose to leave this dependency only for \Cref{thm:small-lb-concave} due to the unique behavior of the term $\frac{G^2}{L}\bigopen{1 + \frac{L}{2\mu n K}}^{n}$.

Unlike typical bounds, this expression cannot be simplified into a clear, closed-form polynomial expression.
Its proportional degree with respect to $\mu$, $L$, $n$, and $K$ varies depending on their values.
In particular, when $K$ is small (e.g., near $\frac{\kappa}{n}$), the term exhibits exponential growth, scaling as $c^n \cdot \frac{G^2}{L}$ where $c$ is a constant greater than $1.1$.

This exponential growth introduces challenges when attempting to express the bound without the ``min'' operator, as in other theorems.
Specifically, the first coordinate of the initialization point, $x_0$, would need to grow to an exponential scale, which is undesirable to when comparing to the upper bound theorems that hide logarithmic dependency.
For these reasons, we leave the dependency on $D$ explicitly in the bound.

In the following subsections, we present the specific construction of $F_1$ and $F_2$, and demonstrate that each satisfies the stated lower bound within its corresponding step size regime.
For simplicity of notation, we omit the index of the overall function when referring to its component functions, e.g., we write $f_i(x)$ instead of $f_{1i}(x)$. 
Moreover, we use the common variable notation $x$ while constructing functions for each dimension, though we use different variables in the ``dimension-aggregation'' step.

\subsubsection{Construction of $F_1$}
\label{subsubsec:small-lb-concave-f1}
Let $F_1(x) = \frac{\mu}{2}x^2$ with component functions $f_i(x) = F_1(x)$ for all $i \in [n]$.
It is clear that $F_1$ satisfies \Cref{ass:common,ass:grad-optimum} with $G=P=0$.
Also, we note that $x^* = 0$ and $F_1(x^*) = 0$.

For all $\eta \in \bigopen{0, \frac{1}{\mu n K}}$, the final iterate is given by
\begin{align*}
    x_n^K = (1 - \eta \mu)^{nK} x_0 \geq \bigopen{1 - \frac{1}{nK}}^{nK} x_0 \geq \frac{x_0}{4},
\end{align*}
where the last inequality uses the fact that $(1 - \frac{1}{m})^m \geq \frac{1}{4}$ for all $m \geq 2$.

Thus, for $x_0 = D$, we have
\begin{align*}
    F_1(x_n^K) - F_1(x^*) = \frac{\mu}{2} (x_n^K)^2 \gtrsim \mu D^2.
\end{align*}

\subsubsection{Construction of $F_2$}
\label{subsubsec:small-lb-concave-f2}

In this section, we focus on the case when $n$ is even.
If $n$ is odd, we set $n-1$ components satisfying the argument, and introduce an additional zero component function.
This adjustment does not affect the final result, but only modifies the parameters $\mu$, $L$, $n$ by at most a constant factor.

Let $F_2(x) = \frac{L}{8}x^2$ with component functions 
\begin{align*}
    f_i(x) = \begin{cases}
        \frac{L}{2}x^2 + Gx & \textrm{ if } i \leq n/2,\\
        -\frac{L}{4}x^2 - Gx &\textrm{ otherwise}.
    \end{cases}
\end{align*}

Note that the first $n/2$ component functions are strongly convex, while the remaining component functions are concave.
The overall function $F_2$ is $\mu$-strongly convex, since $\frac{L}{4} \ge \mu$ holds from the assumption $\kappa \ge 4$, thereby satisfying \Cref{ass:common}.
Also, it is clear that $f_i$ satisfies \Cref{ass:grad-generalized} with $P=3$.
We note that $x^* = 0$ and $F_2(x^*) = 0$.

We now consider the relationship between $x_0^k$ and $x_0^{k+1}$.
Applying \Cref{lem:quadratic-concave-IGD-closed}, we have 
\begin{align}
    x_0^{k+1} = \bigopen{1 + \frac{\eta L}{2}}^{\frac{n}{2}}\bigopen{1 - \eta L}^{\frac{n}{2}} x^k_0 + \frac{G}{L}\bigopen{\bigopen{1 + \frac{\eta L}{2}}^{\frac{n}{2}} \bigopen{1 + (1 - \eta L)^\frac{n}{2}} - 2}.\label{eq:concavesmallepoch}
\end{align}

From $K \leq \frac{\kappa}{4}$, we have $\eta \ge \frac{1}{\mu n K} \ge \frac{4}{nL}$.
We now derive the lower bound for $\bigopen{1 + \frac{\eta L}{2}}^{\frac{n}{2}}$.
To do this, we consider the first three terms of its binomial expansion, which is possible because $\frac{n}{2} \geq 2$:
\begin{align*}
    \bigopen{1 + \frac{\eta L}{2}}^{\frac{n}{2}} 
    \geq \bigopen{1 + \frac{2}{n}}^{\frac{n}{2}}  
    \geq 1 + \frac{2}{n} \cdot \binom{\frac{n}{2}}{1} + \bigopen{\frac{2}{n}}^2  \cdot \binom{\frac{n}{2}}{2} 
    = 1 + 1 + \frac{4}{n^2} \cdot \frac{n(n-2)}{8} = \frac{5}{2} - \frac{1}{n} \geq \frac{9}{4},
\end{align*}
where the last inequality uses $n \ge 4$.
Equivalently, the following inequality holds:
\begin{align*}
 \bigopen{1 + \frac{\eta L}{2}}^{\frac{n}{2}} \ge 2 + \frac{1}{9} \bigopen{1 + \frac{\eta L}{2}}^{\frac{n}{2}}.
\end{align*}

Using this inequality, the numerical term in \cref{eq:concavesmallepoch} becomes
\begin{align*}
    \bigopen{1 + \frac{\eta L}{2}}^{\frac{n}{2}} \bigopen{1 + (1 - \eta L)^\frac{n}{2}} - 2 
    > \bigopen{1 + \frac{\eta L}{2}}^{\frac{n}{2}} - 2 
    \ge \frac{1}{9} \bigopen{1 + \frac{\eta L}{2}}^{\frac{n}{2}}.
\end{align*}

Substituting this back to \cref{eq:concavesmallepoch} yields
\begin{align*}
    x_0^{k+1} 
    \ge \bigopen{1 + \frac{\eta L}{2}}^{\frac{n}{2}}\bigopen{1 - \eta L}^{\frac{n}{2}} x^k_0 + \frac{G}{9L} \bigopen{1 + \frac{\eta L}{2}}^{\frac{n}{2}}.
\end{align*}

Note that if $x_0^k$ is non-negative, we have $x_0^{k+1} \ge \frac{G}{9L} \bigopen{1 + \frac{\eta L}{2}}^{\frac{n}{2}} \geq 0$.
By setting the initialization point $x_0$ as $0$, each $x_0^k$ remains non-negative throughout the process, and therefore the final iterate $x_n^K$ satisfies:
\begin{align*}
    x_n^K \ge \frac{G}{9L} \bigopen{1 + \frac{\eta L}{2}}^{\frac{n}{2}} \ge \frac{G}{9L} \bigopen{1 + \frac{L}{2\mu n K}}^{\frac{n}{2}},
\end{align*}
where we apply $\eta \ge \frac{1}{\mu n K}$ at the last step.
Consequently, the optimality gap is lower bounded as:
\begin{align*}
    F_2(x_n^K) - F_2(x^*) = \frac{L}{8} (x_n^K)^2 \gtrsim \frac{G^2}{L} \bigopen{1 + \frac{L}{2\mu n K}}^{n}.
\end{align*}

\subsection{Technical Lemmas}
\label{appendix:small-lb-techlmm}
\begin{restatable}{lemma}{euler}
\label{lem:euler}
    For any $n \geq 2$, the following holds:
    \begin{align*}
        \sum_{j=0}^{n-1} e^{\frac{2 \pi j}{n} \mathrm{i}} = 0.
    \end{align*}
    where $\mathrm{i}$ denotes the imaginary unit. 
    In particular, the following equations hold:
    \begin{align*}
        \sum_{j=0}^{n-1} \cos \frac{2 \pi j}{n} = 0, \text{ and } \,\, \sum_{j=0}^{n-1} \sin \frac{2 \pi j}{n} = 0.
    \end{align*}
\end{restatable}
\begin{proof}
Let $\zeta = e^{\frac{2 \pi}{n} \mathrm{i}}$. 
Then, $\zeta^n - 1 = e^{2 \pi \mathrm{i}} - 1 = 0$ holds. 
Moreover, we have
\begin{align*}
    \zeta^n - 1 = \bigopen{\sum_{j=0}^{n-1} \zeta^j}(\zeta - 1) = 0.
\end{align*}
Since $\zeta - 1 \neq 0$, it follows that $\sum_{j=0}^{n-1} \zeta^j =0$.
This leads to the results $\sum_{j=0}^{n-1} \cos \frac{2 \pi j}{n} = 0$ and $\sum_{j=0}^{n-1} \sin \frac{2 \pi j}{n} = 0$.
\end{proof}

\begin{restatable}{lemma}{sumsquares}
\label{lem:sumsquares}
For any $n \geq 3$, the following equations hold:
    \begin{align*}
        \frac{1}{n} \sum_{j=0}^{n-1} \cos^2 \frac{2 \pi j}{n} = \frac{1}{2}, \,\, \frac{1}{n} \sum_{j=0}^{n-1} \sin^2 \frac{2 \pi j}{n} = \frac{1}{2}, \text{ and } \,\, \frac{1}{n} \sum_{j=0}^{n-1} \sin \frac{4 \pi j}{n} = 0.
    \end{align*}
\end{restatable}
\begin{proof}
First, notice that
\begin{align*}
    \cos^2 \frac{2 \pi j}{n} &= \frac{1}{2}(1 + \cos \frac{4 \pi j}{n}).\\
    \sin^2 \frac{2 \pi j}{n} &= \frac{1}{2}(1 - \cos \frac{4 \pi j}{n}).
\end{align*}
Hence, it suffices to prove $\sum_{j=0}^{n-1} \cos \frac{4 \pi j}{n} = 0$ and $\sum_{j=0}^{n-1} \sin \frac{4 \pi j}{n} = 0$. 
Let $\zeta = e^{\frac{4\pi}{n} \mathrm{i}}$ where $\mathrm{i}$ denote the imaginary number. 
Then, $\zeta^n-1 = e^{4\pi \mathrm{i}} - 1 = 0$ holds. 
Moreover, we have
\begin{align*}
    \zeta^n - 1 = \bigopen{\sum_{j=0}^{n-1} \zeta^j}(\zeta - 1) = 0.
\end{align*}
Since $\zeta \neq 1$ for $n \geq 3$, it follows that $\sum_{j=0}^{n-1}\zeta^j = 0$. 
This leads to the results $\sum_{j=0}^{n-1} \cos \frac{4 \pi j}{n} = 0$ and $\sum_{j=0}^{n-1} \sin \frac{4 \pi j}{n} = 0$.
\end{proof}

\begin{restatable}{lemma}{phiincresasing}
\label{lem:uincreasing}
    For $\eta \in \left[ \frac{1}{\mu n K}, \frac{2}{L}\right)$, $u_0 (\eta)$ is an increasing function of $\eta$.
\end{restatable}

\begin{proof}
Recall the expression for $u_0(\eta)$ given in \cref{eq:u0}:
\begin{align*}
    u_0(\eta) = \frac{\eta L' - (1 - \cos \delta)}{(1 - \cos \delta)(2 - (\mu + L')\eta) + \eta^2 \mu L'} \cdot \eta G.
\end{align*}

For simplicity of the notation, let $a = 1 - \cos \delta$, $b(\eta) = (\eta L' - a)\eta$, and $c(\eta) = a(2 - \eta(\mu + L')) + \eta^2 \mu L'$. 
Then, $u_0(\eta)$ can be expressed as $\frac{b(\eta)}{c(\eta)}G$ and $u_0'(\eta)$ becomes $\bigopen{b'(\eta)c(\eta) - b(\eta)c'(\eta)}G / c(\eta)^2$. It suffices to prove the numerator $b'(\eta)c(\eta) - b(\eta)c'(\eta)$ is non-negative.

Expanding the numerator, we obtain
\begin{align}
    b'(\eta)c(\eta) - b(\eta)c'(\eta) 
    & = (2 \eta L' - a)(\eta^2 \mu L'  - \eta a(\mu + L') + 2a) - (\eta^2 L' - \eta a)(2\eta \mu L' - a(\mu + L'))\nonumber\\
    & = (2\eta^3 \mu L'^2 - \eta^2 a L' (3\mu + 2 L') + \eta a (4L' + a(\mu + L')) - 2a^2)\nonumber\\
    & \phantom{=} - (2\eta^3 \mu L'^2 - \eta^2 aL'(3\mu + L') + \eta a^2(\mu + L'))\nonumber\\
    & = -\eta^2 a L'^2 + 4\eta a L' - 2a^2\nonumber\\
    & = a(4\eta L' - \eta^2L'^2 - 2a). \label{eq:uderivative}
\end{align}
Since $\eta L' = \frac{\eta L}{2} \leq 1$, we have $\eta^2 L'^2 \leq \eta L'$. Moreover, we have $a = 1 - \cos \delta \leq \frac{\eta L'}{2}$ from \cref{eq:u0numerator}.

Substituting these results into \cref{eq:uderivative}, we have
\begin{align*}
    b'(\eta)c(\eta) - b(\eta)c'(\eta) = a(4\eta L' - \eta^2 L'^2 - 2a) \geq a(4 \eta L' - \eta L' - \eta L') = 2 \eta a L' \geq 0.
\end{align*}
Therefore, we conclude that $u_0(\eta)$ is an increasing function with respect to $\eta$.
\end{proof}

\begin{lemma}
\label{lem:vsmallthanu}
For $\eta \in \left[ \frac{1}{\mu n K}, \frac{2}{L}\right)$, the absolute value of $v_0(\eta)$ is bounded by $u_0(\eta)$ as follows:
\begin{align*}
    \abs{v_0(\eta)} \leq \frac{8\pi K}{\kappa} u_0(\eta).
\end{align*}
\end{lemma}

\begin{proof}
Starting from \cref{eq:u0,eq:v0}, we have
\begin{align*}
    \abs{v_0(\eta)} 
    & = \frac{\sin \delta}{\eta L' - (1 - \cos \delta)} u_0(\eta)\\
    & \leq \frac{2\pi}{n} \cdot \frac{2}{\eta L'}u_0(\eta)\\
    & = \frac{4 \pi}{\eta n L'}u_0(\eta)\\
    & = \frac{8 \pi}{\eta n L} u_0(\eta),
\end{align*}
where we employ $\sin \delta \leq \delta = \frac{2\pi}{n}$, $1 - \cos \delta \leq \frac{\eta L'}{2}$ from \cref{eq:u0numerator}, and $L' = \frac{L}{2}$. 
Finally, applying the condition $\eta \geq \frac{1}{\mu n K}$ completes the proof of desired inequality.
\end{proof}
\newpage
\section{Proofs for Small Epoch Upper Bounds}
\label{appendix:small-epoch-ub}

In this section, we provide detailed proofs for \Cref{thm:small-ub-idhess}, \Cref{thm:small-ub-scvx}, and \Cref{thm:herding-at-opt} which correspond to upper bound results in the small epoch regime.

\subsection{Proof of \Cref{thm:small-ub-idhess}}
\label{subsec:small-ub-idhess}
\thmsmallubidhess*

\begin{proof}
    Since each $f_i$ has the identical Hessian, we have $\nabla^2 f_i (x) = \nabla^2 F(x)$ for every $x \in \R$. 
    Consequently, for all $i \in [n]$, we can express the gradient difference as follows:
    \begin{align*}
        \nabla f_i (x) - \nabla F(x) = \nabla f_i(x^*) - \nabla F(x^*) + \int_{x^*}^x \bigopen{\nabla^2 f_i (\alpha) - \nabla^2 F(\alpha)} \, \mathrm{d}\alpha = \nabla f_i(x^*).
    \end{align*}

    For simplicity, let $a_i = - \nabla f_i(x^*)$ for $i \in [n]$.
    Then, from the definition of $F(x) = \frac{1}{n} \sum_{i=1}^n f_i(x)$, we have $\sum_{i=1}^n a_i = 0$.
    Furthermore, it follows from \Cref{ass:grad-optimum} that $|a_i| \le G_*$.
    To further classify the indices, we define 
    \begin{align*}
        I_+ = \bigset{i \in [n] \left| \, a_i \ge 0\right.} \,\, \text{and} \,\, I_- = \bigset{i \in [n] \left| \, a_i < 0\right.}.
    \end{align*}
    Here, $I_+$ represents the collection of component functions whose minima are greater than or equal to $x^*$, while $I_-$ consists of the remaining functions.        

    We begin by presenting the following lemma:
    \begin{restatable}{lemma}{lemmapq}
        \label{lemma:pq}
        Let $p, q \in \R$ with $p < q$, and let $p^\prime$ and $q^\prime$ denote the results of performing a single step of gradient descent on a $\mu$-strongly convex and $L$-smooth 1-dimensional function $f$, starting from $p$ and $q$, respectively, with a step size $\eta < \frac{1}{L}$.
        Then, it holds that $0 < q^\prime - p^\prime \le (1 - \eta \mu) (q - p)$.
    \end{restatable}

    The proof for \Cref{lemma:pq} is presented in \Cref{appendix:small-ub-techlmm}.
    Now, let $z_0 = x^*$, initialized at the minima of the overall function $F$, and define $z_i^k$ as the $i$-th iterate of the $k$-th epoch, using the same permutations employed for $x_i^k$ but instead starting from the initial point $z_0$.
    Since the distance between $x_i^k$ and $z_i^k$ decreases by at least a factor of $(1 - \eta \mu)$ at each iteration (\Cref{lemma:pq}), we have
    \begin{align}
        \bigabs{x_n^K - z_n^K} \le (1 - \eta \mu)^{nK} \bigabs{x_0 - z_0^1} \le e^{-\eta \mu n K} \bigabs{x_0 - x^*} 
        \le \frac{G_*}{L}, \label{eq:zxgap}
    \end{align}
    where we substitute $\eta = \frac{1}{\mu n K} \max \bigset{\log \bigopen{\frac{L \bigabs{x_0 - x^*}}{G_*}}, 1}$ in the last step.
    This demonstrates that $x_n^K$ and $z_n^K$ remain sufficiently close.
    For the rest of the analysis, we mainly focus on how far $z_n^K$ can deviate from $x^*$.
    The bound for $F(x_n^K)$ will later be controlled by leveraging $L$-smoothness between $x_n^K$ and $z_n^K$.

    In the special case where $I_- = \emptyset$, all $a_i$ are equal to $0$ since $\sum_{i=1}^n a_i = 0$.
    In this scenario, $z_n^K$ remains at $x^*$ because $\nabla f_i (z_0^1) = 0$ holds for all $i \in [n]$, resulting in $z_n^K = x^*$.
    Using this, we have
    \begin{align*}
        F(x_n^K)
        &\le F(z_n^K) + \inner{\nabla F(z_n^K),\, x_n^K - z_n^K} + \frac{L}{2} \bigabs{x_n^K - z_n^K}^2 \quad (\because \, L\text{-smoothness})\\
        &= F(x^*) + \frac{L}{2} \bigopen{\frac{G_*}{L}}^2 =  F(x^*) + \frac{G_*^2}{2L} \le F(x^*) + \frac{G_*^2}{2 \mu K},
    \end{align*}
    where we apply $K \le \kappa$ in the last inequality.
    This concludes the proof for this special case.
    For the remainder of the proof, we assume $I_- \neq \emptyset$.

    For each $k \in [K]$, define $z_+^k$ as the maximum possible final iterate obtained after running \Cref{alg:permutation} starting from $x^*$, i.e., the largest value among the $(n!)^k$ possible options.
    Similarly, for each $k \in [K]$, let $z_-^k$ denote the minimum among the $(n!)^k$ options, also starting from $x^*$.
    Consequently, by convexity of $F$, $F(z_n^K)$ can naturally be upper bounded by
    \begin{align*}
        \max \bigset{F(z_+^K), F(z_-^K)}.
    \end{align*}

    The following lemma helps us to establish upper bounds for $z_+^K - x^*$ and $-(z_-^K - x^*)$.
    \begin{restatable}{lemma}{lemmazproperty}
        \label{lemma:z+-property}
        Let $\{\sigma_{k}^+\}_{k=1}^K$ denote the sequence of permutations applied over $K$ epochs to generate $z_+^K$.
        These permutations and the corresponding $z_+^K$ satisfy the following properties:
        \vspace{-2mm}
        \begin{itemize}[leftmargin=1em, itemsep=0mm]
            \item The permutations $\{\sigma_k^+\}_{k=1}^{K-1}$, applied during the first $K-1$ epochs, produce $z_+^{K-1}$.
            \item For any $k \in [K]$, all indices in $I_+$ appear before all indices in $I_-$ in the permutation $\sigma_k^+$.
            \item For any $k \in [K]$, let $z_+^{k,I_-}$ denote the $\abs{I_-}$-th iterate in the $k$-th epoch, i.e., obtained after processing all indices in $I_-$.
            Then, the inequality $z_+^{k,I_-} \le x^* \le z_+^k$ holds.
        \end{itemize}
        For $z_-^K$ and its corresponding permutations $\{\sigma_k^-\}_{k=1}^k$, the following properties hold:
        \vspace{-2mm}
        \begin{itemize}[leftmargin=1em, itemsep=0mm]
            \item The permutations $\{\sigma_k^-\}_{k=1}^{K-1}$, applied during the first $K-1$ epochs, produce $z_-^{K-1}$.
            \item For any $k \in [K]$, all indices in $I_-$ appear before all indices in $I_+$ in the permutation $\sigma_k^-$.
            \item For any $k \in [K]$, let $z_-^{k,I_+}$ denote the $\abs{I_+}$-th iterate in the $k$-th epoch, i.e., obtained after processing all indices in $I_+$.
            Then, the inequality $z_-^{k,I_+} \ge x^* \ge z_-^k$ holds.
        \end{itemize}
    \end{restatable}

    The proof for \Cref{lemma:z+-property} is presented in \Cref{appendix:small-ub-techlmm}.
    Define $\pi_+: [|I_+|] \rightarrow [n]$ as the ordering of $I_+$ used during the $K$-th epoch to generate $z_+^K$ from $z_+^{K, I_-}$.
    We then define the sequence of iterates $u_0, u_1, \dots u_{|I_+|}$ where $u_0 = x^*$ and each subsequent $u_i$ is obtained by applying a gradient update using the component function $f_{\pi_+(i)}$ to $u_{i-1}$.
    We emphasize the following two key points:
    \vspace{-2mm}
    \begin{enumerate}[itemsep=-1mm]
        \item $z_+^{K, I_-} \le x^* = u_0$.
        \item The sequences of iterates $z_+^{K, I_-}, \dots, z_+^K$ and $u_0, \dots u_{|I_+|}$ are generated by the same component function ordering.
    \end{enumerate}
    \vspace{-2mm}
    From these observations, we conclude that $z_+^K \le u_{|I_+|}$ as \Cref{lemma:pq} ensures that the relationship $p \le q$ is preserved under gradient descent (i.e., if $p \le q$, then $p' \le q'$ after each update).
    Together with $x^* \le z_+^K$ from \Cref{lemma:z+-property}, we obtain $0 \le z_+^K - x^* \le u_{|I_+|} - x^*$.
    
    Similarly, define $\pi_-: [|I_-|] \rightarrow [n]$ as the ordering of $I_-$ used during the $K$-th epoch to generate $z_-^K$ from $z_-^{K, I_+}$.
    Also, define the sequence of iterates $v_0, v_1, \dots, v_{|I_-|}$ where $v_0 = x^*$ and each subsequent $v_i$ is obtained by applying a gradient update using the component function $f_{\pi_-(i)}$ to $v_{i-1}$.
    Then, we have $x^* \ge z_-^K \ge v_{|I_-|}$, leading to $0 \le -(z_-^K - x^*) \le -(v_{|I_-|} - x^*)$.

    To summarize the process so far, we aim to upper bound $\bigabs{z_n^K - x^*}$ where $z_n^K$ is the final iterate obtained using the same permutations as $x_n^K$ but starting from $z_0 = x^*$ instead of $x_0 = x_0$.
    Since $z_+^K$ and $z_-^K$ represent the maximum and minimum possible final iterate of $z_n^K$, respectively, the followings hold:
    \begin{align*}
        \bigabs{z_n^K - x^*} \le \max \bigset{z_+^K - x^*, -(z_-^K - x^*)} \le \max \bigset{u_{|I_+|} - x^*, -(v_{|I_-|} - x^*)}
    \end{align*}
    and therefore, by convexity of $F$,
    \begin{align}
        F(z_n^K) \le \max \bigset{F(u_{|I_+|}), F(v_{|I_-|})}. \label{eq:znkbound}
    \end{align}

    We now focus on providing the upper bound for $\max \bigset{F(u_{|I_+|}), F(v_{|I_-|})}$.
    To this end, we introduce the following lemma:
    \begin{restatable}{lemma}{lemmauvbound}
        \label{lemma:uvbound}
        With a step size $\eta < \frac{1}{L}$, $0 \le \nabla F(u_i) \le 2G_*$ holds for all $i \in \{0\} \cup [|I_+|]$ and $0 \ge \nabla F(v_j) \ge -2G_*$ holds for all $j \in \{0\} \cup [|I_-|]$.
    \end{restatable}

    The proof for \Cref{lemma:uvbound} is presented in \Cref{appendix:small-ub-techlmm}.
     Using \Cref{lemma:uvbound}, we can upper bound the increments in the per-iteration function evaluation as follows:
    \begin{align}
        F(u_i) &= F(u_{i-1}) + \int_{u_{i-1}}^{u_i} \nabla F(\alpha) \, \mathrm{d}\alpha \nonumber \\
        &\le F(u_{i-1}) + \bigabs{u_i - u_{i-1}} \cdot \bigabs{\nabla F(u_{i-1} + c_i \cdot (u_i - u_{i-1}))} \nonumber \\
        &= F(u_{i-1}) + \eta \bigabs{\nabla f_{\pi_+(i)}(u_{i-1})} \cdot \bigabs{\nabla F(u_{i-1} + c_i \cdot (u_i - u_{i-1}))} \nonumber \\
        &= F(u_{i-1}) + \eta \bigabs{\nabla F(u_{i-1}) - a_{\pi_+(i)}} \cdot \bigabs{\nabla F(u_{i-1} + c_i \cdot (u_i - u_{i-1}))} \nonumber \\
        &\le F(u_{i-1}) + \eta \cdot 4G_*^2, \label{eq:periterationbound}
    \end{align}
    where $0 \le c_i \le 1$ by Mean Value Theorem.
    The last inequality follows from the bounds $0 \le \nabla F(u_{i-1}) \le 2G_*$ and $0 \le a_{\pi_+(i)} \le G_*$ (since $\pi_+(i) \in I_+$).
    Additionally, $\min \bigset{\nabla F(u_{i-1}), \nabla F(u_i)} \le \nabla F(u_{i-1} + c_i \cdot (u_i - u_{i-1})) \le \max \bigset{\nabla F(u_{i-1}), \nabla F(u_i)}$ holds as $\nabla F$ is a strictly increasing function.
    
    Unrolling \cref{eq:periterationbound} for $i = 1, 2, \dots, |I_+|$, we obtain:
    \begin{align*}
        F(u_{|I_+|}) \le F(x^*) + 4\eta n G_*^2.
    \end{align*}

    By applying a similar argument, we can derive a corresponding bound for $v_{|I_-|}$:
    \begin{align*}
        F(v_{|I_-|}) \le F(x^*) + 4\eta n G_*^2.
    \end{align*}

    Therefore, \cref{eq:znkbound} becomes
    \begin{align*}
        F(z_n^K) \le F(x^*) + 4 \eta n G_*^2.
    \end{align*}

    We now proceed to derive the upper bound for $F(x_n^K)$.
    We already established in \cref{eq:zxgap} that $\bigabs{x_n^K - z_n^K} \le G_*/L$.
    Consequently, by applying $L$-smoothness,
    \begin{align*}
        F(x_n^K) &\le F(z_n^K) + \inner{\nabla F(z_n^K), \, x_n^K - z_n^K} + \frac{L}{2} \bigabs{x_n^K - z_n^K}^2\\
        &\le \bigopen{F(x^*) + 4 \eta n G_*^2} + 2G_* \cdot \frac{G_*}{L} + \frac{L}{2} \bigopen{\frac{G_*}{L}}^2\\
        &= F(x^*) + \frac{4G_*^2}{\mu K} \max \bigset{\log \bigopen{\frac{L \bigabs{x_0 - x^*}}{G_*}}, 1} + \frac{5G_*^2}{2L},
    \end{align*}
    where we used the fact that $\bigabs{\nabla F(z_n^K)} \le \max \bigset{\nabla F(z_+^K), -\nabla F(z_-^K)} < \max \bigset{\nabla F(u_{|I_+|}), -\nabla F(v_{|I_-|})} \le 2G_*$ and $\eta = \frac{1}{\mu n K} \max \bigset{\log \bigopen{\frac{L \bigabs{x_0 - x^*}}{G_*}}, 1}$.
    Since $K \le \kappa$, we have $L \ge \mu K$, and therefore,
    \begin{align*}
        F(x_n^K) - F(x^*) \lesssim \frac{G_*^2}{\mu K}.
    \end{align*}
    This concludes the proof of \Cref{thm:small-ub-idhess}.
\end{proof}

\subsection{Proof of \Cref{thm:small-ub-scvx}}
\label{subsec:small-ub-scvx}
\thmsmallubscvx*

\begin{proof}
The original statement by Theorem~5 in (the appendix of) \citet{mishchenko2020random} holds only for \igd. 
We here extend the theorem to hold for arbitrary permutation-based SGD, and reorganize some terms to facilitate clear comparison to the proof of \Cref{thm:herding-at-opt}. 

We begin by noting the specific epoch condition stated as $K \gtrsim \frac{\kappa}{n}$ in the theorem statement:
\begin{align*}
    K \geq \frac{2\kappa}{n} \max \bigset{\log \bigopen{\frac{\norm{\vx_0 - \vx^*}\mu K}{\sqrt{\kappa} G_*}}, 1}.
\end{align*}
Under this condition, the specified step size $\eta = \frac{2}{\mu n K} \max \bigset{\log \bigopen{\frac{\norm{\vx_0 - \vx^*} \mu K}{\sqrt{\kappa} G_*}}, 1}$ satisfies $\eta \leq \frac{1}{L}$.

For each $k \in [K]$, we use the permutation $\sigma_k$ to define a sequence of iterates $\{\vx_{k, i}^*\}_{i=0}^n$ as follows:
\begin{align*}
    \vx_{k, 0}^* & = \vx^*,\\
    \vx_{k, i}^* & = \vx_{k, i-1}^* - \eta \nabla f_{\sigma_k(i)}(\vx^*).
\end{align*}
The sequence $\vx_{k, i}^*$ can be interpreted as the sequence starting from $\vx^*$ obtained by using the component gradients in the same order as $\vx^k_i$, but the gradients are being evaluated at $\vx^*$ instead of $\vx^k_{i-1}$. 
From $\sum_{i=1}^n \nabla f_{\sigma_k(i)}(\vx^*) = n \nabla F(\vx^*) = 0$, we can easily deduce that $\vx_{k, n}^* = \vx^* = \vx_{k+1, 0}^*$. 

We analyze the square norm distance $\norm{\vx_i^k - \vx_{k, i}^*}^2$ using an iteration-wise recursive inequality:
\begin{align}
    \norm{\vx_i^k - \vx_{k, i}^*}^2 & = \norm{\vx_{i-1}^k - \eta \nabla f_{\sigma_k(i)}(\vx_{i-1}^k) - \bigopen{\vx_{k, i-1}^* - \eta \nabla f_{\sigma_k(i)}(\vx^*)}}^2 \nonumber\\
    & = \norm{\vx_{i-1}^k - \vx_{k, i-1}^*}^2 - 2\eta \inner{\vx_{i-1}^k - \vx_{k, i-1}^*, \nabla f_{\sigma_k(i)}(\vx_{i-1}^k) - \nabla f_{\sigma_k(i)}(\vx^*)} \nonumber\\
    & \quad + \eta^2\norm{\nabla f_{\sigma_k(i)}(\vx_{i-1}^k) - \nabla f_{\sigma_k(i)}(\vx^*)}^2 \nonumber\\
    & \overset{(a)}{=} \norm{\vx_{i-1}^k - \vx_{k, i-1}^*}^2 - 2\eta \bigopen{D_{f_{\sigma_k(i)}}(\vx_{i-1}^k, \vx^*) + D_{f_{\sigma_k(i)}}(\vx_{k, i-1}^*,\vx_{i-1}^k) - D_{f_{\sigma_k(i)}}(\vx_{k, i-1}^*,\vx^*)} \nonumber\\
    & \quad + \eta^2\norm{\nabla f_{\sigma_k(i)}(\vx_{i-1}^k) - \nabla f_{\sigma_k(i)}(\vx^*)}^2. \label{eq:small-scvx-igd-iteration}
\end{align}
Here, $D_f(\vx, \vy) := f(\vx) - f(\vy) - \inner{\nabla f(\vy), \, \vx - \vy}$ denotes the Bregman divergence of $f$ between $\vx$ and $\vy$.
At $(a)$, we apply the three-point identity of the Bregman divergence.

The term $D_{f_{\sigma_k(i)}}(\vx_{k, i-1}^*,\vx_{i-1}^k)$ in \cref{eq:small-scvx-igd-iteration} can be bounded as follows:
\begin{align*}
    D_{f_{\sigma_k(i)}}(\vx_{k, i-1}^*,\vx_{i-1}^k) & \geq \frac{\mu}{2}\norm{\vx_{k, i-1}^* - \vx_{i-1}^k}^2,
\end{align*}
by the $\mu$-strong convexity of the component function.
Moreover, from Lemma~2.29 of \citet{garrigos2023handbook}, we have
\begin{align*}
    \norm{\nabla f_{\sigma_k(i)}(\vx_{i-1}^k) - \nabla f_{\sigma_k(i)}(\vx^*)}^2 & \leq 2L D_{f_{\sigma_k(i)}}(\vx^k_{i-1}, \vx^*).
\end{align*}

Substituting these inequalities into \cref{eq:small-scvx-igd-iteration}, we derive
\begin{align}
    \norm{\vx_i^k - \vx_{k, i}^*}^2 
    & \leq \norm{\vx_{i-1}^k - \vx_{k, i-1}^*}^2 - 2\eta(1 - \eta L) D_{f_{\sigma_k(i)}}(\vx_{i-1}^k, \vx^*) - \eta \mu \norm{\vx_{i-1}^k - \vx_{k, i-1}^*}^2 + 2\eta D_{f_{\sigma_k(i)}}(\vx_{k, i-1}^*,\vx^*) \nonumber \\
    & \overset{(a)}{\leq} (1-\eta \mu)\norm{\vx_{i-1}^k - \vx_{k, i-1}^*}^2 + 2\eta D_{f_{\sigma_k(i)}}(\vx_{k, i-1}^*,\vx^*) \nonumber \\
    & \overset{(b)}{\leq} (1-\eta \mu)\norm{\vx_{i-1}^k - \vx_{k, i-1}^*}^2 + \eta^3Ln^2G_*^2, \label{eq:small-scvx-igd-iteration2}
\end{align}
where we apply $1 - \eta L \geq 0$ and $D_{f_{\sigma_k(i)}}(\vx^k_{i-1}, \vx^*) \geq 0$ at $(a)$.
At $(b)$, we utilize the $L$-smoothness of the component function and the triangle inequality:
\begin{align*}
    D_{f_{\sigma_k(i)}}(\vx_{k, i-1}^*,\vx^*) \le \frac{L}{2} \norm{\vx_{k, i-1}^* - \vx^*}^2 = \frac{\eta^2 L}{2} \bignorm{ \, \sum_{j=1}^{i-1} \nabla f_{\sigma_k(j)}(\vx^*) \,}^2 \le \frac{\eta^2 L}{2} \cdot (nG_*)^2.
\end{align*}

Thus, by unrolling \cref{eq:small-scvx-igd-iteration2} over all $k \in [K]$ and all $i \in [n]$, and noting that $\vx_{k, n}^* = \vx^* = \vx_{k+1, 0}^*$, we obtain
\begin{align*}
    \norm{\vx_n^K - \vx^*}^2 & \leq (1 - \eta \mu)^{nK} \norm{\vx_0^1 - \vx^*}^2 + \eta^3Ln^2G_*^2 \sum_{t=1}^{nK}\bigopen{1 - \eta \mu}^{t-1}\\
    & = (1 - \eta \mu)^{nK} \norm{\vx_0 - \vx^*}^2 + \eta^3Ln^2G_*^2 \frac{1 - \bigopen{1 - \eta \mu}^{nK}}{\eta \mu}\\
    & \leq e^{-\eta \mu n K} \norm{\vx_0 - \vx^*}^2 + \frac{\eta^2Ln^2G_*^2}{\mu}.
\end{align*}

We now substitute $\eta = \frac{2}{\mu n K} \max \bigset{\log \bigopen{\frac{\norm{\vx_0 - \vx^*} \mu K}{\sqrt{\kappa} G_*}}, 1}$.
When $\norm{\vx_0 - \vx^*}$ is sufficiently large, the above inequality simplifies to
\begin{align*}
    \norm{\vx_n^K - \vx^*}^2 \leq \frac{LG_*^2}{\mu^3 K^2} \bigopen{1 + 4\log^2 \bigopen{\frac{\norm{\vx_0 - \vx^*} \mu K}{\sqrt{\kappa} G_*}}} \lesssim \frac{LG_*^2}{\mu^3 K^2}.
\end{align*}

On the other hand, when $\norm{\vx_0 - \vx^*}$ is small so that $1$ is chosen after the max operation, the above inequality simplifies to
\begin{align*}
    \norm{\vx_n^K - \vx^*}^2 
    \le \frac{1}{e^2} \cdot e^2 \frac{\kappa G_*^2}{\mu^2 K^2} + \frac{4 L G_*^2}{\mu^3 K^2} \lesssim \frac{LG_*^2}{\mu^3 K^2},
\end{align*}
where we use $\norm{\vx_0 - \vx^*} \leq e \cdot \frac{\sqrt{\kappa} G_*}{\mu K}$.

In particular, using the $L$-smoothness of $F$, the function optimality gap can be bounded as:
\begin{align*}
    F(\vx^K_n) - F(\vx^*) \le \frac{L}{2} \bignorm{\vx_n^K - \vx^*}^2 
    \lesssim \frac{L^2G_*^2}{\mu^3 K^2}.
\end{align*}
This ends the proof of \Cref{thm:small-ub-scvx}.
\end{proof}

\subsection{Proof of \Cref{thm:herding-at-opt}}
\label{subsec:herding-at-opt}
\thmherdingatopt*

\begin{proof}
We begin by noting the specific epoch condition stated as $K \gtrsim \frac{\kappa}{n}$ in the theorem statement:
\begin{align*}
    K \geq \frac{2\kappa}{n} \max \bigset{\log \bigopen{\frac{\norm{\vx_0 - \vx^*}\mu n K}{\sqrt{\kappa} H G_*}}, 1}.
\end{align*}
Under this condition, the specified step size $\eta = \frac{2}{\mu n K} \max \bigset{\log \bigopen{\frac{\norm{\vx_0 - \vx^*} \mu n K}{\sqrt{\kappa} H G_*}}, 1}$ satisfies $\eta \leq \frac{1}{L}$.

Next, we consider the scaled gradient of each component function at $\vx^*$:
\begin{align*}
    \bigset{ \frac{\nabla f_1(\vx^*)}{G_*}, \frac{\nabla f_2(\vx^*)}{G_*}, \dots, \frac{\nabla f_n(\vx^*)}{G_*} }.
\end{align*}
From \Cref{ass:grad-optimum}, we have $\norm{\nabla f_i(\vx^*)} \le G_*$ for all $i \in [n]$.
Thus, the norm of each element is bounded by $1$.
Also, since $\sum_{i=1}^n \nabla f_i(\vx^*) = n \nabla F(\vx^*) = 0$, it follows that these elements sum to $0$.
Therefore, we can apply the \herding{} algorithm, as stated in \Cref{lem:herd}, to obtain a permutation $\sigma^*: [n] \rightarrow [n]$ satisfying
\begin{align}
    \max_{i \in [n]} \bignorm{ \, \sum_{j=1}^i \nabla f_{\sigma^*(j)} (\vx^*) \, } \leq H G_*, \label{eq:herdingproperty}
\end{align}
where $H = \tilde{\gO}(\sqrt{d})$.
We will demonstrate that this permutation $\sigma^*$ is the desired one: the final iterate $\vx_n^K$ obtained by running \Cref{alg:permutation} for $K$ epochs of $\sigma^*$ satisfies the desired upper bound.

Using $\sigma^*$, we define a sequence of iterates $\{\vx_i^*\}_{i=0}^n$ as follows:
\begin{align*}
    \vx_0^* & = \vx^*,\\
    \vx_i^* & = \vx_{i-1}^* - \eta \nabla f_{\sigma^*(i)}(\vx^*).
\end{align*}
Note that the sequence is obtained by using the component gradients at the minimizer $\vx^*$.
From $\sum_{i=1}^n \nabla f_{\sigma^*(i)}(\vx^*) = n \nabla F(\vx^*) = 0$, we can easily deduce that $\vx_n^* = \vx_0^* = \vx^*$.

The proof follows the approach used in Theorem~1 in \citet{mishchenko2020random} with several modifications using the property from the \herding{} algorithm.
We analyze the square norm distance $\norm{\vx_i^k - \vx_i^*}^2$ using an iteration-wise recursive inequality:
\begin{align}
    \norm{\vx_i^k - \vx_i^*}^2 & = \norm{\vx_{i-1}^k - \eta \nabla f_{\sigma^*(i)}(\vx_{i-1}^k) - \bigopen{\vx_{i-1}^* - \eta \nabla f_{\sigma^*(i)}(\vx^*)}}^2 \nonumber\\
    & = \norm{\vx_{i-1}^k - \vx_{i-1}^*}^2 - 2\eta \inner{\vx_{i-1}^k - \vx_{i-1}^*, \nabla f_{\sigma^*(i)}(\vx_{i-1}^k) - \nabla f_{\sigma^*(i)}(\vx^*)} \nonumber\\
    & \quad + \eta^2\norm{\nabla f_{\sigma^*(i)}(\vx_{i-1}^k) - \nabla f_{\sigma^*(i)}(\vx^*)}^2 \nonumber\\
    & \overset{(a)}{=} \norm{\vx_{i-1}^k - \vx_{i-1}^*}^2 - 2\eta \bigopen{D_{f_{\sigma^*(i)}}(\vx_{i-1}^k, \vx^*) + D_{f_{\sigma^*(i)}}(\vx_{i-1}^*,\vx_{i-1}^k) - D_{f_{\sigma^*(i)}}(\vx_{i-1}^*,\vx^*)} \nonumber\\
    & \quad + \eta^2\norm{\nabla f_{\sigma^*(i)}(\vx_{i-1}^k) - \nabla f_{\sigma^*(i)}(\vx^*)}^2. \label{eq:herdatoptiteration}
\end{align}
Here, $D_f(\vx, \vy) := f(\vx) - f(\vy) - \inner{\nabla f(\vy), \, \vx - \vy}$ denotes the Bregman divergence of $f$ between $\vx$ and $\vy$.
At $(a)$, we apply the three-point identity of the Bregman divergence.

The term $D_{f_{\sigma^*(i)}}(\vx_{i-1}^*,\vx_{i-1}^k)$ in \cref{eq:herdatoptiteration} can be bounded as follows:
\begin{align*}
    D_{f_{\sigma^*(i)}}(\vx_{i-1}^*,\vx_{i-1}^k) & \geq \frac{\mu}{2}\norm{\vx_{i-1}^* - \vx_{i-1}^k}^2,
\end{align*}
by the $\mu$-strong convexity of the component function.
Moreover, from Lemma~2.29 of \citet{garrigos2023handbook}, we have
\begin{align*}
    \norm{\nabla f_{\sigma^*(i)}(\vx_{i-1}^k) - \nabla f_{\sigma^*(i)}(\vx^*)}^2 & \leq 2L D_{f_{\sigma^*(i)}}(\vx^k_{i-1}, \vx^*).
\end{align*}

Substituting these inequalities into \cref{eq:herdatoptiteration}, we derive
\begin{align}
    \norm{\vx_i^k - \vx_i^*}^2 
    & \leq \norm{\vx_{i-1}^k - \vx_{i-1}^*}^2 - 2\eta(1 - \eta L) D_{f_{\sigma^*(i)}}(\vx_{i-1}^k, \vx^*) - \eta \mu \norm{\vx_{i-1}^k - \vx_{i-1}^*}^2 + 2\eta D_{f_{\sigma^*(i)}}(\vx_{i-1}^*,\vx^*) \nonumber \\
    & \overset{(a)}{\leq} (1-\eta \mu)\norm{\vx_{i-1}^k - \vx_{i-1}^*}^2 + 2\eta D_{f_{\sigma^*(i)}}(\vx_{i-1}^*,\vx^*) \nonumber \\
    & \overset{(b)}{\leq} (1-\eta \mu)\norm{\vx_{i-1}^k - \vx_{i-1}^*}^2 + H^2\eta^3LG_*^2, \label{eq:herdatoptiteration2}
\end{align}
where we apply $1 - \eta L \geq 0$ and $D_{f_{\sigma^*(i)}}(\vx^k_{i-1}, \vx^*) \geq 0$ at $(a)$.
At $(b)$, we utilize the $L$-smoothness of the component function and the property of the \herding{} algorithm, given in \cref{eq:herdingproperty}:
\begin{align*}
    D_{f_{\sigma^*(i)}}(\vx_{i-1}^*,\vx^*) \le \frac{L}{2} \norm{\vx_{i-1}^* - \vx^*}^2 = \frac{\eta^2 L}{2} \bignorm{ \, \sum_{j=1}^{i-1} \nabla f_{\sigma^*(j)}(\vx^*) \,}^2 \le \frac{\eta^2 L}{2} \cdot (HG_*)^2.
\end{align*}

Thus, by unrolling \cref{eq:herdatoptiteration2} over all $k \in [K]$ and all $i \in [n]$, and noting that $\vx_n^* = \vx_0^* = \vx^*$, we obtain
\begin{align*}
    \norm{\vx_n^K - \vx^*}^2 & \leq (1 - \eta \mu)^{nK} \norm{\vx_0^1 - \vx^*}^2 + H^2\eta^3LG_*^2 \sum_{t=1}^{nK}\bigopen{1 - \eta \mu}^{t-1}\\
    & = (1 - \eta \mu)^{nK} \norm{\vx_0 - \vx^*}^2 + H^2\eta^3LG_*^2 \frac{1 - \bigopen{1 - \eta \mu}^{nK}}{\eta \mu}\\
    & \leq e^{-\eta \mu n K} \norm{\vx_0 - \vx^*}^2 + \frac{H^2\eta^2LG_*^2}{\mu}.
\end{align*}

We now substitute $\eta = \frac{2}{\mu n K} \max \bigset{\log \bigopen{\frac{\norm{\vx_0 - \vx^*} \mu n K}{\sqrt{\kappa} H G_*}}, 1}$.
When $\norm{\vx_0 - \vx^*}$ is sufficiently large, the above inequality simplifies to
\begin{align*}
    \norm{\vx_n^K - \vx^*}^2 \leq \frac{H^2LG_*^2}{\mu^3 n^2 K^2} \bigopen{1 + 4\log^2 \bigopen{\frac{\norm{\vx_0 - \vx^*} \mu n K}{\sqrt{\kappa} H G_*}}} \lesssim \frac{H^2LG_*^2}{\mu^3 n^2 K^2}.
\end{align*}

On the other hand, when $\norm{\vx_0 - \vx^*}$ is small so that $1$ is chosen after the max operation, the above inequality simplifies to
\begin{align*}
    \norm{\vx_n^K - \vx^*}^2 
    \le \frac{1}{e^2} \cdot e^2 \frac{\kappa H^2 G_*^2}{\mu^2 n^2 K^2} + \frac{4 H^2 L G_*^2}{\mu^3 n^2 K^2} \lesssim \frac{H^2LG_*^2}{\mu^3 n^2 K^2},
\end{align*}
where we use $\norm{\vx_0 - \vx^*} \leq e \cdot \frac{\sqrt{\kappa} H G_*}{\mu n K}$.

In particular, using the $L$-smoothness of $F$, the function optimality gap can be bounded as:
\begin{align*}
    F(\vx^K_n) - F(\vx^*) \le \frac{L}{2} \bignorm{\vx_n^K - \vx^*}^2 \lesssim \frac{H^2L^2G_*^2}{\mu^3 n^2 K^2}.
\end{align*}
This ends the proof of \Cref{thm:herding-at-opt}.
\end{proof}

\subsection{Technical Lemmas}
\label{appendix:small-ub-techlmm}

\lemmapq*

\begin{proof}
    Using the gradient descent update rule, we obtain:
    \begin{align*}
        &p^\prime = p - \eta \nabla f(p),\\
        &q^\prime = q - \eta \nabla f(q).
    \end{align*}
    
    The difference between $q^\prime$ and $p^\prime$ can then be written as:
    \begin{align}
        q^\prime - p^\prime &= (q - p) - \eta \bigopen{\nabla f(q) - \nabla f(p)} \nonumber \\
        &= (q - p) - \eta \int_p^q \nabla^2 f(u) \, \mathrm{d}u. \label{eq:pqgap}
    \end{align}
    Since $\nabla^2 f (u)$ satisfies $\mu \le \nabla^2 f (u) \le L$, we have $\mu (q - p) \le \int_p^q \nabla^2 f(u) \, \mathrm{d}u \le L (q - p)$.
    Substituting this inequality to \cref{eq:pqgap} yields
    \begin{align*}
        0 < (1 - \eta L) (q - p) \le q^\prime - p^\prime \le (1 - \eta \mu)(q - p),
    \end{align*}
    where the first inequality holds due to $\eta < \frac{1}{L}$.
\end{proof}

\lemmazproperty*

\begin{proof}
    We provide the proof for $z_+^K$ and its corresponding permutations $\{\sigma_k^+\}_{k=1}^K$.
    The proof for $z_-^K$ and $\{\sigma_k^-\}_{k=1}^K$ is analogous, as flipping the sign of $a$'s leads to identical circumstances.

    \paragraph{Step 1: The First Property.}
    Let $w_+^{K-1}$ denote the iterate obtained by running \Cref{alg:permutation} with the sequence of permutations $\{\sigma_k^+\}_{k=1}^{K-1}$, starting from $x^*$ with a step size $\eta$.
    Since $z_+^{K-1}$ is defined as the maximum possible iterate after running \Cref{alg:permutation} with $K-1$ epochs, it follows that $w_+^{K-1} \le z_+^{K-1}$.
    
    Assume for contradiction that $w_+^{K-1} < z_+^{K-1}$.
    By \Cref{lemma:pq}, the iterate obtained by applying $\sigma_+^K$ starting from $z_+^{K-1}$ exceeds $z_+^K$.
    This contradicts the definition of $z_+^K$, which is the maximum possible final iterate after $K$ epochs.
    Therefore, we conclude that $w_+^{k-1} = z_+^{k-1}$.
    
    By recursively applying this reasoning, we deduce that for all $l \in [K]$, running \Cref{alg:permutation} with permutations $\{\sigma_k^+\}_{k=1}^l$ generates $z_+^l$.

    \paragraph{Step 2: The Second Property.}
    We now prove the following claim:
    
    \tightparagraph{Claim.}
    \textit{Consider two steps of gradient updates using two component functions $f_i(x)$ and $f_j(x)$ with $a_i < a_j$, starting from the initialization $u$.
    Then, regardless of the choice of the step size $\eta$, applying $f_i$ first, followed by $f_j$, results in a larger iterate than applying $f_j$ first, followed by $f_i$.}
    
    \textit{Proof of the claim.}
        The update equations are:
        \begin{align*}
            \begin{aligned}[c]
                &u_i = u - \eta \bigopen{\nabla F(u) - a_i}, \\
                &u_{ij} = u_i - \eta \bigopen{\nabla F(u_i) - a_{j}},
            \end{aligned}
            \qquad \qquad
            \begin{aligned}[c]
                &u_j = u - \eta \bigopen{\nabla F(u) - a_j}, \\
                &u_{ji} = u_j - \eta \bigopen{\nabla F(u_j) - a_{i}}.
            \end{aligned}
        \end{align*}
        Since $a_i < a_j$, we have $u_i < u_j$.
        Also, because $\nabla F$ is a monotonically increasing function, it follows that $\nabla F(u_i) < \nabla F(u_j)$.
        Now, we can check that subtracting $u_{ji}$ from $u_{ij}$ yields positive difference:
        \begin{align*}
            u_{ij} - u_{ji}
            &= \bigopen{ u_i - \eta \bigopen{\nabla F(u_i) - a_{j}}} - \bigopen{u_j - \eta \bigopen{\nabla F(u_j) - a_{i}}}\\
            &= \eta \bigopen{\nabla F(u_j) - \nabla F(u_i)} + \underbrace{(u_i + \eta a_j) - (u_j + \eta a_i)}_{=0}\\
            &= \eta \bigopen{\nabla F(u_j) - \nabla F(u_i)} > 0.
        \end{align*}
        Thus, $u_{ji} > u_{ij}$ holds, completing the proof of the claim.
    \qed

    From the claim, we conclude in $\sigma_k^+$, all indices in $I_-$ (indices corresponding to negative $a$ values) must appear before indices in $I_+$ (indices corresponding to positive $a$ values).
    Otherwise, if there exists an index in $I_-$ that immediately follows an index in $I_+$, switching these two indices would result in a larger final iterate (due to \Cref{lemma:pq}), contradicting the optimality of $\sigma_k^+$.
    This concludes the proof of the second property.

    \paragraph{Step 3: The Third Property.}
    Define $M := \sum_{i \in I_+} a_i = -\sum_{i \in I_-} a_i$.
    We claim that:
    
    \tightparagraph{Claim.} \textit{If $0 \le z_+^{k} - x^* \le \eta M$, then $-\eta M \le z_+^{k+1, I_-} - x^* \le 0$ holds.}

    \textit{Proof of the claim.}
        Note that the iterate $z_+^{k+1, I_-}$ is obtained by applying gradient update starting from $z_+^k$ using the first $I_-$ component functions of the permutation $\sigma_k$.
        Let $\sigma_k^f$ denote the first $I_-$ parts of the permutation $\sigma_k$.
        We verify the bound as follows:

        \textbf{Lower Bound: $-\eta M \le z_+^{k+1, I_-} - x^*$.}
        
        By \Cref{lemma:pq}, the iterate $z_+^{k+1, I_-}$ is at least as large as the iterate obtained by applying gradient updates following $\sigma_k^f$, starting from $x^*$.
        
        Also, if $p < x^*$ holds, then
        \begin{align*}
            - \nabla F(p) = \nabla F(x^*) - \nabla F(p) = \int_{p}^{x^*} \nabla^2 F(\alpha) \, \mathrm{d}\alpha \le L(x^* - p).
        \end{align*}
        Hence, $p - \eta \nabla F(p) \le x^*$ holds.
        Thus, if the iterate falls below $x^*$, the next iterate obtained by applying the gradient update from the component in $I_-$ will also remain below $x^*$.

        This property guarantees that when the gradient update starts $x^*$ and follows $\sigma_k^f$, every iterate remains below $x^*$.
        Moreover, the total contribution of the gradient updates towards the negative direction by indices in $I_-$ when starting from $x^*$ is at most $- \eta \sum_{i \in I_-} a_i = \eta M$.
        Hence, $z_+^{k+1, I_-} - x^* \ge - \eta M$ holds.

        \textbf{Upper Bound: $z_+^{k+1, I_-} - x^* \le 0$.}

        Again, by \Cref{lemma:pq}, the iterate $z_+^{k+1, I_-}$ is at most the iterate obtained by applying gradient updates following $\sigma_k^f$, starting from $x^* + \eta M$.

        Assume by contradiction that $z_+^{k+1, I_-} > x^*$ holds.
        This means that the iterate obtained by following $\sigma_k^f$ starting from $x^* + \eta M$ is also greater than $x^*$.
        Due to the property stated in the proof of lower bounding $z_+^{k+1, I_-}$, all intermediate iterates should be greater than $x^*$ as well.
        This leads to a contradiction, as the total contribution of the gradient updates towards the negative direction by indices in $I_-$ when starting from $x^* + \eta M$ will exceed $\eta M$, leading $z_+^{k+1, I_-}$ to fall below $x^*$.
        Hence, $z_+^{k+1, I_-} - x^* \le 0$ holds.
        
        Combining these two bounds, we obtain
        \begin{align*}
            -\eta M \le z_+^{k+1, I_-} - x^* \le 0,
        \end{align*}
        and this ends the proof of the claim.
    \qed

    The claim shows that if $0 \le z_+^{k} - x^* \le \eta M$, then $-\eta M \le z_+^{k+1, I_-} - x^* \le 0$ holds.
    By analogous reasoning, if $-\eta M \le z_+^{k+1, I_-} - x^* \le 0$, then $0 \le z_+^{k+1} - x^* \le \eta M$ holds.
    Combining these two statements, we have: \textit{if $0 \le z_+^{k} - x^* \le \eta M$, then $0 \le z_+^{k+1} - x^* \le \eta M$ and $z_+^{k+1, I_-} \le x^* \le z_+^{k+1}$ hold.}

    Using these, we now proceed by induction to prove the third property.
    For the base case, the initialization point is $z_0 = x^*$, satisfying the initial condition by $z_0 - x^* = 0$.
    By induction, it follows that
    \begin{align*}
        z_+^{k, I_-} \le x^* \le z_+^k.
    \end{align*}
    for all $k \in [K]$.
    This concludes the proof of the third property. 
\end{proof}

\lemmauvbound*

\begin{proof}
    Recall that the sequence of iterate $\{u_i\}_{i=0}^{|I_+|}$ is defined as $u_0 = x^*$ and each subsequent $u_i$ is obtained by applying a gradient update using the component function $f_{\pi_+(i)}$ to $u_{i-1}$.
    Specifically, we have
    \begin{align*}
        u_i 
        &= u_{i-1} - \eta \nabla f_{\pi_+(i)} (u_{i-1})\\
        &= u_{i-1} - \eta \bigopen{\nabla F (u_{i-1}) - a_{\pi_+(i)}},
    \end{align*}
    for $i \in [|I_+|]$.
    
    Now, we will prove by induction that $0 \le \nabla F(u_i) \le 2G_*$ holds for all $i \in [|I_+|]$.
    Initially, we have $u_0 = x^*$ and thus $\nabla F(u_0) = 0$.
    Now, assume that $0 \le \nabla F(u_{j-1}) \le 2G_*$.
    We divide the proof into two cases based on the value of $\nabla F(u_{j-1})$.

    \textbf{Case 1.} $\nabla F(u_{j-1}) \le a_{\pi_+(j)}$.

    In this case, the update equation becomes:
    \begin{align*}
        u_j = u_{j-1} - \eta \bigopen{\nabla F (u_{j-1}) - a_{\pi_+(j)}} \ge u_{j-1},
    \end{align*}
    meaning that the iterate increases.
    Since $\nabla F$ is an increasing function, we have $\nabla F(u_j) \ge \nabla F(u_{j-1}) \ge 0$.

    Also, using the fact that all $\bigabs{a_i}$ is bounded by $G_*$, we can bound the difference of the gradient between successive iterates via the $L$-smoothness of $F$:
    \begin{align*}
        \bigabs{\nabla F(u_j) - \nabla F(u_{j-1})} \le L \bigabs{u_j - u_{j-1}} \le \eta L G_* < G_*.
    \end{align*}
    
    Thus, the deviation of $\nabla F(u_j)$ from $\nabla F(u_{j-1})$ is at most $G_*$, leading to the following inequality:
    \begin{align*}
        \nabla F(u_j) \le \nabla F(u_{j-1}) + G_* \le a_{\pi_+(j)} + G_* \le 2G_*.
    \end{align*}

    \textbf{Case 2.} $\nabla F(u_{j-1}) > a_{\pi_+(j)}$.

    In this case, the update equation becomes:
    \begin{align*}
        u_j = u_{j-1} - \eta \bigopen{\nabla F (u_{j-1}) - a_{\pi_+(j)}} \le u_{j-1},
    \end{align*}
    meaning that the iterate decreases.
    Since $\nabla F$ is an increasing function, we have $\nabla F(u_j) \le \nabla F(u_{j-1}) \le 2G_*$.

    Furthermore, by $L$-smoothness of $F$, we have $\nabla F(u_{j-1}) = \nabla F(u_{j-1}) - \nabla F(x^*) \le L \bigopen{u_{j-1} - x^*}$.
    Then, we can ensure that $u_j$ is greater than or equal to $x^*$ as follows:
    \begin{align*}
        u_j
        &= u_{j-1} - \eta \nabla \bigopen{F(u_{j-1}) - a_{\pi_+(j)}}\\
        &\ge u_{j-1} - \eta \nabla F(u_{j-1})\\
        &\ge u_{j-1} - \frac{1}{L} \cdot L \bigopen{u_{j-1} - x^*} = x^*.
    \end{align*}

    For both cases, we have shown that $0 \le \nabla F(u_j) \le 2G_*$.
    
    We can apply the same approach for $\{v_i\}_{i=1}^{|I_-|}$.
    The key difference is that the sign of $a_{\pi_{-}(j)}$ is negative.
    This leads to the result $0 \ge \nabla F(v_j) \ge -2G_*$ for all $j \in [|I_-|]$.
    This concludes the proof of \Cref{lemma:uvbound}.
\end{proof}
\newpage
\section{Proofs for Large Epoch Lower Bounds}
\label{appendix:large-epoch-lb}

\subsection{Proof of \Cref{thm:large-lb-idhess}}
\label{subsec:large-lb-idhess}
\thmlargelbidhess*

\begin{proof}
Similar to the approach in \Cref{thm:small-lb-idhess}, we divide the range of step size into three regimes. 
For each regime, we construct the overall function $F_1, F_2$, and $F_3$, respectively, along with their respective component functions and an initial point. Finally, we aggregate these functions across different dimensions to derive the stated lower bound. 

Each overall function is $1$-dimensional, and carefully designed to satisfy the following properties:
\begin{itemize}
    \item (Small step size regime) There exists an initialization point $x_0 = \mathrm{poly}(\mu, L, n, K, G)$ such that for any choice of $\eta \in \bigopen{0, \frac{1}{\mu n K}}$, the final iterate $x^K_n$ obtained by running \Cref{alg:igd} satisfies 
    $F_1(x_n^K) - F_1(x^*) \gtrsim \frac{L G^2}{\mu^2 K^2}$.
    \item (Moderate step size regime) There exists an initialization point $y_0 = \mathrm{poly}(\mu, L, n, K, G)$ such that for any choice of $\eta \in \left[\frac{1}{\mu n K}, \frac{2}{L}\right) $, the final iterate $y^K_n$ obtained by running \Cref{alg:igd} satisfies $F_2(y_n^K) - F_2(y^*) \gtrsim \frac{LG^2}{\mu^2K^2}$.
    \item (Large step size regime) There exists an initialization point $z_0 = \mathrm{poly}(\mu, L, n, K, G)$ such that for any choice of $\eta \in \left[\frac{2}{L}, \infty\right)$, the final iterate $z^K_n$ obtained by running \Cref{alg:igd} satisfies $F_3(z_n^K) - F_3(z^*) \gtrsim \frac{LG^2}{\mu^2K^2}$.
\end{itemize}
Here, $x^*$, $y^*$, $z^*$ denote the minimizers of $F_1$, $F_2$, and $F_3$, respectively. 
All these functions are designed to satisfy \Cref{ass:common}.
$F_1$ and $F_3$ satisfy
\Cref{ass:grad-generalized} with $G=P=0$, and $F_2$ satisfies with $P=0$. 
Moreover, each component function within each overall function shares the same Hessian. 
Detailed constructions of $F_1$, $F_2$, and $F_3$, as well as the verification of the assumptions and the stated properties are presented in \Cref{subsubsec:large-lb-idhess-f1,subsubsec:large-lb-idhess-f2,subsubsec:large-lb-idhess-f3}.

By following a similar approach to the proof of \Cref{thm:small-lb-idhess,thm:small-lb-sc}, we can conclude that the aggregated 3-dimensional function $F(\vx) := F(x, y, z) = F_1(x) + F_2(y) + F_3(z)$ and its component functions satisfy the stated assumptions. 
Also, since each dimension is independent, it is obvious that $\vx^* = (x^*, y^*, z^*)$ minimizes $F$.
Finally, by choosing the initialization point as $\vx_0 = (x_0, y_0, z_0)$, the final iterate $\vx^K_n = (x^K_n, y^K_n, z^K_n)$ obtained by running \Cref{alg:igd} on $F$ satisfies
\begin{align*}
    F(\vx_n^K) - F(\vx^*) \gtrsim \frac{LG^2}{\mu^2 K^2},
\end{align*}
regardless of the choice of $\eta > 0$.

This concludes the proof of \Cref{thm:large-lb-idhess}.
\end{proof}

In the following subsections, we present the specific construction of $F_1$, $F_2$, and $F_3$, and demonstrate that each satisfies the stated lower bound within its corresponding step size regime.
For simplicity of notation, we omit the index of the overall function when referring to its component functions, e.g., we write $f_i(x)$ instead of $f_{1i}(x)$. 
Moreover, we use the common variable notation $x$ while constructing functions for each dimension, though we use different variables in the ``dimension-aggregation'' step.

\subsubsection{Construction of $F_1$}
\label{subsubsec:large-lb-idhess-f1}
Let $F_1(x) = \frac{\mu}{2}x^2$ with component functions $f_i(x) = F_1(x)$ for all $i \in [n]$.
It is clear that $F_1$ satisfies \Cref{ass:common} and \Cref{ass:grad-generalized} with $G = P = 0$, and its component functions share an identical Hessian.
Also, we note that $x^* = 0$ and $F_1(x^*) = 0$.

Let the initialization be $x_0 = \sqrt{\kappa} \frac{G}{\mu K}$.
For all $\eta \in \bigopen{0, \frac{1}{\mu n K}}$, the final iterate is given by
\begin{align*}
    x_n^K = (1 - \eta \mu)^{nK} x_0 \geq \bigopen{1 - \frac{1}{nK}}^{nK} x_0 \geq \frac{\sqrt{\kappa}G}{4 \mu K},
\end{align*}
where the last inequality uses the fact that $(1 - \frac{1}{m})^m \geq \frac{1}{4}$ for all $m \geq 2$.

Thus, we have
\begin{align*}
    F_1(x_n^K) - F_1(x^*) = \frac{\mu}{2} (x_n^K)^2 \gtrsim \frac{LG^2}{\mu^2 K^2}.
\end{align*}

\subsubsection{Construction of $F_2$}
\label{subsubsec:large-lb-idhess-f2}
In this subsection, we let $L'$ denote $L/2$.
We construct the function by dividing the cases by the parity of $n$.
We first consider the case where $n$ is even, and address the case where $n$ is odd later in this subsection.
Let $F_2(x) = \frac{L'}{2}x^2$ with component functions
\begin{align*}
    f_i(x) = \begin{cases}
        \frac{L'}{2}x^2 + Gx & \textrm{ if } \, i \leq n/2,\\
        \frac{L'}{2}x^2 - Gx & \textrm{ otherwise}.
    \end{cases}
\end{align*}

Since $\kappa \ge 2$, we have $L' = \frac{L}{2} \ge \mu$.
Thus, it is clear that $f_i$ satisfies \Cref{ass:common,ass:grad-generalized} with $P=0$, and shares the same Hessian.
By \Cref{lem:quadratic-twotype-IGD-closed}, the final iterate obtained by running \Cref{alg:igd} is given by
\begin{align*}
    x_n^K = (1 - \eta L')^{nK}x_0 + \frac{G}{L'}\cdot \frac{1 - (1 - \eta L')^{\frac{n}{2}}}{1 + (1 - \eta L')^{\frac{n}{2}}}\bigopen{1 - (1 - \eta L')^{nK}}.
\end{align*}

By applying $\eta \ge \frac{1}{\mu n K}$ and setting $x_0 = 0$, we derive
\begin{align}
    x_n^K\
    &= \frac{G}{L'}\cdot \frac{1 - (1 - \eta L')^{\frac{n}{2}}}{1 + (1 - \eta L')^{\frac{n}{2}}}\bigopen{1 - (1 - \eta L')^{nK}} \nonumber \\
    & \geq \frac{G}{2L'}\bigopen{1 - (1 - \eta L')^{\frac{n}{2}}}\bigopen{1 - (1 - \eta \mu)^{nK}} \nonumber \\
    & \geq \frac{G}{2L'}\bigopen{1 - (1 - \eta L')^{\frac{n}{2}}}\bigopen{1 - \bigopen{1 - \frac{1}{nK}}^{nK}} \nonumber \\
    & \geq \frac{G}{2L'}\bigopen{1 - (1 - \eta L')^{\frac{n}{2}}}\bigopen{1 - e^{-1}}. \label{eq:large-idhesslb-x}
\end{align}

We analyze \cref{eq:large-idhesslb-x} by dividing the range of $\eta$ into two regimes.

\textbf{Regime 1.} $\eta \in \left[\frac{1}{\mu n K}, \frac{1}{nL'}\right)$.

In this regime, we can bound $1 - (1 - \eta L')^{\frac{n}{2}}$ as:
\begin{align*}
    1 - (1 - \eta L')^{\frac{n}{2}} \geq 1 - e^{-\frac{\eta n L'}{2}} \geq 1 - \bigopen{1 - \frac{\eta n L'}{4}} = \frac{\eta n L'}{4} \geq \frac{L'}{4 \mu K},
\end{align*}
where the second inequality uses $e^{-u} \leq 1 - \frac{u}{2}$ for all $u \in [0, 1]$. 
Substituting this inequality into \cref{eq:large-idhesslb-x} gives
\begin{align*}
    x_n^K \geq \frac{(1 - e^{-1})G}{8 \mu K}.
\end{align*}

Consequently, the function optimality gap satisfies
\begin{align*}
    F_2(x_n^K) - F_2(x^*) = \frac{L'}{2} (x_n^K)^2 \gtrsim \frac{L G^2}{\mu^2 K^2}.
\end{align*}

\textbf{Regime 2.} $\eta \in \left[\frac{1}{nL'}, \frac{1}{L'}\right)$.

In this regime, we can bound $1 - (1 - \eta L')^{\frac{n}{2}}$ as:
\begin{align*}
    1 - (1 - \eta L')^{\frac{n}{2}} \geq 1 - \bigopen{1 - \frac{1}{n}}^{\frac{n}{2}} \geq 1 - e^{-\frac{1}{2}}.
\end{align*}

Substituting this inequality into \cref{eq:large-idhesslb-x} gives
\begin{align*}
    x_n^K \geq \frac{\bigopen{1 - e^{-1}}\bigopen{1 - e^{-\frac{1}{2}}}G}{2L'}.
\end{align*}

Since $K \geq \kappa$, we have $\frac{1}{L'} = \frac{2}{L} \geq \frac{2}{\mu K}$.
Therefore, the final iterate $x_n^K$ can be bounded as:
\begin{align*}
    x_n^K \geq \frac{\bigopen{1 - e^{-1}}\bigopen{1 - e^{-\frac{1}{2}}}G}{\mu K}.
\end{align*}

Consequently, the function optimality gap satisfies
\begin{align*}
    F_2(x_n^K) - F_2(x^*) = \frac{L'}{2} (x_n^K)^2 \gtrsim \frac{L G^2}{\mu^2 K^2}.
\end{align*}

We now focus on the case where $n$ is odd.
Let $F_2(x) = \frac{L'}{2}x^2$ with component functions
\begin{align*}
    f_i(x) = 
    \begin{cases}
        \frac{L'}{2}x^2 & \textrm{ if } \, i=1,\\
        \frac{L'}{2}x^2 + Gx & \textrm{ if } \, 2 \leq i \leq (n+1)/2,\\
        \frac{L'}{2}x^2 - Gx & \textrm{ if } \, (n+3)/2 \leq i \leq n.
    \end{cases}
\end{align*}
Compared to the case of even $n$, $f_1(x) = \frac{L'}{2} x^2$ is introduced newly.
By \Cref{lem:quadratic-threetype-IGD-closed}, the final iterate $x_n^K$ obtained by running \Cref{alg:igd} satisfies the following equation:
\begin{align*}
    x_n^K = (1 - \eta L')^{nK} x_0 + \frac{G}{L'} \cdot \frac{1 - (1 - \eta L')^{nK}}{1 - (1 - \eta L')^n}\bigopen{1 - (1 - \eta L')^{\frac{n-1}{2}}}^2.
\end{align*}

By applying $\eta \ge \frac{1}{\mu n K}$ and setting $x_0  = 0$, we have
\begin{align*}
    x^K_n & = \frac{G}{L'} \cdot \frac{1 - (1 - \eta L')^{nK}}{1 - (1 - \eta L')^n}\bigopen{1 - (1 - \eta L')^{\frac{n-1}{2}}}^2 \\
    &  = \frac{G}{L'}\bigopen{1 - (1 - \eta L')^{nK}}\frac{1 - (1 - \eta L')^{n-1}}{1 - (1 - \eta L')^n}\frac{\bigopen{1 - (1 - \eta L')^{\frac{n-1}{2}}}^2}{1 - (1 - \eta L')^{n-1}}\\
    & \ge \frac{G}{L'}\bigopen{1 - (1 - \eta \mu)^{nK}} \frac{1 - (1 - \eta L')^{n-1}}{1 - (1 - \eta L')^n}\frac{1 - (1 - \eta L')^{\frac{n-1}{2}}}{1  + (1 - \eta L')^{\frac{n-1}{2}}}\\
    & \geq \frac{G}{2L'} \bigopen{1 - e^{-\eta \mu n K}} \frac{1 - (1 - \eta L')^{n-1}}{1 - (1 - \eta L')^n}\bigopen{1 - (1 - \eta L')^{\frac{n-1}{2}}}\\
    & \geq \frac{G}{2L'} \bigopen{1 - e^{-1}}\frac{1 - (1 - \eta L')^{n-1}}{1 - (1 - \eta L')^n}\bigopen{1 - (1 - \eta L')^{\frac{n-1}{2}}}.
\end{align*}

Note that the inequality
\begin{align*}
    \frac{1 - (1 - \eta L')^{n-1}}{1 - (1 - \eta L')^n} \geq \frac{1}{2}
\end{align*}
holds for $n \ge 2$ since
\begin{align*}
    2 - 2(1 - \eta L')^{n - 1} \geq 1 - (1 - \eta L')^n & \Leftrightarrow 1 \geq 2(1 - \eta L')^{n-1} - (1 - \eta L')^n\\
    & \Leftrightarrow 1 \geq (1 - \eta L')^{n-1}(2 - (1 - \eta L')) \\
    & \Leftrightarrow 1 \ge (1 - \eta L')^{n-2} (1 - \eta^2 L'^2).
\end{align*}

Hence, we deduce that
\begin{align}
    x^K_n \geq \frac{G}{4L'}\bigopen{1 - e^{-1}}\bigopen{1 - (1 - \eta L')^{\frac{n-1}{2}}}.\label{eq:large-idhesslb-x2}
\end{align}

We again analyze \cref{eq:large-idhesslb-x2} by dividing the range of $\eta$ into two regimes.

\textbf{Regime 1.} $\eta \in \left[\frac{1}{\mu n K}, \frac{1}{n L'}\right)$.

In this regime, we can bound $1 - (1 - \eta L')^{\frac{n-1}{2}}$ as:
\begin{align*}
    1 - (1 - \eta L')^{\frac{n-1}{2}} \geq 1 - e^{-\frac{\eta (n-1) L'}{2}} \geq 1 - \bigopen{1 - \frac{\eta (n-1) L'}{4}} = \frac{\eta (n-1) L'}{4} \geq \frac{\eta n L'}{8} \geq \frac{L'}{8 \mu K},
\end{align*}
where the second inequality uses $e^{-u} \leq 1 - \frac{u}{2}$ for all $u \in [0, 1]$.
Substituting this inequality into \cref{eq:large-idhesslb-x2} gives
\begin{align*}
    x^K_n \geq \frac{\bigopen{1 - e^{-1}}G}{32 \mu K}.
\end{align*}

Consequently, the function optimality gap satisfies
\begin{align*}
    F_2(x_n^K) - F_2(x^*) = \frac{L'}{2} (x_n^K)^2 \gtrsim \frac{L G^2}{\mu^2 K^2}.
\end{align*}

\textbf{Regime 2.} $\eta \in \left[\frac{1}{nL'}, \frac{1}{L'}\right)$.

In this regime, we can bound $1 - (1 - \eta L')^{\frac{n-1}{2}}$ as:
\begin{align*}
    1 - (1 - \eta L')^{\frac{n-1}{2}} \geq 1 - \bigopen{1 - \frac{1}{n}}^{\frac{n-1}{2}} \geq 1 - e^{-\frac{n-1}{2n}} \geq 1 - e^{-\frac{1}{4}}.
\end{align*}

Substituting this inequality into \cref{eq:large-idhesslb-x2} gives
\begin{align*}
    x^K_n \geq \frac{\bigopen{1 - e^{-1}}\bigopen{1 - e^{-\frac{1}{4}}}G}{4L'}.
\end{align*}

Since $K \geq \kappa$, we have $\frac{1}{L'} = \frac{2}{L} \geq \frac{2}{\mu K}$. 
Therefore, the final iterate $x_n^K$ can be bounded as:
\begin{align*}
    x^K_n \geq \frac{\bigopen{1 - e^{-1}}\bigopen{1 - e^{-\frac{1}{4}}}G}{2\mu K}.
\end{align*}

Consequently, the function optimality gap satisfies
\begin{align*}
    F_2(x_n^K) - F_2(x^*) = \frac{L'}{2} (x_n^K)^2 \gtrsim \frac{L G^2}{\mu^2 K^2}.
\end{align*}

\subsubsection{Construction of $F_3$}
\label{subsubsec:large-lb-idhess-f3}
Let $F_3(x) = \frac{L}{2} x^2$ with component functions $f_i(x) = F_3(x)$ for all $i \in [n]$.
It is clear that $F_1$ satisfies \Cref{ass:common}, \Cref{ass:grad-generalized} with $G=P=0$ and its component functions share an identical Hessian.
Also, we note that $x^* = 0$ and $F_3(x^*) = 0$.

For all $\eta \in \left[\frac{2}{L}, \infty\right)$, the final iterate is given by
\begin{align*}
    x_n^K = \bigopen{1 - \eta L}^{nK} x_0.
\end{align*}

In this regime, the step size is excessively large, resulting in
\begin{align*}
    1 - \eta L \le 1 - \frac{2}{L} \cdot L \le -1,
\end{align*}
which implies $\bigabs{(1 - \eta L)^{nK}} \ge 1$.
Thus, the iterate does not converge and satisfies $\abs{x^K_n} \geq \abs{x_0}$. 

By setting the initialization $x_0 = \frac{G}{\mu K}$, we have
\begin{align*}
    F_3(x_n^K) - F_3(x^*) = \frac{L}{2} (x_n^K)^2 \geq \frac{L}{2}(x_0)^2 \gtrsim \frac{L G^2}{\mu^2 K^2}.
\end{align*}

\subsection{Proof of \Cref{thm:large-lb-concave}}
\label{subsec:large-lb-concave}
\thmlargelbconcave*

\begin{proof}
Similar to the approach in \Cref{thm:small-lb-idhess}, we divide the range of step sizes. 
However, unlike the previous theorems where the range is divided into three regimes, we divide the range into four regimes in this case.
For each regime, we construct the overall functions $F_1, F_2, F_3$, and $F_4$, along with their respective component functions and an initial point. 
Finally, we aggregate these functions across different dimensions to derive the stated lower bound. 

Each function is 1-dimensional, and carefully designed to satisfy the following properties:
\begin{itemize}
    \item (Small step size regime) There exists an initial point $x_0 = \text{poly}(\mu, L, n, K, G)$ such that for any choice of $\eta \in \bigopen{0, \frac{1}{\mu n K}}$, the final iterate $x^K_n$ obtained by running \Cref{alg:igd} satisfies $F_1(x^K_n) - F_1(x^*) \gtrsim \frac{L^2G^2}{\mu^3K^2}$,
    \item (Moderate step size regime 1) There exists an initial point $y_0 = \text{poly}(\mu, L, n, K, G)$ such that for any choice of $\eta \in \left[\frac{1}{\mu n K}, \frac{1}{nL}\right)$, the final iterate $y^K_n$ obtained by running \Cref{alg:igd} satisfies $F_2(y^K_n) - F_2(y^*) \gtrsim \frac{L^2G^2}{\mu^3K^2}$.
    \item (Moderate step size regime 2) There exists an initial point $z_0 = \text{poly}(\mu, L, n, K, G)$ such that for any choice of $\eta \in \left[\frac{1}{nL}, \frac{2}{L}\right)$, the final iterate $z^K_n$ obtained by running \Cref{alg:igd} satisfies $F_3(z^K_n) - F_3(z^*) \gtrsim \frac{L^2G^2}{\mu^3K^2}$.
    \item (Large step size regime) There exists an initial point $w_0 = \text{poly}(\mu, L, n, K, G)$ such that for any choice of $\eta \in \left[\frac{1}{L}, \infty\right)$, the final iterate $w^K_n$ obtained by running \Cref{alg:igd} satisfies $F_4(w^K_n) - F_4(w^*) \gtrsim \frac{L^2G^2}{\mu^3K^2}$.
\end{itemize}
Here, $x^*$, $y^*$, $z^*$, and $w^*$ denote the minimizers of $F_1$, $F_2$, $F_3$, and $F_4$, respectively.
All these functions are designed to satisfy \Cref{ass:common}. 
$F_1$ and $F_4$ satisfy \Cref{ass:grad-generalized} with $G=P=0$, $F_3$ satisfies with $P=0$, and $F_2$ satisfies with $P=\kappa$.
Detailed constructions for $F_1$ through $F_4$, as well as the verification of the assumptions and the stated properties are presented in \Cref{subsubsec:large-lb-concave-f1,subsubsec:large-lb-concave-f2,subsubsec:large-lb-concave-f3,subsubsec:large-lb-concave-f4}.

By following a similar approach to the proof of \Cref{thm:small-lb-idhess,thm:small-lb-sc}, we can conclude that the aggregated 4-dimensional function $F(\vx) := F_1(x) + F_2(y) + F_3(z) + F_4(w)$ satisfy the stated assumptions (additional scalar in $G$ can be absorbed by rescaling $G$ in each overall function). 
Also, since each dimension is independent, it is obvious that $\vx^* = (x^*, y^*, z^*, w^*)$ minimizes $F$.
Finally, by choosing the initial point as $\vx_0 = (x_0, y_0, z_0, w_0)$, the final iterate $\vx^K_n = (x^K_n, y^K_n, z^K_n, w^K_n)$ obtained by running \Cref{alg:igd} on $F$ satisfies
\begin{align*}
    F(\vx^K_n) - F(\vx^*) \gtrsim \frac{L^2G^2}{\mu^3K^2},
\end{align*}
regardless of the choice of $\eta > 0$.

This concludes the proof of \Cref{thm:large-lb-concave}. 
\end{proof}

In the following subsections, we present the specific construction of $F_1$, $F_2$, $F_3$, and $F_4$, and demonstrate that each satisfies the stated lower bound within its corresponding step size regime.
For simplicity of notation, we omit the index of the overall function when referring to its component functions, e.g., we write $f_i(x)$ instead of $f_{1i}(x)$. 
Moreover, we use the common variable notation $x$ while constructing functions for each dimension, though we use different variables in the ``dimension-aggregation'' step.

\subsubsection{Construction of $F_1$}
\label{subsubsec:large-lb-concave-f1}
Let $F_1(x) = \frac{\mu}{2}x^2$ with component functions $f_i(x) = F_1(x)$ for all $i \in [n]$.
It is clear that $F_1$ satisfies \Cref{ass:common} and \Cref{ass:grad-generalized} with $G = P = 0$.
Also, we note that $x^* = 0$ and $F_1(x^*) = 0$.

Let the initialization be $x_0 = \frac{LG}{\mu^2K}$.
For all $\eta \in \bigopen{0, \frac{1}{\mu n K}}$, the final iterate is given by
\begin{align*}
    x_n^K = (1 - \eta \mu)^{nK} x_0 \geq \bigopen{1 - \frac{1}{nK}}^{nK} x_0 \geq \frac{LG}{4\mu^2K},
\end{align*}
where the last inequality uses the fact that $(1 - \frac{1}{m})^m \geq \frac{1}{4}$ for all $m \geq 2$.

Thus, we have
\begin{align*}
    F_1(x_n^K) - F_1(x^*) = \frac{\mu}{2} (x_n^K)^2 \gtrsim \frac{L^2 G^2}{\mu^3 K^2}.
\end{align*}

\subsubsection{Construction of $F_2$}
\label{subsubsec:large-lb-concave-f2}

In this section, we focus on the case when $n$ is a multiple of $4$.
Otherwise, we set $4\floor{\frac{n}{4}}$ components satisfying the argument, and introduce at most three zero component functions.
This adjustment does not affect the final result, but only modifies the parameters $\mu$ and $L$ by at most a constant factor.

Let $F_2(x) = \frac{\mu}{2}x^2$ with component functions
\begin{align*}
    f_i(x) = \begin{cases}
        Gx & \textrm{ if } \, 1 \leq i \leq n / 4, \\
        \frac{L}{2}x^2 & \textrm{ if } \, n/4 + 1 \leq i \leq n / 2, \\
        -Gx & \textrm{ if } \, n/2 + 1 \leq i \leq 3n / 4, \\
        -\frac{L - 4 \mu}{2}x^2 & \textrm{ if } \, 3n / 4 + 1 \leq i \leq n.
    \end{cases}
\end{align*}
For simplicity of the notation, let $a$ denote $L - 4\mu$. 
Since $\kappa \geq 4$, we have $0 \le a < L$.
Thus, each $f_i$ is $L$-smooth, ensuring that the construction satisfies \Cref{ass:common}.
The gradient difference between the component function $f_i$ and the overall function $F_2$ is bounded as
\begin{align*}
    \bignorm{\nabla f_i (x) - \nabla F_2(x)}
    \le \begin{cases}
         \bignorm{\mu x} + G & \textrm{ if } \, 1 \leq i \leq n / 4 \, \textrm{ or } \, n/2 + 1 \leq i \leq 3n / 4,\\
        \bignorm{(L - \mu) x} & \textrm{ if } \,n/4 + 1 \leq i \leq n / 2 \, \textrm{ or } \, 3n / 4 + 1 \leq i \leq n.
    \end{cases}
\end{align*}
Since $\nabla F_2(x) = \mu x$, it follows that $\bignorm{(L - \mu) x} < \kappa \bignorm{\nabla F_2(x)}$.
Therefore, the construction satisfies \Cref{ass:grad-generalized} with $P = \kappa$.
Additionally, we note that $x^* = 0$ and $F_2(x^*) = 0$.
Using these component functions, we first derive the closed-form expression for the iterates obtained by running \Cref{alg:igd}:
\begin{align*}
    x^k_{n/4} & = x^k_0 - \frac{\eta n G}{4},\\
    x^k_{n/2} & = \bigopen{1 - \eta L}^{\frac{n}{4}}x^k_{n/4} = \bigopen{1 - \eta L}^{\frac{n}{4}}x^k_0 - \bigopen{1 - \eta L}^{\frac{n}{4}}\frac{\eta n G}{4},\\
    x^k_{3n/4} & = x^k_{n/2} + \frac{\eta n G}{4} = \bigopen{1 - \eta L}^{\frac{n}{4}}x^k_0 + \bigopen{1 - \bigopen{1 - \eta L}^{\frac{n}{4}}}\frac{\eta n G}{4},\\
    x^k_n & = \bigopen{1 + \eta a}^{\frac{n}{4}}x^k_{3n/4} = \bigopen{1 + \eta a}^{\frac{n}{4}} \bigopen{1 - \eta L}^{\frac{n}{4}}x^k_0 + \bigopen{1 + \eta a}^{\frac{n}{4}} \bigopen{1 - (1 - \eta L)^{\frac{n}{4}}}\frac{\eta n G}{4}.
\end{align*}

Let $p := (1 - \eta L)^{\frac{n}{4}}$ and $q := (1 + \eta a)^{\frac{n}{4}}$.
Using these definitions, the epoch-wise recursion equation can be expressed as:
\begin{align*}
    x^{k+1}_0 = pqx^k_0 + q(1 - p) \frac{\eta n G}{4}.
\end{align*}

By unrolling the above equation over $k \in [K]$, we obtain the final iterate $x_n^K$:
\begin{align}
    x^K_n = (pq)^K x_0 + \frac{1 - (pq)^K}{1 - pq} \cdot q(1 - p) \frac{\eta n G}{4}.\label{eq:large-lb-concave-x}
\end{align}

We now state key inequalities regarding $p$ and $q$:
\begin{restatable}{lemma}{lempqinequality}
    \label{lemma:lempqinequality}
    Under the conditions $K \ge \kappa \ge n \ge 3$, the following inequalities hold for $\eta \in \left[\frac{1}{\mu n K}, \frac{1}{n L}\right)$:
    \vspace{-2mm}
    \begin{enumerate}
        \item $1 - p \geq \begin{cases}
            \frac{L}{8 \mu K} & \textrm{ if } \, \eta \in \left[\frac{1}{\mu n K}, \frac{\mu}{L^2}\right),\\
             \frac{n\mu}{8L} & \textrm{ if } \, \eta \in \left[\frac{\mu}{L^2}, \frac{1}{nL}\right).
        \end{cases}$
        \item $1 - (pq)^K \geq 1 - e^{-1}$.
        \item $\frac{1}{1 - pq} \geq \begin{cases}
            \frac{4}{5\eta n \mu} & \textrm{ if } \, \eta \in \left[\frac{1}{\mu n K}, \frac{\mu}{L^2}\right),\\ \frac{4}{5 \eta^2 n L^2} & \textrm{ if } \, \eta \in \left[\frac{\mu}{L^2}, \frac{1}{nL}\right). \end{cases}$
    \end{enumerate}
\end{restatable}

The proof of \Cref{lemma:lempqinequality} is presented in \Cref{appendix:large-lb-techlmm}.
Setting the initialization point $x_0 = 0$, \cref{eq:large-lb-concave-x} simplifies to
\begin{align}
    x_n^K 
    = \frac{1 - (pq)^K}{1 - pq} \cdot q(1 - p) \frac{\eta n G}{4} 
    \ge \frac{1 - e^{-1}}{1 - pq} \cdot 1 \cdot (1 - p) \cdot \frac{\eta n G}{4} 
    = \frac{1 - e^{-1}}{4} \cdot \frac{1 - p}{1 - pq} \cdot \eta n G. \label{eq:large-lb-concave-x2}
\end{align}

Now, we divide the range of step size into two regimes: $\left[\frac{1}{\mu n K}, \frac{\mu}{L^2}\right)$ and $\left[\frac{\mu}{L^2}, \frac{1}{nL}\right)$.

\textbf{Regime 1.} $\eta \in \left[\frac{1}{\mu n K}, \frac{\mu}{L^2}\right)$.

In this regime, we have $1 - p \ge \frac{L}{8 \mu K}$ and $\frac{1}{1 - pq} \ge \frac{4}{5 \eta n \mu}$.
Substituting these inequalities to \cref{eq:large-lb-concave-x2} results
\begin{align*}
    x_n^K \ge \frac{\bigopen{1-e^{-1}} L G}{40 \mu^2 K}.
\end{align*}

\textbf{Regime 2.} $\eta \in \left[\frac{\mu}{L^2}, \frac{1}{nL}\right)$.

In this regime, we have $1 - p \ge \frac{n \mu}{8L}$ and $\frac{1}{1 - pq} \ge \frac{4}{5 \eta^2 n L^2}$.
Substituting these inequalities to \cref{eq:large-lb-concave-x2} results
\begin{align*}
    x_n^K \ge \frac{1 - e^{-1}}{40} \cdot \frac{n \mu G}{\eta L^3} \ge \frac{1 - e^{-1}}{40} \cdot \frac{n^2 \mu G}{L^2}.
\end{align*}

Using the assumption $K \ge \frac{\kappa^3}{n^2}$, it follows that $n^2 \ge \frac{\kappa^3}{K}$, resulting
\begin{align*}
    x^K_n \geq \frac{\bigopen{1 - e^{-1}} L G}{40 \mu^2 K}.
\end{align*}
Combining the results for the two subdivided step size regimes, we have
\begin{align*}
    x^K_n \geq \frac{\bigopen{1 - e^{-1}} L G}{40 \mu^2 K}.
\end{align*}
for all $\eta \in \left[\frac{1}{\mu n K}, \frac{1}{n L}\right)$.

Finally, the function optimality gap is
\begin{align*}
    F_2(x^K_n) - F_2(x^*) = \frac{\mu}{2}\bigopen{x^K_n}^2 \gtrsim \frac{L^2G^2}{\mu^3 K^2}.
\end{align*}

\subsubsection{Construction of $F_3$}
\label{subsubsec:large-lb-concave-f3}

We focus on the case where $n$ is even.
If $n$ is odd, we introduce an additional zero component function.
This does not affect the final result but only modifies each parameter at most by a constant factor.

In this subsection, we let $L'$ denote $L/2$.
Let $F_3 = \frac{L'}{2}x^2$ with component functions
\begin{align*}
    f_i(x) = \begin{cases}
        \frac{L'}{2}x^2 + Gx & \textrm{ if } \, i \leq n/2,\\
        \frac{L'}{2}x^2 - Gx & \textrm{ otherwise}.
    \end{cases}
\end{align*}

It is clear that each $f_i$ is $L$-smooth.
Since $\kappa \ge 2$, we have $L' = \frac{L}{2} \ge \mu$.
Thus, $F_3$ is $\mu$-strongly convex, satisfying \Cref{ass:common}.
Also, we can easily verify that the construction satisfies \Cref{ass:grad-generalized} with $P = 0$.
We note that $x^*=0$ and $F_3(x^*) = 0$.

By \Cref{lem:quadratic-twotype-IGD-closed}, the final iterate obtained by running \Cref{alg:igd} is given by
\begin{align*}
    x^K_n = \bigopen{1 - \eta L'}^{nK}x_0 + \frac{G}{L'} \cdot \frac{1 - (1-\eta L')^{\frac{n}{2}}}{1 + (1-\eta L')^{\frac{n}{2}}} \bigopen{1 - \bigopen{1 - \eta L'}^{nK}}.
\end{align*}

Recall that $\frac{1}{nL} = \frac{1}{2nL'}$ and $\frac{2}{L} = \frac{1}{L'}$.
Since $\eta \in \left[\frac{1}{2nL'}, \frac{1}{L'}\right)$, it follows that
\begin{align*}
    (1 - \eta L')^{\frac{n}{2}} \leq \bigopen{1 - \frac{1}{2n}}^{\frac{n}{2}} \leq e^{-\frac{1}{4}}, \text{ and } \, (1 - \eta L')^{nK} \le e^{-\frac{K}{4}}.
\end{align*}
   
Using these inequalities and setting the initialization as $x_0 = 0$, the final iterate $x_n^K$ is expressed as:
\begin{align*}
    x^K_n = \frac{G}{L'} \cdot \frac{1 - (1-\eta L')^{\frac{n}{2}}}{1 + (1-\eta L')^{\frac{n}{2}}} \bigopen{1 - \bigopen{1 - \eta L'}^{nK}} \geq \frac{G}{L'}\frac{1 - e^{-\frac{1}{4}}}{2}\bigopen{1 - e^{-\frac{K}{4}}} \geq \frac{G}{L'}\frac{\bigopen{1 - e^{-\frac{1}{4}}}^2}{2}.
\end{align*}

Finally, the function optimality gap becomes
\begin{align*}
    F_3(x^K_n) - F_3(x^*) = \frac{L'}{2}\bigopen{x^K_n}^2 \gtrsim \frac{G^2}{L} \geq \frac{L^2G^2}{\mu^3 K^2},
\end{align*}
where the last inequality holds since $K \geq \kappa^{3/2}$.

\subsubsection{Construction of $F_4$}
\label{subsubsec:large-lb-concave-f4}
Let $F_4(x) = \frac{L}{2} x^2$ with component functions $f_i(x) = F_4(x)$ for all $i \in [n]$.
It is clear that $F_4$ satisfies \Cref{ass:common}, \Cref{ass:grad-generalized} with $G=P = 0$.
Also, we note that $x^* = 0$ and $F_4(x^*) = 0$.

For all $\eta \in \left[\frac{2}{L}, \infty\right)$, the final iterate is given by
\begin{align*}
    x_n^K = \bigopen{1 - \eta L}^{nK} x_0.
\end{align*}

In this regime, the step size is excessively large, resulting in
\begin{align*}
    1 - \eta L \le 1 - \frac{2}{L} \cdot L \le -1,
\end{align*}
which implies $\bigabs{(1 - \eta L)^{nK}} \ge 1$.
Thus, the iterate does not converge and satisfies $\abs{x^K_n} \geq \abs{x_0}$. 

By setting the initialization $x_0 = \sqrt{\kappa} \frac{G}{\mu K}$, we have
\begin{align*}
    F_4(x_n^K) - F_4(x^*) = \frac{L}{2} (x_n^K)^2 \geq \frac{L}{2}(x_0)^2 \gtrsim \frac{L^2G^2}{\mu^3 K^2}.
\end{align*}

\subsection{Technical Lemmas}
\label{appendix:large-lb-techlmm}

\lempqinequality*

\begin{proof}
    Recall the definitions of $p$ and $q$:
    \begin{align*}
        &p = (1 - \eta L)^{\frac{n}{4}},\\
        &q = (1 + \eta a)^{\frac{n}{4}} = (1 + \eta (L - 4 \mu))^{\frac{n}{4}}.
    \end{align*}

    To prove the first inequality, we divide the range of step size into two regimes: $\left[\frac{1}{\mu n K}, \frac{\mu}{L^2}\right)$ and $\left[\frac{\mu}{L^2}, \frac{1}{nL}\right)$.
    Note that the first regime may be empty depending on the condition on $K$, but remains valid (i.e. $\frac{1}{\mu n K} \leq \frac{\mu}{L^2}$) under the condition $K \geq \kappa^2/n$ in the current theorem.
    
    \textbf{Regime 1.} $\eta \in \left[\frac{1}{\mu n K}, \frac{\mu}{L^2}\right)$.

    In this regime, we can bound $p$ as:
    \begin{align*}
        p = (1 - \eta L)^{\frac{n}{4}} 
        \leq \bigopen{1 - \frac{L}{\mu n K}}^{\frac{n}{4}} 
        \leq e^{-\frac{L}{4 \mu K}} 
        \leq 1 - \frac{L}{8 \mu K}.
    \end{align*}

    Here, the first step holds because $\eta \ge \frac{1}{\mu n K}$.
    In the final step, we utilize the inequalities $\frac{L}{4 \mu K} < 1$ and $e^{-u} \le 1 - \frac{1}{2} u$ for all $u \in [0, 1]$.
    Hence, we can obtain $1- p \ge \frac{L}{8 \mu K}$.

    \textbf{Regime 2.} $\eta \in \left[\frac{\mu}{L^2}, \frac{1}{nL}\right)$.
 
    In this regime, we can bound $p$ as:
    \begin{align*}
        p = (1 - \eta L)^{\frac{n}{4}} 
        \leq \bigopen{1 - \frac{\mu}{L}}^{\frac{n}{4}} 
        \leq e^{-\frac{n \mu}{4 L}} 
        \leq 1 - \frac{n \mu}{8 L}.
    \end{align*}
    
    Here, the first step holds because $\eta \ge \frac{\mu}{L^2}$.
    At the final step, we utilize the inequalities $\frac{n \mu}{4L} < 1$ and $e^{-u} \le 1 - \frac{1}{2} u$ for all $u \in [0, 1]$.
    Hence, we can obtain $1- p \ge \frac{n \mu}{8 L}$.

    To bound $1 - (pq)^K$, we first establish bounds for $pq$:
    \begin{align*}
            pq 
            = (1 - \eta L)^{\frac{n}{4}}(1 + \eta a)^{\frac{n}{4}}
            \leq e^{-\frac{\eta n L}{4}} \cdot e^{\frac{\eta n a}{4}} 
            = e^{-\frac{\eta n (L-a)}{4}} 
            = e^{-\eta n \mu} \leq e^{-\frac{1}{K}},
    \end{align*}
    where we apply $\eta \ge \frac{1}{\mu n K}$ at the last step.
    Therefore, we can obtain
    \begin{align*}
        1- (pq)^K \ge 1 - \bigopen{e^{-\frac{1}{K}}}^K = 1 - e^{-1}.
    \end{align*}

    The last inequality requires more careful analysis.
    We further refine the bounds for $pq$.
    Using $a = L - 4\mu < L$, it follows that
    \begin{align*}
        1 - \eta (L-a) - \eta^2 a L \geq 1 - 4 \eta \mu  - \eta^2 L^2 \geq 1 - \frac{4}{n \kappa} - \frac{1}{n^2} \geq 0,
    \end{align*}
    where the second step is due to $\eta \leq \frac{1}{nL}$ and last step holds by the condition $\kappa \geq n \geq 3$. Hence,
    \begin{align}
        pq 
        = (1 - \eta L)^{\frac{n}{4}} (1 + \eta a)^{\frac{n}{4}}
        = (1 - \eta (L-a) - \eta^2 a L)^{\frac{n}{4}}
        \ge (1 - 4 \eta \mu - \eta^2 L^2)^{\frac{n}{4}}. \label{eq:pqbound}
    \end{align}
    
    We again divide the range of step size into two regimes: $\left[\frac{1}{\mu n K}, \frac{\mu}{L^2}\right)$ and $\left[\frac{\mu}{L^2}, \frac{1}{nL}\right)$.

    \textbf{Regime 1.} $\eta \in \left[\frac{1}{\mu n K}, \frac{\mu}{L^2}\right)$.

    In this regime, we have $\eta^2 L^2 \le \eta \mu$.
    Hence, \cref{eq:pqbound} becomes
    \begin{align*}
        pq \ge (1 - 4 \eta \mu - \eta^2 L^2)^{\frac{n}{4}} \ge (1 - 5 \eta \mu)^{\frac{n}{4}} \ge 1 - \frac{5}{4} \eta n \mu,
    \end{align*}
    since $5 \eta \mu \le \frac{5 \mu^2}{L^2} < 1$ (assuming $\kappa \ge n \ge 3$).
    Therefore, we obtain the following inequality:
    \begin{align*}
        \frac{1}{1 - pq} \ge \frac{4}{5 \eta n \mu}.
    \end{align*}

    \textbf{Regime 2.} $\eta \in \left[\frac{\mu}{L^2}, \frac{1}{nL}\right)$.

    In this regime, we have $\eta^2 L^2 \ge \eta \mu$.
    Hence, \cref{eq:pqbound} becomes
    \begin{align*}
        pq \ge (1 - 4 \eta \mu - \eta^2 L^2)^{\frac{n}{4}} \ge (1 - 5 \eta^2 L^2)^{\frac{n}{4}} \ge 1 - \frac{5}{4} \eta^2 n L^2,
    \end{align*}
    since $5 \eta^2 L^2 < \frac{5}{n^2} < 1$ (assuming $n \ge 3$).
    Therefore, we obtain the following inequality:
    \begin{align*}
        \frac{1}{1 - pq} \ge \frac{4}{5 \eta^2 n L^2}.
    \end{align*}

    This concludes the proof of the lemma.
\end{proof}
\newpage
\section{Proofs for Large Epoch Upper Bounds}
\label{appendix:large-epoch-ub}

In this section, we provide detailed proof for \Cref{thm:large-ub-generalizedgrad}.

\subsection{Proof of \Cref{thm:large-ub-generalizedgrad}}
\label{subsec:large-ub-generalizedgrad}
\thmlargeubgeneralizedgrad*
\begin{proof}
We begin by noting the specific epoch condition used to prove the statement:
\begin{align*}
    K \geq 8\kappa \max\bigset{1, P} \max \bigset{\log \bigopen{\frac{(F(\vx_0) - F(\vx^*))\mu^3K^2}{L^2G^2}}, 1}.
\end{align*}

Given this epoch condition and the choice of step size $\eta$ specified in the theorem statement, we have $\eta n L \leq \frac{1}{4}\min\bigset{1,\frac{1}{P}}$, which will be repeatedly utilized throughout the proof.

Consider the following epoch-wise recursive inequality for the objective function:
\begin{align}
    F \bigopen{\vx_0^{k+1}}
    & \le F \bigopen{\vx_0^k} + \inner{\nabla F \bigopen{\vx_0^k}, \vx_0^{k+1} - \vx_0^k} + \frac{L}{2} \bignorm{\vx_0^{k+1} - \vx_0^k}^2 \nonumber \\
    & = F \bigopen{\vx_0^k} - \eta n \inner{\nabla F \bigopen{\vx_0^k}, \frac{1}{n} \sum_{i=1}^n \nabla f_{\sigma_k \bigopen{i}} \bigopen{\vx_{i-1}^k}} + \frac{\eta^2 n^2 L}{2} \bignorm{\frac{1}{n} \sum_{i=1}^n \nabla f_{\sigma_k \bigopen{i}} \bigopen{\vx_{i-1}^k}}^2 \nonumber\\
    & = F \bigopen{\vx_0^k} - \frac{\eta n}{2} \bignorm{\nabla F \bigopen{\vx_0^k}}^2 - \frac{\eta n}{2} \bignorm{\frac{1}{n} \sum_{i=1}^n \nabla f_{\sigma_k \bigopen{i}} \bigopen{\vx_{i-1}^k}}^2 \nonumber \\ 
    & \quad + \frac{\eta n}{2} \bignorm{\nabla F \bigopen{\vx_0^k} \nonumber - \frac{1}{n} \sum_{i=1}^n \nabla f_{\sigma_k \bigopen{i}} \bigopen{\vx_{i-1}^k}}^2 + \frac{\eta^2 n^2 L}{2} \bignorm{\frac{1}{n} \sum_{i=1}^n \nabla f_{\sigma_k \bigopen{i}} \bigopen{\vx_{i-1}^k}}^2 \nonumber \\
    & \overset{(a)}{\le} F \bigopen{\vx_0^k} - \frac{\eta n}{2} \bignorm{\nabla F \bigopen{\vx_0^k}}^2 + \frac{\eta n}{2} \bignorm{\nabla F \bigopen{\vx_0^k} - \frac{1}{n} \sum_{i=1}^n \nabla f_{\sigma_k \bigopen{i}} \bigopen{\vx_{i-1}^k}}^2 \nonumber \\
    & \overset{(b)}{\le} F \bigopen{\vx_0^k} - \frac{\eta n}{2} \bignorm{\nabla F \bigopen{\vx_0^k}}^2 + \frac{\eta L^2}{2} \sum_{i=1}^n \bignorm{\vx^k_0 - \vx^k_{i-1}}^2, \label{eq:largeubgenepoch}
\end{align}
where (a) holds due to $\eta n L \le \frac{1}{4} < 1$ and (b) follows from the inequality:
\begin{align*}
    \bignorm{\nabla F \bigopen{\vx_0^k} - \frac{1}{n} \sum_{i=1}^n \nabla f_{\sigma_k \bigopen{i}} \bigopen{\vx_{i-1}^k}}^2 & = \bignorm{\frac{1}{n} \sum_{i=1}^n \bigopen{\nabla f_{\sigma_k \bigopen{i}} \bigopen{\vx_0^k} - \nabla f_{\sigma_k \bigopen{i}} \bigopen{\vx_{i-1}^k}}}^2 \nonumber\\
    & \le \frac{1}{n} \sum_{i=1}^n \bignorm{\nabla f_{\sigma_k \bigopen{i}} \bigopen{\vx_0^k} - \nabla f_{\sigma_k \bigopen{i}} \bigopen{\vx_{i-1}^k}}^2 \nonumber\\
    & \le \frac{L^2}{n} \sum_{i=1}^n \bignorm{\vx_0^k - \vx_{i-1}^k}^2.
\end{align*}

Next, we need to derive an upper bound for $\bignorm{\vx_0^k - \vx_{i-1}^k}^2$. 
For $t \in [n]$, we have 
\begin{align}
    \bignorm{\vx_0^k - \vx_{t}^k}^2
    &= \eta^2 \bignorm{\sum_{i=1}^t \nabla f_{\sigma_k \bigopen{i}} \bigopen{\vx_{i-1}^k}}^2 \nonumber \\
    &\le 3\eta^2 \bignorm{\sum_{i=1}^t \bigopen{\nabla f_{\sigma_k \bigopen{i}} \bigopen{\vx_{i-1}^k} - \nabla f_{\sigma_k \bigopen{i}} \bigopen{\vx_0^k}}}^2 + 3\eta^2
    \bignorm{\sum_{i=1}^t \bigopen{\nabla f_{\sigma_k \bigopen{i}} \bigopen{\vx_0^k} - \nabla F \bigopen{\vx_0^k}}}^2 \nonumber \\
    &\quad+ 3\eta^2 
    \bignorm{\sum_{i=1}^t \nabla F \bigopen{\vx_0^k}}^2 \nonumber \\
    & \overset{(a)}{\leq} 3 \eta^2 t \sum_{i=1}^t L^2 \bignorm{\vx_0^k - \vx_{i-1}^k}^2  + 6 \eta^2 t^2 \bigopen{G^2 + P^2 \bignorm{\nabla F(\vx_0^k)}^2} + 3 \eta^2 t^2 \bignorm{\nabla F \bigopen{\vx_0^k}}^2 \nonumber \\
    & = 3 \eta^2 t L^2 \sum_{i=1}^t \bignorm{\vx_0^k - \vx_{i-1}^k}^2  + 6 \eta^2 t^2 G^2 + 3 \eta^2 t^2(1 + 2P^2) \bignorm{\nabla F \bigopen{\vx_0^k}}^2. \label{eq:y_diff}
\end{align}

Here, (a) is derived by applying \Cref{ass:grad-generalized} through the following sequence of inequalities:
\begin{align*}
    \bignorm{\sum_{i=1}^t \bigopen{\nabla f_{\sigma_k \bigopen{i}} \bigopen{\vx} - \nabla F \bigopen{\vx}}}^2 
    & \leq t \sum_{i=1}^t \bignorm{\nabla f_{\sigma_k \bigopen{i}} \bigopen{\vx} - \nabla F \bigopen{\vx}}^2\\
    & \leq t \sum_{i=1}^t \bigopen{G + P\bignorm{\nabla F(\vx)}}^2\\
    & \leq t \sum_{i=1}^t \bigopen{2G^2 + 2P^2 \bignorm{\nabla F(\vx)}^2}\\
    & \leq 2t^2\bigopen{G^2 + P^2 \bignorm{\nabla F(\vx)}^2}.
\end{align*}

Summing \cref{eq:y_diff} over $t = 1, \dots, n-1$, we have
\begin{align*}
    \sum_{i=1}^n \bignorm{\vx_0^k - \vx_{i-1}^k}^2
    & \leq 3 \eta^2 \frac{(n - 1) n}{2} L^2 \sum_{i=1}^n \bignorm{\vx_0^k - \vx_{i-1}^k}^2 + 6 \eta^2 \frac{(n - 1)n(2n - 1)}{6} G^2 \\
    & \phantom{\le} + 3 \eta^2 \frac{(n - 1) n (2n -1)}{6} (1 + 2P^2)\bignorm{\nabla F \bigopen{\vx_0^k}}^2\\
    & \le 3 \eta^2 n^2 L^2 \sum_{i=1}^n \bignorm{\vx_0^k - \vx_{i-1}^k}^2 + 2\eta^2 n^3 G^2 + \eta^2 n^3(1 + 2 P^2) \bignorm{\nabla F \bigopen{\vx_0^k}}^2.
\end{align*}

Given $\eta n L \le \frac{1}{4}$, it follows that $3 \eta^2 n^2 L^2 \leq \frac{1}{2}$ and the above inequality simplifies to
\begin{align}
    \sum_{i=1}^n \bignorm{\vx_0^k - \vx_{i-1}^k}^2 \leq 4 \eta^2 n^3G^2 + 2 \eta^2 n^3 (1 + 2 P^2) \bignorm{\nabla F \bigopen{\vx_0^k}}^2. \label{eq:y_diff_sum}
\end{align}

Substituting \cref{eq:y_diff_sum} to \cref{eq:largeubgenepoch} results in
\begin{align*}
    F \bigopen{\vx_0^{k+1}}
    &\leq F \bigopen{\vx_0^k} - \frac{\eta n}{2} \bignorm{\nabla F \bigopen{\vx_0^k}}^2 + \frac{\eta L^2}{2} \sum_{i=1}^n\bignorm{\vx^k_0 - \vx^k_{i-1}}^2\\
    & \leq F(\vx^k_0) - \frac{\eta n}{2} \bignorm{\nabla F(\vx^k_0)}^2 + \frac{\eta L^2}{2} \bigopen{4 \eta^2n^3G^2 + 2\eta^2 n^3(1 + 2P^2)\bignorm{\nabla F(\vx^k_0)}^2}\\
    & \leq F(\vx^k_0) - \frac{\eta n}{2}\bigopen{1 - 2\eta^2n^2L^2\bigopen{1 + 2P^2}} \bignorm{\nabla F(\vx^k_0)}^2 + 2 \eta^3n^3L^2G^2\\
    & \overset{(a)}{\leq} F(\vx^k_0) - \frac{\eta n}{4}\bignorm{\nabla F(\vx^k_0)}^2 + 2 \eta^3n^3L^2G^2\\
    & \overset{(b)}{\leq} F(\vx^k_0) - \frac{\eta n \mu}{2}\bigopen{F(\vx^k_0) - F(\vx^*)} + 2 \eta^3n^3L^2G^2,
\end{align*}
where at $(a)$, we use $\eta n L \leq \frac{1}{4}$ and $\eta n L \leq \frac{1}{4P}$, ensuring $\eta^2n^2L^2\bigopen{1 + 2P^2} \leq \frac{1}{16} + \frac{1}{8} \leq \frac{1}{4}$, and at $(b)$, we utilize the assumption that $F$ satisfies $\mu$-strongly convexity.
We note that $(b)$ is the only step where $\mu$-strong convexity of $F$ is utilized, and it also holds under the weaker assumption that $F$ satisfies the Polyak-\L ojasiewicz condition.
Thus, \Cref{thm:large-ub-generalizedgrad} remains valid when $F$ satisfies the P\L{} condition.

Rearranging this inequality leads to
\begin{align*}
    F \bigopen{\vx_0^{k+1}} - F(\vx^*) \le \bigopen{1 - \frac{\eta n \mu}{2}} \bigopen{F \bigopen{\vx_0^k} - F^*} + 2 \eta^3n^3L^2G^2,
\end{align*}
and we can obtain
\begin{align*}
    F \bigopen{\vx_n^K} - F(\vx^*)
    & \le \bigopen{1 - \frac{\eta n \mu}{2}}^K \bigopen{F \bigopen{\vx_0} - F(\vx^*)} + 2\eta^3n^3L^2G^2 \cdot \sum_{k=1}^K \bigopen{1 - \frac{\eta n \mu}{2}}^{k-1}  \\
    & \leq \bigopen{1 - \frac{\eta n \mu}{2}}^K \bigopen{F \bigopen{\vx_0} - F(\vx^*)} + 2\eta^3n^3L^2G^2 \cdot \frac{2}{\eta n \mu} \\
    & \leq e^{- \frac{\eta \mu n K}{2}} \bigopen{F \bigopen{\vx_0} - F^*} + \frac{4 \eta^2 n^2 L^2G^2}{\mu}.
\end{align*}

We now substitute $\eta = \frac{2}{\mu n K} \max \bigset{\log\bigopen{\frac{(F(\vx_0) - F(\vx^*))\mu^3K^2}{L^2G^2}}, 1}$.
For the case when $F(\vx_0) - F(\vx^*)$ is sufficiently large, the above inequality becomes
\begin{align*}
    F \bigopen{\vx_n^K} - F(\vx^*)
    \le \frac{L^2 G^2}{\mu^3 K^2} + \frac{16L^2 G^2}{\mu^3 K^2} \cdot \log^2 \bigopen{\frac{(F(\vx_0) - F(\vx^*))\mu^3K^2}{L^2G^2}} \lesssim \frac{L^2G^2}{\mu^3 K^2}.
\end{align*}

For the case when $F(\vx_0) - F(\vx^*)$ is small so that $1$ is chosen after the max operation, the above inequality then becomes
\begin{align*}
    F \bigopen{\vx_n^K} - F(\vx^*)
    \le \frac{1}{e} \cdot e \cdot \frac{L^2 G^2}{\mu^3 K^2} + \frac{16L^2 G^2}{\mu^3 K^2} \lesssim \frac{L^2G^2}{\mu^3 K^2},
\end{align*}
where we utilize $F(\vx_0) - F(\vx^*) \le e \cdot \frac{L^2 G^2}{\mu^3 K^2}$.
This ends the proof of \Cref{thm:large-ub-generalizedgrad}.
\end{proof}
\newpage
\section{Lemmas}
\label{appendix:lemmas}

\begin{lemma}
\label{lem:quadratic-twotype-IGD-closed}
Let $n$ be an even number. 
Define $F(x) = \frac{a}{2}x^2$ with component functions 
\begin{align*}
    f_i(x) = \begin{cases}
        \frac{a}{2}x^2 + Gx & \textrm{ if } i \leq n/2,\\
        \frac{a}{2}x^2 - Gx & \textrm{ otherwise}.
    \end{cases}
\end{align*}
Then, the final iterate $x_n^K$ obtained by running \Cref{alg:igd} for $K$ epochs with a step size $\eta$ starting from the initialization point $x_0$, satisfies:
\begin{align*}
    x_n^K = \bigopen{1 - \eta a}^{nK}x_0 + \frac{G}{a} \cdot \frac{1 - (1-\eta a)^{\frac{n}{2}}}{1 + (1-\eta a)^{\frac{n}{2}}} \bigopen{1 - \bigopen{1 - \eta a}^{nK}}.
\end{align*}
\end{lemma}

\begin{proof}
For $i \leq \frac{n}{2}$, the update rule is given as:
\begin{align*}
    x^k_i = x^k_{i-1} - \eta (ax^k_{i-1} + G) = (1 - \eta a)x^k_{i-1} - \eta G.
\end{align*}

For $i \geq \frac{n}{2} + 1$, the update rule is given as:
\begin{align*}
    x^k_i = x^k_{i-1} - \eta \bigopen{ax^k_{i-1} - G} = (1 - \eta a)x_{i-1}^k + \eta G.
\end{align*}

By sequentially applying the component functions, we derive the following epoch-wise recursion equation: 
\begin{align}
    x^{k+1}_0 &= (1-\eta a)^n x^k_0 - \eta G\sum_{i=1}^{\frac{n}{2}} (1-\eta a)^{n-i} + \eta G\sum_{i=\frac{n}{2} + 1}^{n} (1-\eta a)^{n-i} \nonumber \\
    & = (1-\eta a)^n x^k_0 + \frac{G}{a}{\bigopen{1 - \bigopen{1 - \eta a}^{\frac{n}{2}}}^2}, \label{eq:lmmevenepochwise}
\end{align}
where the last equality follows from the following observation:
\begin{align*}
     - \eta G\sum_{i=1}^{\frac{n}{2}} (1-\eta a)^{n-i} + \eta G\sum_{i=\frac{n}{2} + 1}^{n} (1-\eta a)^{n-i} 
     & = \eta G \bigopen{1 - \bigopen{1 - \eta a}^{\frac{n}{2}}}\sum_{i=\frac{n}{2}+1}^n \bigopen{1 - \eta a}^{n-i}\\
     & = \eta G \bigopen{1 - \bigopen{1 - \eta a}^{\frac{n}{2}}}\frac{1 - \bigopen{1 - \eta a}^{\frac{n}{2}}}{\eta a}\\
     & = \frac{G}{a} \bigopen{1 - \bigopen{1 - \eta a}^{\frac{n}{2}}}^2.
\end{align*}

By unrolling \cref{eq:lmmevenepochwise} over $k \in [K]$, we obtain the equation for the final iterate $x_n^K$:
\begin{align*}
    x_n^K & = (1 - \eta a)^{nK}x_0 + \frac{G}{a} \cdot \frac{1 - \bigopen{1 - \eta a}^{nK}}{1 - \bigopen{1 - \eta a}^{n}}\bigopen{1 - \bigopen{1 - \eta a}^{\frac{n}{2}}}^2\\
    & = (1 - \eta a)^{nK}x_0 + \frac{G}{a} \cdot \frac{1 - \bigopen{1 - \eta a}^{\frac{n}{2}}}{1 + \bigopen{1 - \eta a}^{\frac{n}{2}}}\bigopen{1 - \bigopen{1 - \eta a}^{nK}}.
\end{align*}
\end{proof}

\begin{lemma}
    \label{lem:quadratic-threetype-IGD-closed}
Let $n$ be an odd number.
Define $F(x) = \frac{a}{2}x^2$ with component functions
\begin{align*}
    f_i(x) = \begin{cases}
        \frac{a}{2}x^2 & \textrm{ if } i=1,\\
        \frac{a}{2}x^2 + Gx & \textrm{ if } 2 \leq i \leq (n+1)/2,\\
        \frac{a}{2}x^2 - Gx & \textrm{ if } (n+3)/2 \leq i \leq n.
    \end{cases}
\end{align*}
Then, the final iterate $x_n^K$ obtained by running \Cref{alg:igd} for $K$ epochs with a step size $\eta$ starting from the initialization point $x_0$, satisfies:
\begin{align*}
    x^K_n & = (1 - \eta a)^{nK} x_0 + \frac{G}{a} \cdot \frac{1 - (1-\eta a)^{nK}}{1 - (1-\eta a)^n}\bigopen{1 - \bigopen{1 - \eta a}^{\frac{n-1}{2}}}^2.
\end{align*}
\end{lemma}

\begin{proof}
Compared to \Cref{lem:quadratic-twotype-IGD-closed}, we have an additional component function $f_1 (x) = \frac{a}{2}x^2$ at the beginning of each epoch.
For this function, the update for $x_1^k$ is given as:
\begin{align*}
    x^k_1 = x^k_0 - \eta a x^k_0 = (1 - \eta a)x^k_0. 
\end{align*}
Thus, the epoch-wise equation in \cref{eq:lmmevenepochwise} of \Cref{lem:quadratic-twotype-IGD-closed} is modified as follows:
\begin{align*}
    x_0^{k+1} = (1-\eta a)^{n-1} x^k_1 + \frac{G}{a}{\bigopen{1 - \bigopen{1 - \eta a}^{\frac{n-1}{2}}}^2} = (1-\eta a)^n x^k_0 + \frac{G}{a}{\bigopen{1 - \bigopen{1 - \eta a}^{\frac{n-1}{2}}}^2}.
\end{align*}

By unrolling the above inequality over $k \in [K]$, we obtain the equation for the final iterate $x_n^k$:
\begin{align*}
    x_n^K & = (1 - \eta a)^{nK} x_0 + \frac{G}{a} \cdot \frac{1 - (1-\eta a)^{nK}}{1 - (1-\eta a)^n}\bigopen{1 - \bigopen{1 - \eta a}^{\frac{n-1}{2}}}^2.
\end{align*}
\end{proof}

\begin{lemma}
\label{lem:quadratic-concave-IGD-closed}
Let $n$ be an even number. 
Define $F(x) = \frac{a}{8}x^2$ with component functions
\begin{align*}
    f_i(x) = \begin{cases}
        \frac{a}{2}x^2 + Gx & \textrm{ if } i \leq n/2,\\
        -\frac{a}{4}x^2 - Gx &\textrm{ otherwise}.
    \end{cases}
\end{align*}
Consider applying \Cref{alg:igd} for a single epoch, starting from $x_0^k$.
The updated iterate $x_0^{k+1}$ satisfies the following equation:
\begin{align*}
    x_0^{k+1} = \bigopen{1 + \frac{\eta a}{2}}^{\frac{n}{2}} \bigopen{1 - \eta a}^{\frac{n}{2}} x^k_0 + \frac{G}{a}\bigopen{\bigopen{1 + \frac{\eta a}{2}}^{\frac{n}{2}} \bigopen{1 + (1 - \eta a)^\frac{n}{2}} - 2}.
\end{align*}
\end{lemma}

\begin{proof}
For $i \leq \frac{n}{2}$, the update rule is given as:
\begin{align*}
    x^k_i = x^k_{i-1} - \eta (ax^k_{i-1} + G) = (1 - \eta a)x^k_{i-1} - \eta G.
\end{align*}

By sequentially applying the first half of the component functions, we obtain
\begin{align}
    x^k_{\frac{n}{2}} & = (1 - \eta a)^{\frac{n}{2}}x^k_0 - \eta G \sum_{i=0}^{\frac{n}{2} - 1} (1 - \eta a)^i = (1 - \eta a)^{\frac{n}{2}}x^k_0 - \eta G \cdot \frac{1 - (1 - \eta a)^{\frac{n}{2}}}{\eta a} \nonumber \\
    & = (1 - \eta a)^{\frac{n}{2}}x^k_0 - \frac{G}{a}\bigopen{1 - (1 - \eta a)^{\frac{n}{2}}}. \label{eq:lmmconcavefirsthalf}
\end{align}

For $i \geq \frac{n}{2} + 1$, the update rule is given as:
\begin{align*}
    x^k_i = x^k_{i-1} - \eta \bigopen{-\frac{a}{2}x^k_{i-1} - G} = \bigopen{1 + \frac{\eta a}{2}}x^k_{i-1} + \eta G.
\end{align*}

Substituting the result from \cref{eq:lmmconcavefirsthalf} into the update rule for the second half of the component functions, we obtain $x_0^{k+1}$ (equivalently, $x_n^k$) as follows:
\begin{align*}
    x_0^{k+1} & = \bigopen{1 + \frac{\eta a}{2}}^{\frac{n}{2}}x^k_{\frac{n}{2}} + \eta G \sum_{i=0}^{\frac{n}{2} - 1} \bigopen{1 + \frac{\eta a}{2}}^i = \bigopen{1 + \frac{\eta a}{2}}^{\frac{n}{2}}x^k_\frac{n}{2} + \eta G \cdot \frac{2\bigopen{\bigopen{1 + \frac{\eta a}{2}}^{\frac{n}{2}} - 1}}{\eta a}\\
    & = \bigopen{1 + \frac{\eta a}{2}}^{\frac{n}{2}}x^k_{\frac{n}{2}} + \frac{2G}{a}\bigopen{\bigopen{1 + \frac{\eta a}{2}}^{\frac{n}{2}} - 1}\\
    & = \bigopen{1 + \frac{\eta a}{2}}^{\frac{n}{2}} \bigopen{(1 - \eta a)^{\frac{n}{2}}x^k_0 - \frac{G}{a}\bigopen{1 - (1 - \eta a)^{\frac{n}{2}}}} + \frac{2G}{a}\bigopen{\bigopen{1 + \frac{\eta a}{2}}^{\frac{n}{2}} - 1}\\
    & = \bigopen{1 + \frac{\eta a}{2}}^{\frac{n}{2}}\bigopen{1 - \eta a}^{\frac{n}{2}} x^k_0 + \frac{G}{a}\bigopen{\bigopen{1 + \frac{\eta a}{2}}^{\frac{n}{2}} \bigopen{1 + (1 - \eta a)^\frac{n}{2}} - 2}.
\end{align*}
\end{proof}
\newpage
\section{Experiments}
\label{appendix:experiments}

In this section, we validate the lower bound convergence rates for the functions used in the lower bound construction of \Cref{thm:small-lb-sc,thm:small-lb-concave}.
We compare the performance of four permutation-based SGD methods: \igd, \randr, Herding at Optimum, and with-replacement SGD. 
Here, Herding at Optimum refers to the instance of \Cref{alg:permutation} using the permutation suggested from \Cref{thm:herding-at-opt} satisfying \cref{eq:herdingproperty}. 
As mentioned in \Cref{subsec:small-optimal}, this permutation is generally unknown without prior knowledge of $\vx^*$. 
However, for the specific functions used in the lower bound construction, we can explicitly determine a permutation $\sigma$ that satisfies \cref{eq:herdingproperty}.
Thus, the plot for Herding represents the convergence rate achieved by a well-chosen permutation in permutation-based SGD.

Additionally, we conduct experiments using the MNIST (\cref{appendix:mnist}) and CIFAR-10 (\cref{appendix:cifar}) datasets.
For the real-world dataset, we compare the performance of \igd, \randr, and with-replacement SGD.
Training loss and the test accuracy of both MNIST and CIFAR-10 reveal the significant slowdown of \igd at the early stages of the training.
For details of the experiments, we refer readers to the corresponding subsections. 

\subsection{Results for the Function in \Cref{thm:small-lb-sc}}
\label{appendix:small-exp-sc}
Recall that the proof of \Cref{thm:small-lb-sc} uses 4-dimensional functions, formulated through the ``dimension aggregation'' step.
For a clear observation, we conduct experiments using the construction for the ``Moderate'' step size regime, and remove the first and the last dimension.

We use the parameters $\mu = 1.0 \times 10^0$, $L = 1.0 \times 10^4$, $G = 1.0 \times 10^0$, $n = 1.0 \times 10^3$, and choose the step size as $\eta = \frac{1}{\mu n K}$ which corresponds to the moderate step size regime. 
First, we examine the trajectory of \igd{} when initialized at $\vx^*$ in \Cref{fig:small-sc-trajectory}.
Recall that our construction is carefully designed so that the trajectory forms a regular $n$-polygon when starting from $(u_0(\eta), v_0(\eta))$ (see \Cref{subsec:small-lb-sc} for definitions).
As illustrated in \Cref{fig:small-sc-trajectory}, even when the iterate starts at $\vx^*$, it gradually drifts outward and rotates along a circular path.

\tightmargin
\tightmargin
\begin{figure}[ht]
    \centering
    \includegraphics[width=0.55\linewidth]{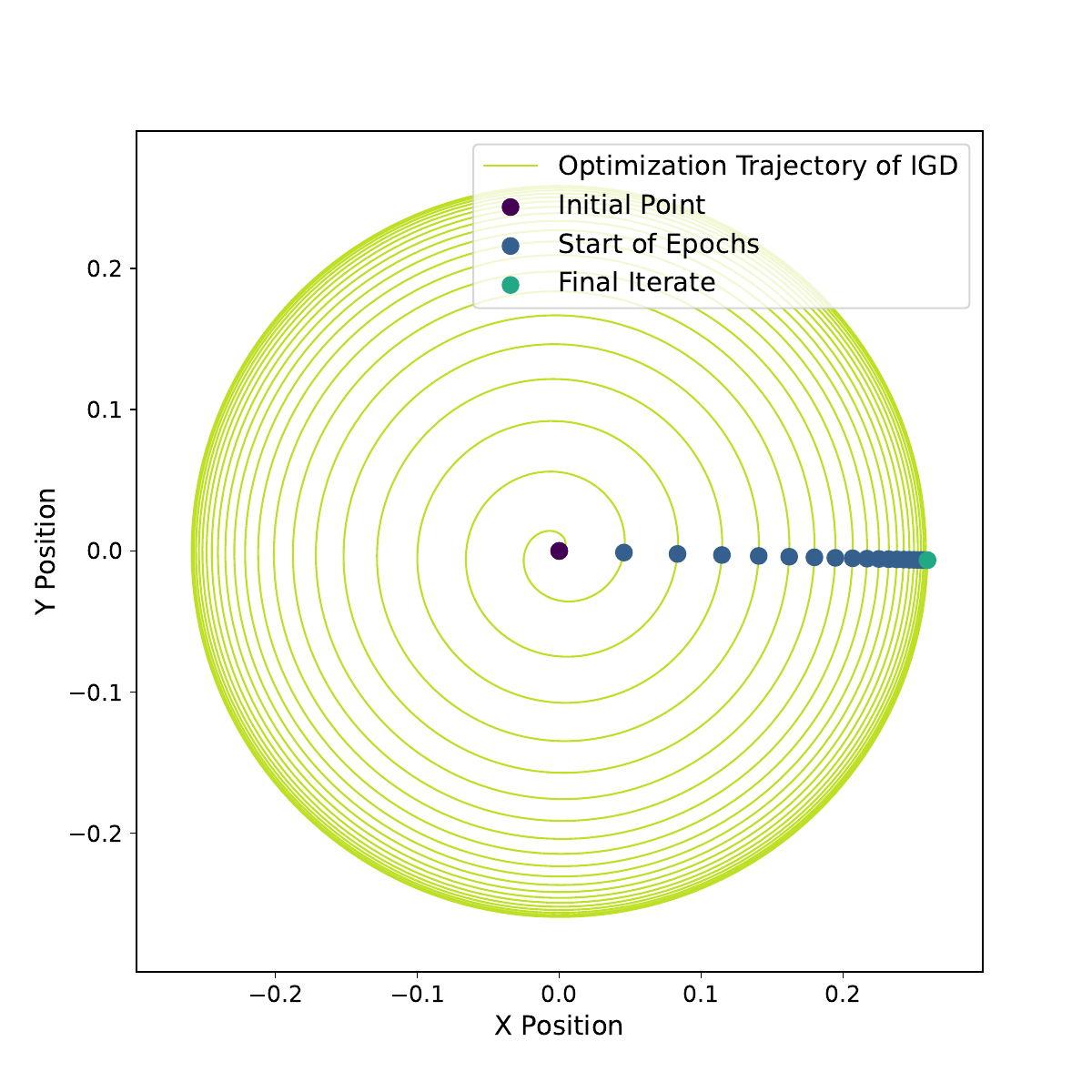}
    \tightmargin
    \tightmargin
    \caption{Trajectory of \igd{} with the function for \Cref{thm:small-lb-sc}, starting from the $\vx^*$ (the origin, purple dot), when $K = 20$. Blue dots starting point of each epoch, $\vx^k_0$, while the cyan dot indicates the final iterate $\vx^K_0$. }
    \label{fig:small-sc-trajectory}
\end{figure}

\Cref{fig:small-sc} reports the function optimality gap for different permutation-based SGD methods, when initialized at $(u_0(\frac{1}{\mu n K}, v_0(\frac{1}{\mu n K})))$.
Results for \randr{} and with-replacement SGD, which involves randomness, are reported after averaging over 20 trials for each number of epochs $k$.
The shaded region represents the first and the third quartiles across the 20 trials.

One might wonder why the trend of \igd{} does not match the rate derived in \Cref{thm:small-lb-sc}, given by $\frac{LG^2}{\mu^2}\min\left\{1, \, \frac{\kappa^2}{K^4}\right\}$.
We believe this occurs because the theoretical rate serves as a lower bound on the true convergence rate, and the empirical performance of \igd{} in this experiment can be influenced by additional factors not captured in the theoretical bound.

\begin{figure}[ht]
    \centering
    \includegraphics[width=0.55\linewidth]{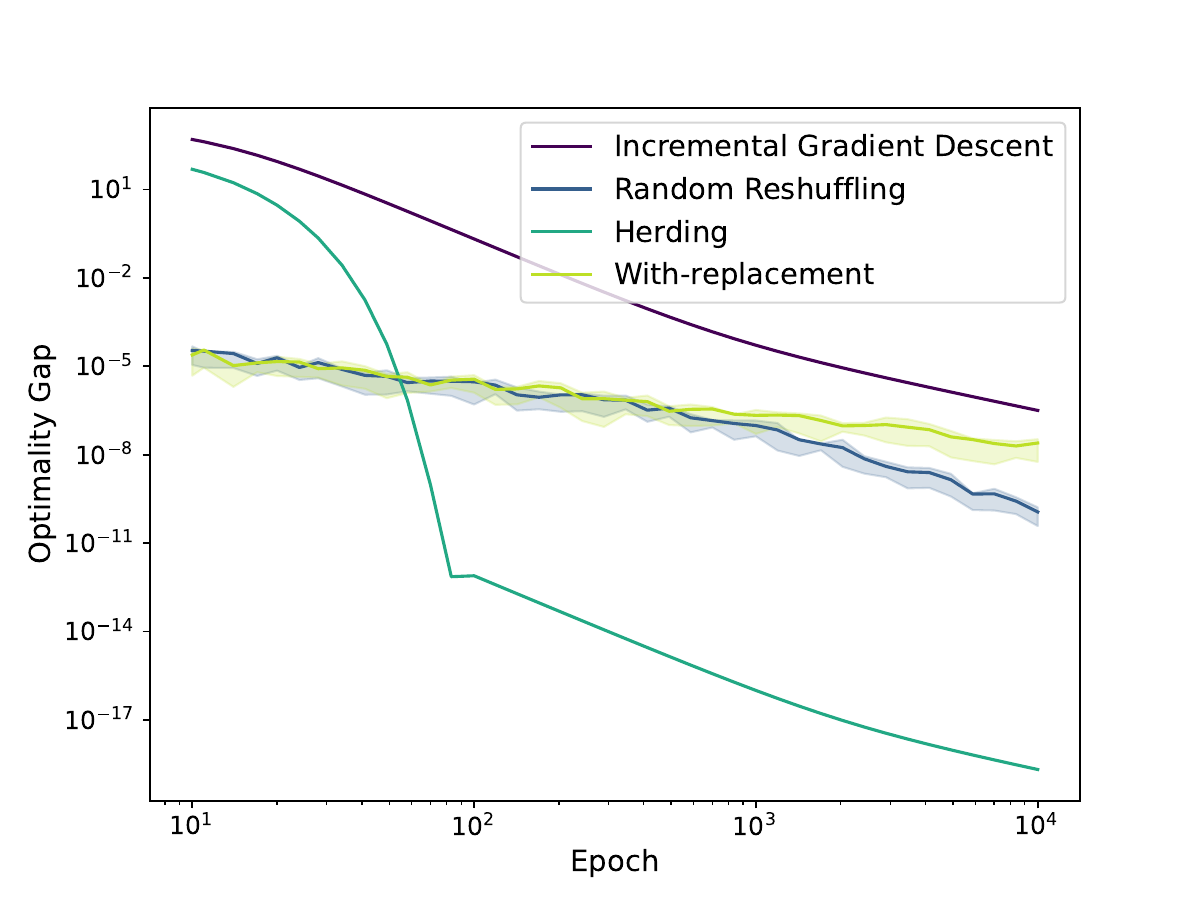}
    \caption{Experiments on \Cref{thm:small-lb-sc} for \igd{}, \randr{}, Herding at Optimum, and with-replacement SGD.
    Both axes are log-scaled.}
    \label{fig:small-sc}
\end{figure}

\subsection{Results for the Function in \Cref{thm:small-lb-concave}}
\label{appendix:small-exp-concave}
Recall that the proof of \Cref{thm:small-lb-concave} uses 4-dimensional functions, formulated through the ``dimension aggregation'' step. 
For a clear observation, we conduct experiments using the construction for the ``Moderate \& Large'' step size regime, and remove the first dimension.

We use the parameters $\mu = 1.0 \times 10^0$, $L = 1.0 \times 10^4$, $G = 1.0 \times 10^0$, $n = 1.0 \times 10^2$, and choose the step size as $\eta = \frac{1}{\mu n K}$.

\tightmargin
\tightmargin
\begin{figure}[ht]
    \centering
    \includegraphics[width=0.55\linewidth]{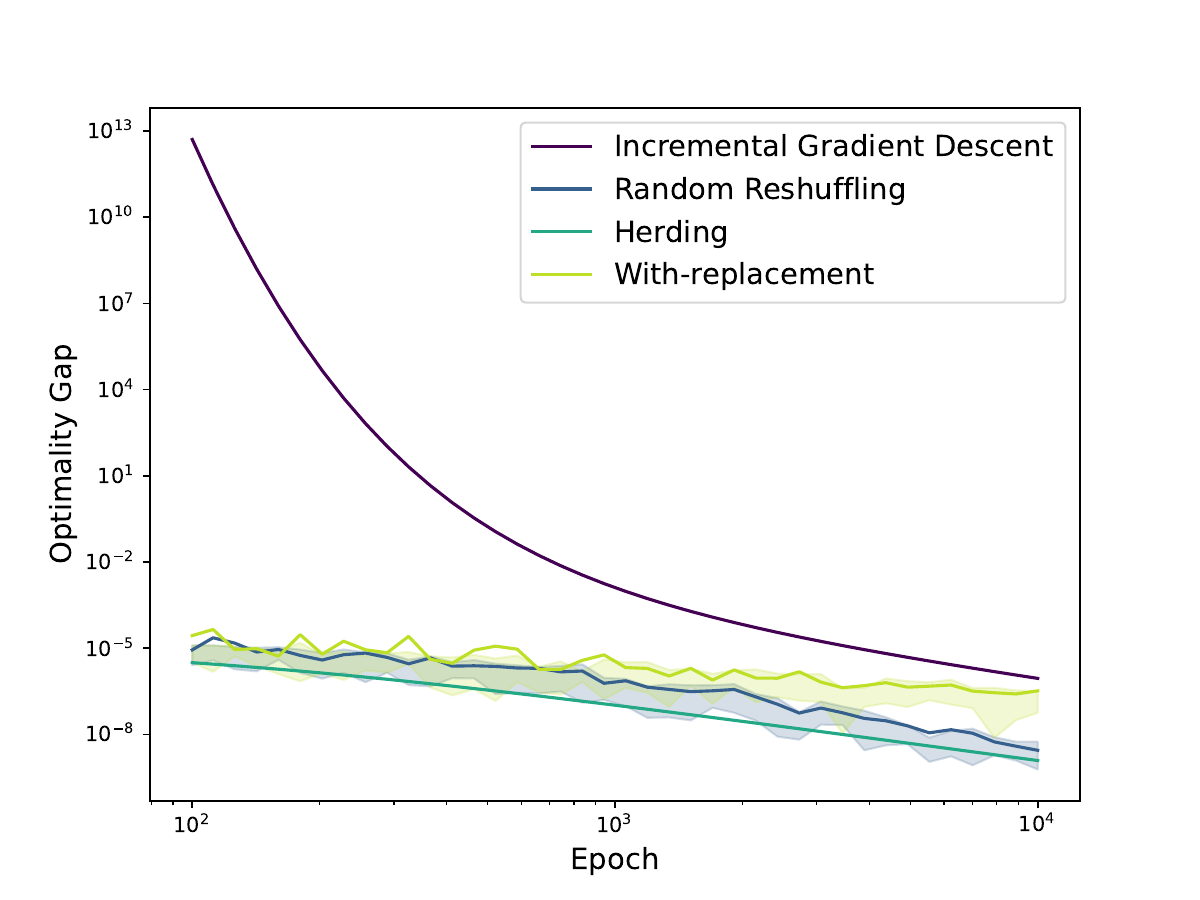}
    \tightmargin
    \tightmargin
    \caption{Experiments on \Cref{thm:small-lb-concave} for \igd{}, \randr{}, Herding at Optimum, and with-replacement SGD.
    Both axes are log-scaled.}
    \label{fig:small-concave}
\end{figure}

\Cref{fig:small-concave} reports the function optimality gap for different permutation-based SGD methods, when initialized at $(0, 0)$.
Results for \randr{} and with-replacement SGD, which involves randomness, are reported after averaging over 20 trials for each number of epochs, $k$.
The shaded region represents the first and the third quartiles over the 20 trials.
As suggested by \Cref{thm:small-lb-concave}, the function optimality gap increases sharply as $K$ decreases.
In contrast, \randr{} remains robust even for small $K$.

\subsection{Experiments on MNIST Dataset}
For the MNIST dataset, we consider the binary classification using only the data corresponding to the labels \textit{0} and \textit{1}. 
We consider the natural data ordering where all \textit{0} images are followed by all \textit{1} images. 
In this configuration, we have a total of 5,923 + 6,742 = 12,665 training data. 
We use a step size $\eta = 0.01$ throughout every part of the training. 
\label{appendix:mnist}
\begin{figure}[ht]
    \centering
    \begin{subfigure}{0.47\textwidth}
        \centering
        \includegraphics[width=\textwidth]{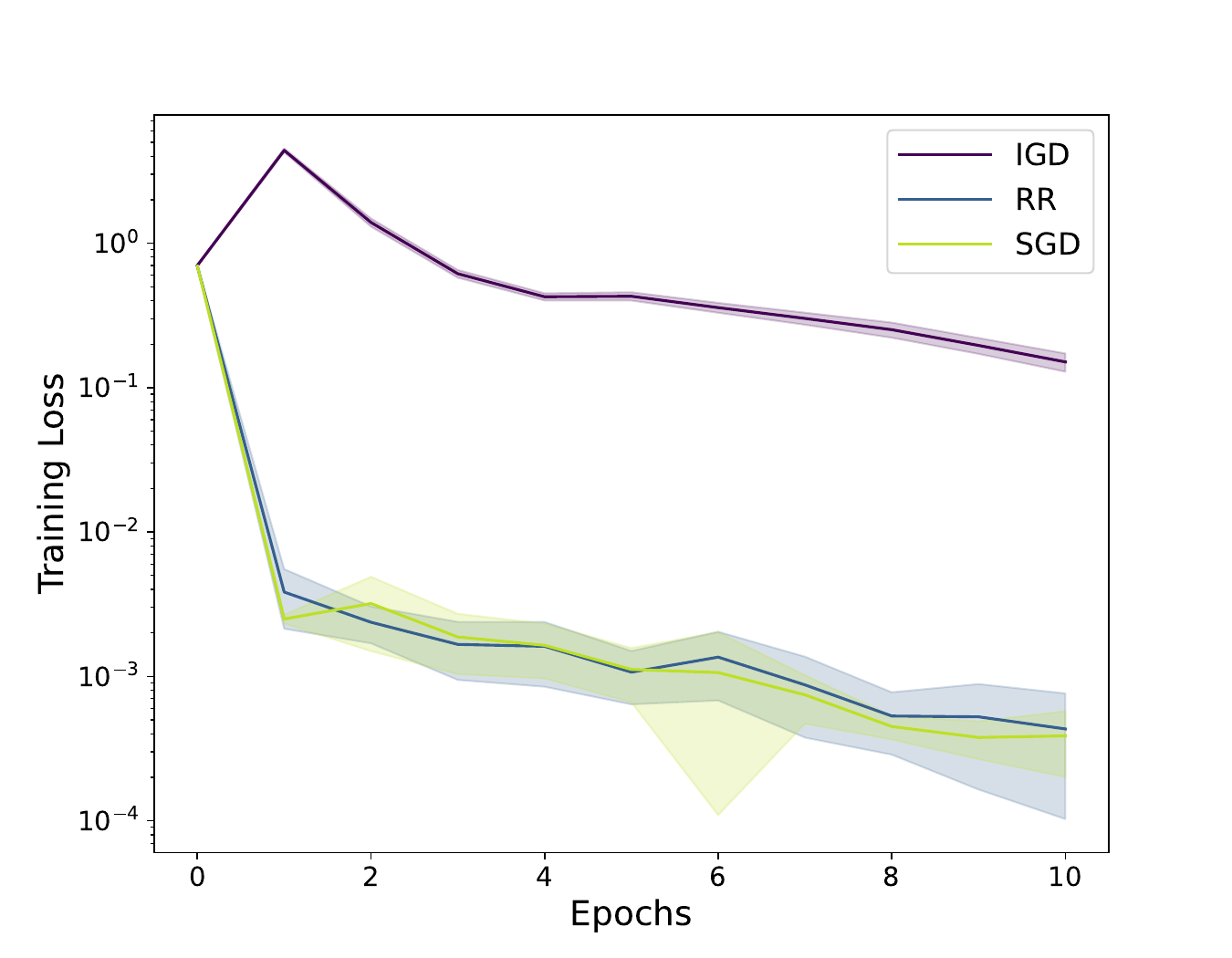}
        \caption{Training loss}
        \label{fig:MNIST-training-loss}
    \end{subfigure}
    \begin{subfigure}{0.47\textwidth}
        \centering
        \includegraphics[width=\textwidth]{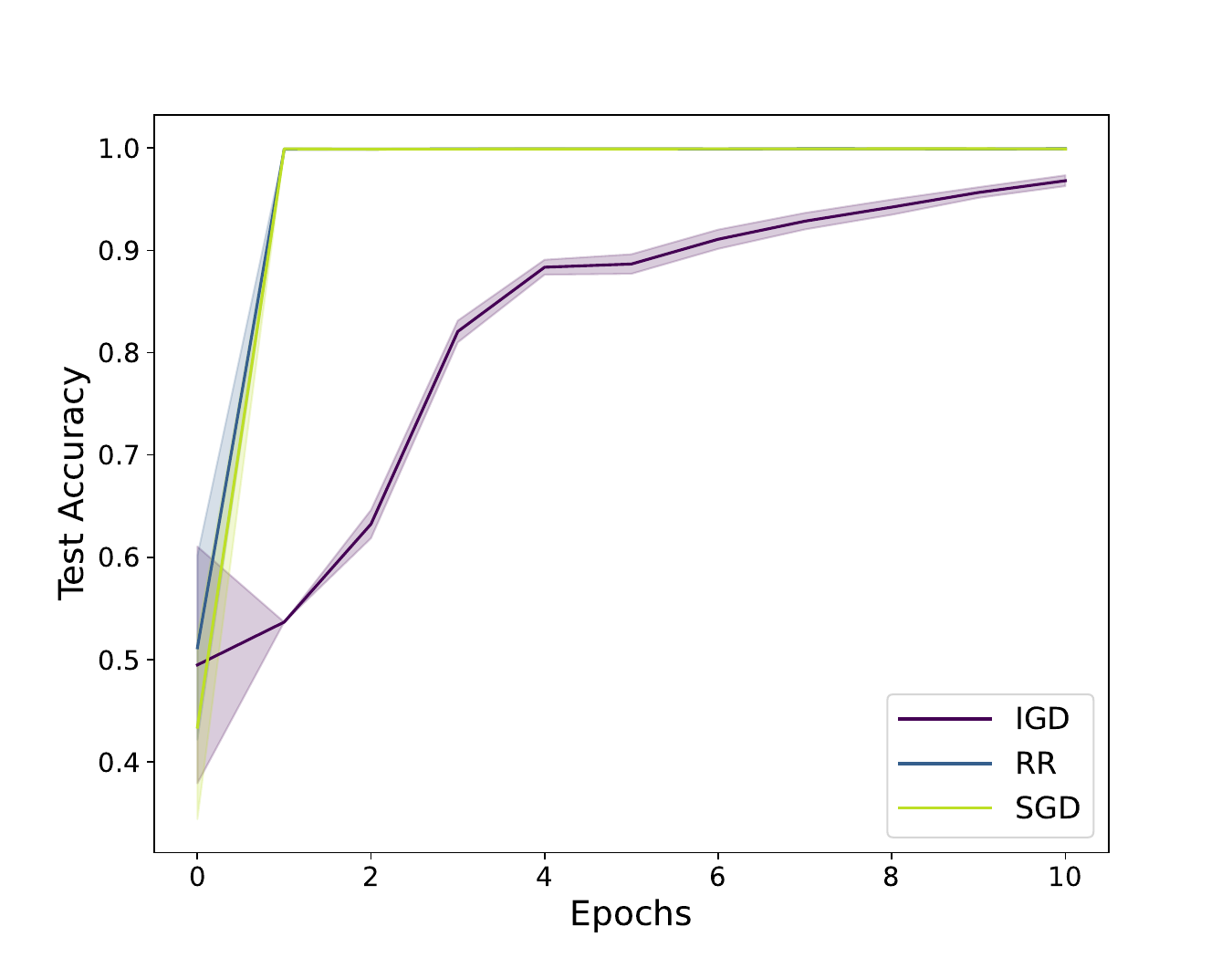}
        \caption{Test accuracy}
        \label{fig:MNIST-test-error}
    \end{subfigure}
    \caption{Experiments on MNIST dataset for \igd, \randr, and with-replacement SGD. 
    $y$-axis for the training loss is log-scaled.}
    \label{fig:mnist-results}
\end{figure}

\Cref{fig:mnist-results} reports the training loss and the test accuracy for different permutation-based SGD methods, with a random initialization.
Results are reported after averaging over 10 trials for each number of epochs, $k$.
The shaded region represents the 95\% confidence interval over 10 trials.
Unlike the experiments on the functions corresponding to the theorems using a fixed initialization, the randomness in the initialization for this experiment introduces a confidence interval even to \igd. 
Both the loss and the accuracy show no significant difference between \randr{} and with-replacement SGD, while \igd~shows a significantly slower convergence compared to the other two methods. 

\subsection{Experiments on CIFAR-10 Dataset}
\label{appendix:cifar}
For the CIFAR-10 dataset, we also consider the binary classification using only the data corresponding to the labels \textit{airplane} and \textit{automobile}. 
We consider the natural data ordering where all \textit{airplane} images are followed by all \textit{automobile} images. 
In this configuration, we have a total of 5,000 + 5,000 = 10,000 training data. 
We use a step size $\eta = 0.001$ throughout every part of the training. 

\begin{figure}[ht!]
    \centering
    \begin{subfigure}{0.47\textwidth}
        \centering
        \includegraphics[width=\textwidth]{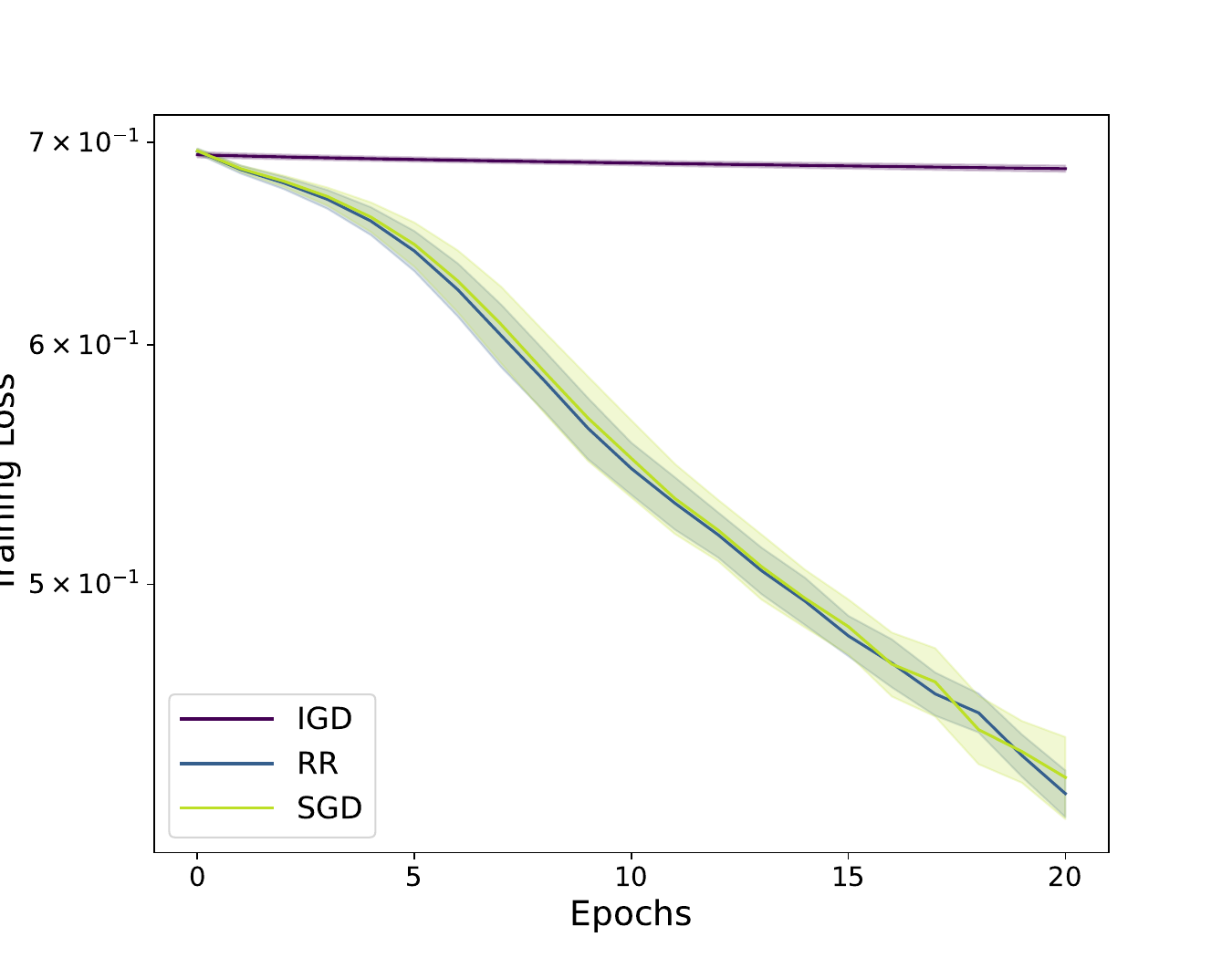}
        \caption{Training loss}
        \label{fig:CIFAR-training-loss}
    \end{subfigure}
    \begin{subfigure}{0.47\textwidth}
        \centering
        \includegraphics[width=\textwidth]{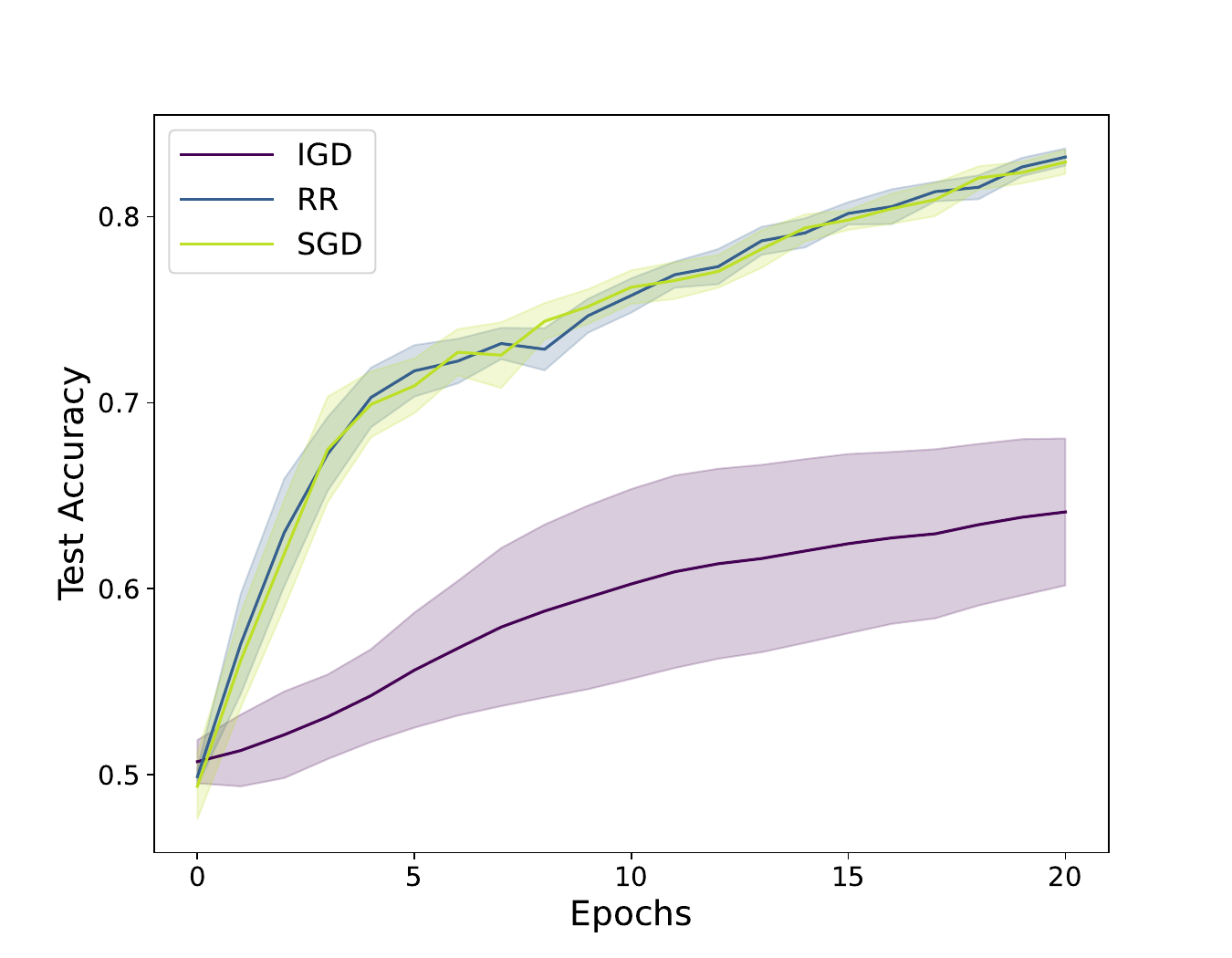}
        \caption{Test accuracy}
        \label{fig:CIFAR-test-error}
    \end{subfigure}
    \caption{Experiments on CIFAR-10 dataset for \igd, \randr, and with-replacement SGD. 
    $y$-axis for the training loss is log-scaled.}
    \label{fig:cifar-results}
\end{figure}

One slight difference from the experiment on the MNIST dataset is that we use a mini-batch of size 16 for the training. 
This is due to the instability of \igd~training. 
To ensure convergence of \igd~with a reasonable step size---such that the loss function decreases even with a small number of training epochs---we employ its mini-batch variant. 
For a fair comparison, we also adopt the corresponding mini-batch versions of \randr~and with-replacement SGD.

\Cref{fig:cifar-results} reports the training loss and the test accuracy for different permutation-based SGD methods, with a random initialization. 
Results are reported after averaging over 10 trials for each number of epochs, $k$. 
The shaded region represents the 95\% confidence interval over 10 trials. 
Both the loss and the accuracy show no significant difference between \randr{} and with-replacement SGD, while \igd~shows a significantly slower convergence compared to the other two methods.

\end{document}